\documentclass{article}

\usepackage{microtype}
\usepackage{graphicx}
\usepackage{subcaption}
\usepackage{booktabs} 

\usepackage{hyperref}

\usepackage[preprint]{icml2026}

\usepackage{amsmath}
\usepackage{amssymb}
\usepackage{mathtools}
\usepackage{amsthm}

\usepackage{amsmath,amsfonts,bm}

\def\eqref#1{equation~\ref{#1}}

\def\1{\bm{1}}

\def\vone{{\bm{1}}}

\def\va{{\bm{a}}}

\def\ve{{\bm{e}}}

\def\vh{{\bm{h}}}

\def\vn{{\bm{n}}}

\def\vx{{\bm{x}}}

\def\eva{{a}}

\def\evx{{x}}

\def\mA{{\bm{A}}}

\def\mC{{\bm{C}}}

\def\mE{{\bm{E}}}

\def\mI{{\bm{I}}}

\def\mN{{\bm{N}}}

\def\mP{{\bm{P}}}

\def\mW{{\bm{W}}}
\def\mX{{\bm{X}}}

\DeclareMathAlphabet{\mathsfit}{\encodingdefault}{\sfdefault}{m}{sl}
\SetMathAlphabet{\mathsfit}{bold}{\encodingdefault}{\sfdefault}{bx}{n}

\def\gD{{\mathcal{D}}}
\def\gE{{\mathcal{E}}}

\def\gG{{\mathcal{G}}}
\def\gH{{\mathcal{H}}}

\def\gK{{\mathcal{K}}}
\def\gL{{\mathcal{L}}}

\def\gO{{\mathcal{O}}}

\def\gS{{\mathcal{S}}}

\def\gU{{\mathcal{U}}}
\def\gV{{\mathcal{V}}}
\def\gW{{\mathcal{W}}}

\def\sI{{\mathbb{I}}}

\def\sN{{\mathbb{N}}}

\def\emC{{C}}

\def\emE{{E}}

\def\emX{{X}}

\DeclareMathOperator*{\argmax}{arg\,max}

\usepackage{wrapfig}
\usepackage{multirow}
\usepackage{colortbl}
\usepackage{rotating}
\usepackage{float}

\usepackage[capitalize,noabbrev]{cleveref}

\theoremstyle{plain}
\newtheorem{theorem}{Theorem}[section]
\newtheorem{proposition}[theorem]{Proposition}

\theoremstyle{definition}

\theoremstyle{remark}

\usepackage[textsize=tiny]{todonotes}

\icmltitlerunning{Hierarchical Discrete Flow Matching for Graphs}

\begin{document}

\twocolumn[
  \icmltitle{Hierarchical Discrete Flow Matching for Graph Generation}



  \icmlsetsymbol{equal}{*}

  \begin{icmlauthorlist}
    \icmlauthor{Yoann Boget}{y,x}
    \icmlauthor{Pablo Strasser}{x}
    \icmlauthor{Alexandros Kalousis}{x}

  \end{icmlauthorlist}

  \icmlaffiliation{y}{Department of Computer Science, University of Geneva, Switzerland}
  \icmlaffiliation{x}{DMML Group, Geneva School for Business administration HES-SO, Switzerland}

  \icmlcorrespondingauthor{Yoann Boget}{yoann.boget@hes-so.ch}

  \icmlkeywords{Machine Learning, ICML}

  \vskip 0.3in
]



\printAffiliationsAndNotice{}  

\begin{abstract}
Denoising-based models, including diffusion and flow matching, have led to substantial advances in graph generation. Despite this progress, such models remain constrained by two fundamental limitations: a computational cost that scales quadratically with the number of nodes and a large number of function evaluations required during generation. In this work, we introduce a novel hierarchical generative framework that reduces the number of node pairs that must be evaluated and adopts discrete flow matching to significantly decrease the number of denoising iterations. We empirically demonstrate that our approach more effectively captures graph distributions while substantially reducing generation time.
\end{abstract}

\section{Introduction}

Graph generation plays an important role across a wide range of scientific and engineering domains, including drug discovery, materials and protein design \citep{Lu2020, Ingraham}, program modeling \citep{Brockschmidt}, natural language processing \citep{Chen2018, Klawonn2018}, and robotics \citep{Li}. In recent years, deep generative models have led to substantial progress in this area, and more recent denoising-based approaches have further improved the ability to learn complex graph distributions \citep{gdss, digress, drum, sid}.

Despite these advances, most denoising-based models for graph generation remain computationally inefficient, mainly for two reasons. First, they require a large number of function evaluations (NFE) at inference time. Second, they incur substantial computational cost by operating over all pairs of nodes. This work aims to address both limitations.

Discrete flow matching methods, such as DeFog \citep{DeFog}, effectively address the first issue by requiring substantially fewer denoising iterations. While DeFog is built on the discrete flow matching framework of \citet{Discrete_FlowMatching_campbell}, we adopt the formulation of \citet{discrete_flow_matching_gat}, which offers significant simplifications.

To address the second issue, we introduce a novel coarse-to-fine hierarchical generation framework designed to reduce the number of node pairs considered during both training and generation. The proposed method follows an expand-and-refine strategy. Its central innovation is the construction of expanded graphs with minimal density, thereby substantially improving the efficiency of the refinement stage.

Our method is both efficient and effective. It outperforms state-of-the-art approaches on multiple standard benchmarks, while drastically reducing both training and generation time. We further leverage these efficiency gains to scale graph generation to larger graphs.

We summarize our contributions as follows:
\begin{itemize}
    \item We propose a novel ad hoc algorithm for hierarchical graph modeling that minimizes the number of node pairs considered. This approach yields substantial improvements in training and generation speed, while also reducing memory requirements.

    \item We leverage our hierarchical method to build a hierarchical discrete flow matching framework for graph generation and show its effectiveness.

    \item We show that our approach naturally yields a conditional generative model enabling effective structural conditioned generation.

    \item We demonstrate that our method achieves state-of-the-art performance on multiple benchmarks while requiring only a fraction of the generation time of concurrent methods. In addition, we show that it scales to larger graphs and introduce new evaluation datasets with larger graph sizes and increased numbers of instances.
\end{itemize}

\section{Background}

Denoising-based models emerged as the dominant paradigm in graph generation (see Appendix \ref{ap:related_work} for alternative approaches). These methods can be broadly categorized into three classes. (1) Continuous-space approaches operate in an unbounded Euclidean domain and include score-based diffusion models \citep{edp-gnn, gdss}, diffusion bridges \citep{drum}, and flow matching methods \citep{CatFlow}. (2) Simplex-based approaches model data directly on the probability simplex, as exemplified by Beta diffusion \citep{BetaGraph}, Graph Bayesian Flow Networks \citep{graphBFN}, and Unrestrained Simplex Denoising \citep{unside}. (3) Discrete-domain approaches act directly on discrete state spaces and encompass discrete diffusion in both discrete and continuous time \citep{discdiff_haefeli, digress, continuous_time_graph_diffusion}, discrete flow matching \citep{DeFog}, and non-Markovian discrete diffusion models \citep{sid}.

These approaches suffer from two principal limitations. First, they typically require a large number of function evaluations (NFE) although some recent works have shown partial success in reducing the number of denoising steps \cite{sid, DeFog, graphBFN}. Second, and more critically, all of these methods rely on dense modeling, which entails substantial computational overhead, often involving multiple feedforward neural networks—applied to every pair of nodes. As the number of node pairs scales as 
$\gO(n^2)$, where $n$ is the number of vertices, this dense dependence constitutes the primary computational bottleneck, leading to slow training and inference as well as significant memory requirement during training. 

To address the computational challenges of graph generation, two main strategies have emerged: methods that operate on a subset of nodes, and hierarchical approaches. The first strategy restricts each denoising step to a subset of active vertices and the relations among them. For example, SparseDiff \citep{sparsediff} randomly selects active nodes, while EDGE \citep{EDGE} identifies them based on predicted degree changes. As a result, only a fraction of node pairs is processed in each forward pass. However, all node pairs must eventually be evaluated, which either leads to degraded modeling performance (as in EDGE) or requires a large number of function evaluations (as in SparseDiff), often reaching several thousand iterations.

The second class of methods consists of hierarchical approaches, which we adopt in this work. In the next section, we introduce the fundamental principles of hierarchical graph generation and discuss existing hierarchical methods in detail. 

\subsection{Hierarchical Graph Generation}

Hierarchical graph modeling typically comprises two components: a graph coarsening procedure and a learned reverse expansion process. We first describe graph coarsening, which constitutes the main point of differentiation between hierarchical methods, and then outline the general principle of refinement from coarse to fine graphs.

\subsubsection{Graph Coarsening}

The goal of graph coarsening is to progressively reduce the number of nodes while preserving task-relevant properties. Given an input graph \(\gG\), hierarchical models construct a sequence of representations \({\gG^0, \ldots, \gG^L}\), where \(\gG^0\) denotes the original graph and \(\gG^L\) the coarsest level. At each level \(\ell\), nodes in \(\gG^\ell\) are obtained by aggregating nodes from \(\gG^{\ell-1}\). Graph coarsening is a broad field of research. Here, we focus on methods that are particularly relevant to hierarchical generative modeling. For comprehensive reviews of graph coarsening techniques, we refer the reader to \citet{graph_pooling_review1, graph_pooling_review2, graph_pooling_review3}.

Coarsening methods can be broadly categorized into contraction-based and clustering-based methods. Contraction-based methods \citep{contraction_loukas, edge_contractionpoolinggraph} include edge contraction, which merges the endpoints of an edge, and neighborhood contraction, which aggregates a node with its neighbors; this strategy is used in models such as \citet{graphle} and \citet{higen}. However, contraction-based coarsening has important drawbacks: it is not permutation-invariant due to its dependence on contraction order. More importantly, it provides limited control over the coarsening ratio \( n^{\ell+1} / n^{\ell} \), where $n^\ell = |\gV^\ell|$. For edge contraction, this ratio is upper-bounded by \(1/2\) and often much smaller in practice. Consequently, such methods typically require many coarsening levels, limiting their efficiency and applicability.

Clustering-based methods \citep{diffpool, mincut} aggregate nodes according to an explicit clustering of the graph and can be viewed as a generalization of contraction-based approaches. By allowing arbitrary aggregation patterns, they provide direct control over the coarsening ratio. Moreover, when the clustering procedure is permutation-equivariant, the resulting coarsening operation is permutation-invariant. This is the approach adopted in this work. 

However, as discussed further in Section \ref{sec:ssg}, the computational advantages of hierarchical modeling depend critically on the choice of clustering algorithm. Standard community-detection–based objectives often yield dense expanded graphs. This is, for instance, the case for the modularity objective \citep{DMoN} employed by \citet{higen} (see Section \ref{sec:ssg_emp}), which is, to our knowledge, the only generative model adopting this approach besides our work. Moreover, \citet{higen} does not exploit the permutation-invariant nature of the clustering procedure, as it models graphs through a sequential process over edges.

\subsubsection{Graph Expansion}\label{sec:graph_expansion}

The reverse process in hierarchical models is graph expansion, which generally follows a two-step procedure: (i) a deterministic expansion step and (ii) a learned refinement step. 

During deterministic expansion, each coarse (or parent) node is split into multiple child nodes. The children originating from the same parent form a clique, while the children of two connected parent nodes form a biclique. We denote $\gH^\ell$, the graph resulting from the expansion of $\gG^{\ell+1}$ Two key observations follow. First, graph expansion requires knowledge of the number of children associated with each parent node. Second, even when these cardinalities are known, expansion does not invert the coarsening operation. Specifically, letting $C$ denote the coarsening operator and $U$ the expansion operator, we generally have $U(C(\gG^\ell)) = \gH^\ell \neq \gG^\ell$. Instead $\gH^\ell$ is spanning supergraph of $\gG^\ell$ satisfying $\gV^\ell_\gG =\gV^\ell_\gH$ and $\gE^\ell_\gG \subseteq \gE^\ell_\gH$, also denotes $\gH^\ell = \gG^\ell + \gE_\gS$, where $\gE_\gS$ represents the additional set of edges in $\gH^\ell$ but not in $\gG^\ell$.

The refinement step follows from this relation between $\gG^\ell$ and $\gH^\ell$. It removes the excessive edges from the expanded graph, and, when necessary, to generate node and edge attributes. In hierarchical graph generation, this stage constitutes the core generative component by modeling the conditional distribution 
$p_{\theta_\ell}(\gG^\ell|\gH^\ell)$. 

\begin{figure}[ht]
  \vskip 0.2in
    \begin{center}
     \centerline{\includegraphics[width=0.8\columnwidth]{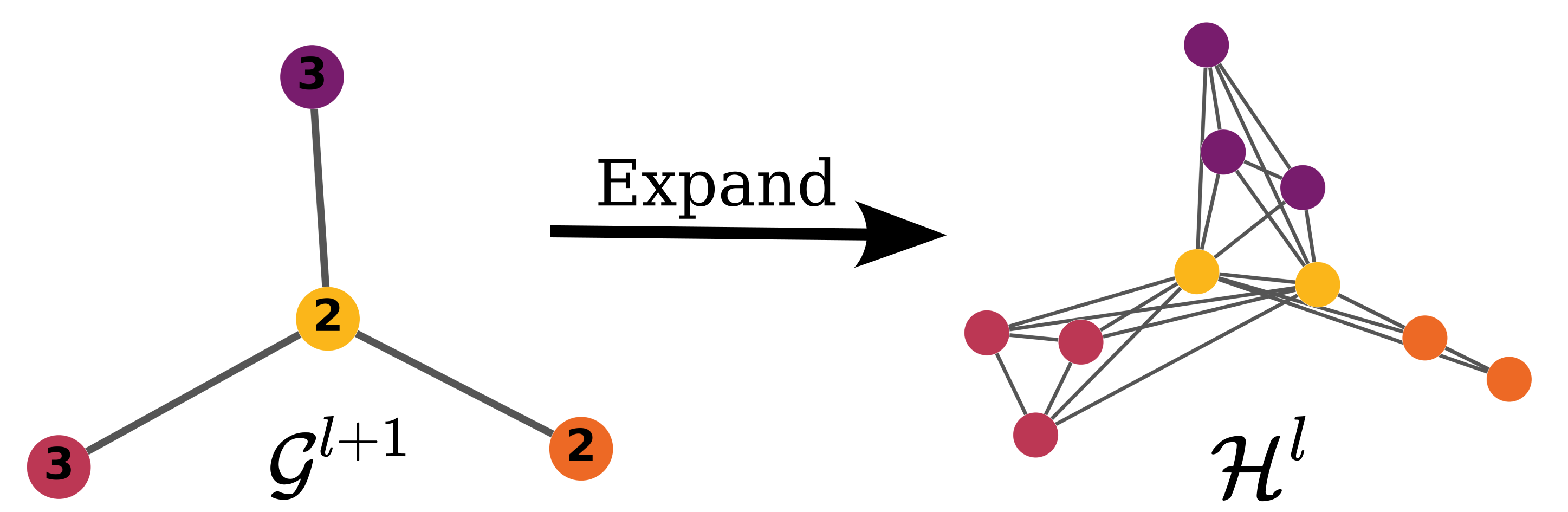}}
     \centerline{\includegraphics[width=0.8\columnwidth]{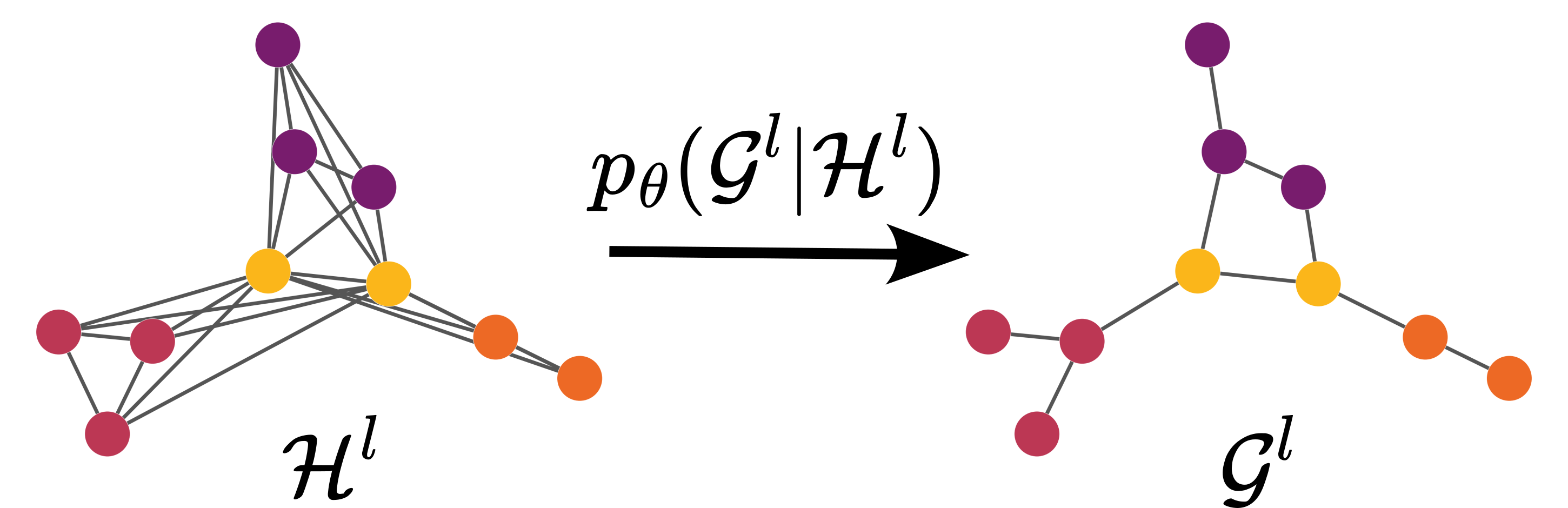}}
    \end{center}
    \caption{\footnotesize \textit{Top}: Deterministic Graph Expansion.
Each parent node spawns a fixed number of child nodes. Children of the same parent form cliques, while children of connected parents form bicliques.
\textit{Bottom}: Graph Refinement.
A generative model prunes excess edges and optionally generates node and edge attributes.
    }
\end{figure}

The overall hierarchical model can thus be expressed as a Markovian process across levels: 
\begin{equation}\label{eq:model}
    p(\gG) = \prod_{\ell=0}^{L} p_{\theta_\ell}(\gG^\ell|\gH^\ell) 
\end{equation}
where $\gH^{\ell} = U(\gG^{\ell+1})$ for $\ell<L$, and $\gH^L = \gK_{n^L}$ denotes the complete graph of size $n^L$, sampled from the empirical distribution. Notably, standard dense graph denoising models arise as a special case of this formulation when $L=0$.  

By leveraging the sparsity of $\mathcal{H}^\ell$, hierarchical models based on node clustering can offer substantial advantages. First, they avoid computations over all node pairs, reducing the computational complexity from $\mathcal{O}(n^2)$ to $\mathcal{O}(m_{\mathcal{H}})$, where $m_{\mathcal{H}} = |\mathcal{E}{\mathcal{H}}|$. Second, since $\mathcal{G}^\ell$ is a subgraph of $\mathcal{H}^\ell$, the sparsity of $\mathcal{H}^\ell$ constrains the space of the conditional distribution $p{\theta_\ell}(\mathcal{G}^\ell \mid \mathcal{H}^\ell)$, thereby simplifying the modeling task.

However, these benefits critically depend on the sparsity of the expanded graphs. Existing clustering-based coarsening methods are typically designed for community detection, which often produces dense coarse graphs and, consequently, dense spanning supergraphs. To address this limitation, we propose a clustering algorithm that explicitly maximizes the sparsity of the expanded graph. We present this method in Section \ref{sec:ssg}, followed by a description of the hierarchical generative models in Section \ref{sec:model}.
\section{Spanning Supergraph Density Minimization}\label{sec:ssg}


Sparsity of the expanded graph is a crucial property in hierarchical models, as it significantly reduces computational cost and simplifies modeling by constraining the graph space. However, most existing clustering-based coarsening methods, such as DiffPool \citep{diffpool}, MinCut \citep{mincut}, or DMoN \citep{DMoN}, are designed to identify communities in the graph. Consequently, they often produce dense coarse graphs, which in turn lead to dense expanded graphs.
To illustrate this issue, consider a partition of a graph into \(K\) communities that are all mutually connected. The resulting coarse graph is then the complete graph on \(K\) nodes, and expanding it as described in Section \ref{sec:graph_expansion} again yields a complete graph. This observation motivates the need for a dedicated coarsening strategy that explicitly aims to minimize the density of the expanded graph.

In the following, we first formally describe the proposed coarsening and expansion procedures. We then introduce a training objective that explicitly minimizes the density of the expanded graph. Finally, we empirically demonstrate that our method produces expanded graphs that are substantially sparser than those obtained with existing clustering approaches.

\subsection{Hierarchical Graph Representations}

In this section, we represent each coarse graph $\{\gG^1, \ldots, \gG^L \}$ by an annotation matrix and an adjacency matrix, denoting  $\gG^\ell = (\mX^\ell_\gG, \mA_\gG^\ell)$, except for the original data graph $\gG^0$,  which includes an annotation matrix only when its nodes are attributed. Expanded graphs $\gH^\ell$ are always unannotated and are therefore fully specified by their adjacency matrices $\mA_\gH^\ell$. Finally, we define the assignment matrix $\mC_\ell \in \{0, 1\}^{n^\ell \times K^\ell}$, where each row is a one-hot indicator of a node's cluster membership, and $K^\ell$ represents the number of clusters. 

Given a graph $\gG^\ell$ and an assignment matrix $\mC_\ell$, we define the corresponding coarse graph as
\begin{equation}\label{eq:coarsening}
    \gG^{\ell+1} := \left( \mX^{\ell+1}, \mA_\gG^{\ell+1} \right), \; \text{with}
\end{equation}
\begin{equation*}
    \mX^{\ell+1} = \mC^T_\ell \vone_{n^\ell}, \;  \; \mA_\gG^{\ell+1}  = \sI (\mC_\ell^T \mA_\gG^\ell \mC_\ell > 0),
\end{equation*}
where $\sI(\cdot)$denotes the element-wise indicator function and $\vone_{n^\ell}$ is the vector of ones of length $n^\ell$. Intuitively,  while the adjacency matrix $\mA_\gG^{\ell+1}$ connects two parent nodes whenever at least one edge exists between their respective child nodes. The annotation matrix $\mX^{\ell+1} \in \sN^{n \times 1}$ is actually a vector, which encodes the number of child nodes aggregated into each parent node. In the following, we also denotes this vector $\vn^{l+1}$. 

Conversely, we define the graph expansion as: 
\begin{equation}\label{eq:expansion}
\mA_\gH^\ell := \mC_\ell(\mA_\gG^{\ell+1} + \mI_{n^{\ell+1}}) _\gG\mC_\ell^T - \mI_{n^\ell},
\end{equation}
where $\mI_{n}$ denotes the identity matrix of size $n$, interpreted as self-connections. Adding self-connections at the coarse level ensures that child nodes originating from the same parent form cliques in $\gH^\ell$, while children of connected parent nodes form bicliques.  

Importantly, the expanded graph $\gH^\ell$ only depends on the coarse graph $\gG^{\ell+1}$.  Indeed, by replicating the $i^\text{th}$ row of the identity matrix $n_i$ times, where $n_i$ is given by the corresponding node annotation $\emX_i$, we recover, up to a permutation, the assignment matrix $\mC_\ell$. In practice, this operation corresponds to  $\texttt{repeat\_interleave}(\mI_{n^{\ell+1}}, \vx^{\ell+1})$ in \texttt{PyTorch}. Consequently, the expansion operator is entirely determined by the coarse graph, allowing the expanded graph to be constructed without any additional information, which is crucial for generation (See Algorithm \ref{algo:sample}).

\subsection{Properties}

We derive two key properties of this hierarchical framework that are essential for the generative model.
Let $C$ and $U$ denote the coarsening and expansion operators defined in Equations \ref{eq:coarsening} and \ref{eq:expansion}, respectively. Considering an unattributed graph represented only by its adjacency matrix, $\gG^\ell_{\text{unattr.}} = \mA^\ell_{\gG}$, we can state the following propositions.

\begin{proposition}\label{prop:spanning}
    The expanded graph $\gH^\ell$ is a spanning supergraph of $\gG^\ell_\text{unattr.}$, that is:
    \begin{equation}
        \gH^\ell = U(C(\gG^\ell_\text{unattr.}))  \implies \gH^\ell = \gG^\ell + \gE_\gS     
    \end{equation}
    Proofs are provided in Appendix \ref{ap:proofs}. \hfill \qedsymbol
\end{proposition}

\begin{proposition}\label{prop:invar_coarse}
    As functions of $\gG^\ell$, the coarsened representation $\gG^{\ell+1}=(\mA_\gG^{\ell+1}, \mX_\gG^{\ell+1})$ is permutation-invariant, whereas the expanded representation $\gH^{\ell}=\mA_\gH^\ell$ is permutation-equivariant. \hfill \qedsymbol
\end{proposition}
The equivariance property is particularly important since it ensures that the representations of $\gG^\ell$ and $\gH^\ell$ are aligned, preventing the need for a costly matching procedure. 

By coarsening and expanding graphs, we obtain a dataset with graph representations ordered from finer to coarser  $\{(\gG_i^0, \gH_i^0), ... , (\gG_i^{L}, \gH_i^{L})\}_{i=1}^N$.

\subsection{\emph{D}-Min Clustering}

Given a graph $\gG^\ell$, the spanning supergraph $\gH^\ell$ is fully determined by the node clustering encoded by the assignment matrix $\mC_{\ell}$. 
Our goal is to find the clustering that minimize the density of the expanded graph, $D_{\gH^\ell}^\ell = \frac{m_\gH^\ell
}{n^\ell(n^\ell-1)}$.
However, the exact minimizer is computationally intractable even for moderate graph sizes. 
We therefore propose to parameterize the assignment with a graph neural network and learn 
$\mC_{\ell}$ by directly minimizing the density of $D_{\gH}^\ell$. 

A key observation is that $D_{\gH^\ell}$ can be computed from the weighted coarse graph defined as:
\begin{equation}
    \gG_\gW^{\ell+1} := \mA_\gW^{\ell+1} = \mN_\gG^{\ell+1}{(\mA^{l+1})_\gG + \mI_{n^{\ell + 1}})\mN^{\ell+1}},
\end{equation}
where, $\mN^{\ell+1} = \text{diag}(\vn^{\ell+1})$. 
Intuitively, $\mA_\gW^{\ell+1} $ quantifies, for each coarse node pair, how many fine-level edges would be induced after expansion. Thus, the density of the spanning supergraph follows as
\begin{equation}
    D_\gH^\ell = \frac{1}{n^\ell(n^\ell-1)}\left[\left(\vone_{n^\ell+1}^T \mA_\gW^{\ell+1} \vone_{n^\ell+1} \right) - n^\ell\right].
\end{equation}

We then parameterize the soft assignment matrix $\tilde{\mC_\ell}$ as:
\begin{equation}\label{eq:assignment}
 \hat{\mC}_\ell^\varphi =  \text{Softmax}_\text{row} \left(\text{GNN}_\varphi(\mX^\ell_\gG, \mA^\ell_\gG)\right),
\end{equation}
where $\sigma$ denotes the row-wise softmax. We use the discretization  $(\mC_{\ell})_{i, j} = \sI\left[ j = \argmax_k \hat{\emC}_{i, k}^\varphi \right]$ during the forward pass, and the straight-through estimator in backpropagation.
Thus, simple and easily interpretable, the training objective consists in directly minimizing the expected density of the expanded graph: 
\begin{equation}
    \gL_\varphi = \mathbb{E}_{\gG^\ell \sim p_\text{data}} \left [D_\gH^\ell \right].
\end{equation}
In Section \ref{sec:ssg_emp}, we empirically demonstrate that our objective consistently yields substantially sparser spanning supergraphs.

\subsection{Empirical Evaluation}\label{sec:ssg_emp}

We finally demonstrates the empirical effectiveness of our method.
In Table \ref{tab:gamma}, we reports the average density of the spanning supergraphs obtained after the first coarsening step across various datasets (See Appendix  \ref{ap:dataset}), comparing our model ($D$-Min) with three reference deep coarsening methods: DiffPool \citep{diffpool}, DMoN \citep{DMoN} and MinCut \citep{mincut}, (See Appendix \ref{ap:gamma_min} for the experimental details).

The experiments show that our $D$-Min clustering method yields significantly sparsier spanning supergraphs. Specifically, on large dataset our model yield sparsity is several time lower than competitive model. It thus improves the model's scalability and efficiency, as well as its capacity to model larger graph. These empirical results highlight the importance of developing \emph{ad hoc} coarsening methods for hierarchical generative approaches.

\begin{table}[ht]
    \scriptsize
    \begin{center}
    \begin{sc}
    
    \caption{\footnotesize Average density of the spanning supergraphs for four coarsening methods.}
    
    \begin{tabular}{l c c c}
    \toprule
       Method & Zinc250k & SBM20k & Reddit12k \\
    \midrule
        Data     & 0.09 $\pm$ 0.00 & 0.08 $\pm$ 0.00 & 0.01 $\pm$ 0.00 \\
        DiffPool & 0.34 $\pm$ 0.00 & 0.87 $\pm$ 0.03 & 0.82 $\pm$ 0.03 \\
        DMoN     & 0.57 $\pm$ 0.01 & 0.63 $\pm$ 0.02 & 0.04 $\pm$ 0.00 \\
        MinCut   & 0.37 $\pm$ 0.02 & 0.26 $\pm$ 0.01 & 0.12 $\pm$ 0.01 \\
        $D$-Min (Ours) & \textbf{0.25 $\pm$ 0.00} & \textbf{0.22 $\pm$ 0.00} & \textbf{0.02 $\pm$ 0.00} \\
    \bottomrule
    \end{tabular}
    \label{tab:gamma}

    \end{sc}
    \end{center}
\end{table}

\section{Hierarchical Discrete Flow Matching}\label{sec:model}

The hierarchical preprocessing introduced in Section \ref{sec:ssg} produces a dataset made of increasingly coarse graph representations, each paired with a spanning supergraph: $\gD = \{(\gG_i^0, \gH_i^0), ... , (\gG_i^{L}, \gH_i^{L})\}_{i=1}^N$.
Our generative models operate independently at each level $\ell$ by learning the conditional distribution $p_\theta(\gG^\ell \mid \gH^\ell)$ leveraging the sparse structure of $\gH^\ell$.

The proposed formulation is compatible with a broad class of denoising-based generative models. In this work, we instantiate it using the discrete flow matching framework of \citet{discrete_flow_matching_gat}. For simplicity, we assume that node and edge attributes are univariate categorical variables; extension to the multivariate case is straightforward.

\subsection{Discrete Flow Matching}

Two main formulations of Discrete Flow Matching (DFM) were introduced concurrently by \citet{Discrete_FlowMatching_campbell} and \citet{discrete_flow_matching_gat}. The former was subsequently adapted to graph generation by \citet{DeFog} and further augmented with heuristics that improve some evaluation metrics, albeit at the cost of a larger hyperparameter space and increased algorithmic complexity. In contrast, the latter formulation is conceptually simpler, involves fewer hyperparameters, and is computationally more efficient, which motivates our choice to adapt it to graph generation. A recent submission \citep{simGFM} also builds on this formulation for graph modeling, but proposes a variant based on double-sampling updates, which introduce additional complexity and slow down sampling. We further discuss this method in Appendix~\ref{ap:rw_dfm}.

In the following, we briefly introduce DFM. We denote by $z^i \in [k]$ a generic univariate $k$-categorical variable, such as a node or an edge attribute, and by $Z$ the whole instance such as a grpah. For clarity, we omit level indices and subscripts whenever no ambiguity arises.

In DFM, a probability path is defined between a source distribution $p_0$ from which noise samples are drawn, and the data distribution $p_1$. For categorical variables, a simple and practical choice is the interpolant formulation, given by
\begin{equation}\label{eq:prob_path}
    p_t(z^i \mid z^i_0, z^i_t) =  \alpha [z^i = z^i_1] + (1-\alpha)[z^i = z^i_0],
\end{equation}
where $[\cdot]$ denotes the Iverson bracket and $\alpha \in [0, 1]$ is a scheduling function, which we instantiate in practice as $1 - (1 - t)^u$ with $u \in \{2, 3\}$.

In sequence modeling, masked tokens or uniform distributions are commonly used as source distributions. For graph diffusion models, \citet{digress} showed that using the marginal distribution yields superior performance. Following this observation, we adopt the node and edge marginals as the respective source distributions.

Considering the update rule
\begin{equation}\label{eq:update_rule}
    z^i_{t+\Delta_t} \sim  [z^i_{t+\Delta_t} = z^i_t] + \Delta_t u^i_t( z^i_{t+\Delta_t}, Z),
\end{equation}
we say that a probability velocity $u_t$ generates the probability path $p_t$ if, given that $Z_t\sim p_t$, $Z_{t+\Delta_t} \sim p_{t+\Delta_t} $. 
We note that the update rules acts independently on each dimension, but the probability velocity depends on the whole instance $Z$.
\citet{discrete_flow_matching_gat} show that a valid probability velocity for the probability path defined in \ref{eq:prob_path} is given by: 
\begin{equation}\label{eq:prob_velocity}
    u^i_t(z^i, Z_t) = \frac{\Dot{\alpha}}{1-\alpha}\Delta_t \left(p_{1|t}(z^i|Z_t)  - [z^i = z^i_t]\right)
\end{equation}

In Equation \ref{eq:prob_velocity}, the term $p_{1|t}(x^i|x_t)$, called denoisier, is usually intractable and is approximated via learned model. 

\subsection{Conditional Discrete Flow Matching for Graphs}

We now extend DFM to graph generation and the conditional setting required to model $p(\gG^\ell \mid \gH^\ell)$. We represent the spanning supergraph $\gH$ by its edge index matrix $\mE \in \mathbb{N}^{2 \times m_{\gH}}$. We encode $\gG$ as node and edge attributes defined on the expanded graph $\gH$. Specifically, $\gG$ is represented as a vector of edge attributes $\va \in \mathbb{N}^{m_{\gH}}$, and a node annotation vector $\vx \in \mathbb{N}^{n_\gH}$ when $\gG$ has node attributes. Thanks to the equivariance of the hierarchical preprocessing, edge attributes are aligned with the edge indices, so that each attribute $\eva_k$ corresponds to the edge $\emE_{\cdot,k}$.
For unattributed graphs, $\va$ is a binary vector indicating edge presence; for attributed graphs, the absence of an edge is encoded as label $0$.

We define a conditional probability path over the node and edge attributes $(\va,\vx)$. Under this construction, the conditional probability path over graphs factorizes as
\begin{multline}
        p_t(\gG \mid \gG_1, \gG_0, \gH) = \prod_{i=1}^{n_\gH} p_t(x^i \mid x_0^i, x_1^i) \times\\
    \prod_{k=1}^{m_\gH} p_t(a^{k} \mid a_0^{k}, a_1^{k}).
\end{multline}
This choice leaves the generic form of the probability velocity essentially unchanged, except for the additional conditioning on $\gH$ in the probability velocity $u^i_t(z^i, \gG_t, \gH)$ and the corresponding modification of the denoiser $p_{1|t}(z^i|\gG_t, \gH)$.   

We parametrize the denoiser by a sparse graph neural network $f_\theta(\mE, \va, \mX)$, whose connectivity is given by $\mE$. The denoiser is trained using a standard cross-entropy loss:
\begin{multline}\label{eq:loss}
    \gL_\theta = \mathbb{E} \left[ \gamma \sum_{i=1}^{n_\gH} [-\text{log}(p_\theta(\evx_{i}|\gG^{t}, \gH))] + \right. \\ \left.
    (1-\gamma) \sum_{j=1}^{m_\gH}[-\text{log}(p_\theta(\eva_{j}| \gG_{t}, \gH))] \right],
\end{multline}
where the expectation is taken over $t \sim \gU(0, 1)$ and $(\gG, \gH)\sim p_{\text{data}}$. We set the weighting factor between nodes and edges to $\gamma = n_\gH/(n_\gH+m_\gH)$.

Sampling proceeds by drawing an initial graph $\gG_0 \sim p_0$ and iteratively updating the graph according to Equation \ref{eq:update_rule} yielding $\hat{\gG} = \gG_1 \sim p^\theta_1$. 

\subsection{Sampling and Generation}

Using our DFM model $p_\theta(\gG^\ell \mid \gH^\ell)$ together with the deterministic expansion operator $\gH^{\ell} = U(\gG^{\ell+1})$, graphs can be generated auto-regressively across resolution levels according to Equation~\ref{eq:model}. The complete sampling procedure is detailed in Algorithm \ref{algo:sample} in Appendix \ref{ap:sampling_algo}.

While the hierarchical process naturally generates graphs from pure noise, our framework also supports two additional generation settings based on conditioning on a spanning supergraph, either sampled from data or fixed.

First, rather than conditioning on the expanded graph $\hat{\gH}^0 = U(\hat{\gG}^{1})$, one may directly sample 
$\gH^0$ from the dataset, i.e. graphs obtained during preprocessing after a single coarsening–expansion. By adjusting the coarsening ratio during preprocessing, we explicitly control the sparsity of $\gH^0$
and, consequently, the degrees of freedom of the generative model. Moreover, this approach eliminates the need to train models at multiple levels.

Second, fixing $\gH^0$ enables a novel form of conditional generation that constrains $\gG^0$ to its subgraphs. Since $\gH^0$ preserves the eigenvalues of the weighted graph encoded in 
$\gG^1$, this can be interpreted as a form of spectral conditioning. Empirically, we show that this strategy is highly effective for generating graphs that remain close to a reference graph in terms of spectral distance.

\section{Experiments}\label{sec:eval}

In this section, we evaluate our method on multiple benchmarks, including molecular, synthetic, and real-world graph datasets. We show that our Hierarchical Discrete Flow Matching (HDFM) model captures graph distributions more accurately than existing approaches, while achieving substantially higher sampling efficiency.
Our model achieve this thanks to two key factors:
(i) the exploitation of sparsity in HDFM, which leads to significantly faster sampling compared to its dense counterpart (DFM); and
(ii) the ability to operate effectively with a small number of function evaluations (NFE).
We further evaluate our approach under two additional settings: conditioning on real spanning supergraphs and conditional graph generation.

\subsection{Unconditional Graph Generation}

At a high level, our experimental results lead to three main observations:
(i) Our dense Discrete Flow Matching (DFM) model consistently outperforms all baselines, while the hierarchical variant further improves modeling performances.
(ii) For a fixed NFE, ourr hierarchical approach achieves significantly faster sampling than dense methods.
(iii) Our hierarchical model enables low-NFE sampling regimes with little to no degradation in modeling performance.

We compare our approach against four main baseline models: \textsc{DiGress} \citep{digress}, a reference discrete diffusion model for graph generation; \textsc{SparseDiff} \citep{sparsediff}, a scalable variant of \textsc{DiGress}; \textsc{SID} \citep{sid}, a strong non-Markovian denoising-based method; and EDGE, a scalable approach for unattributed graph generation.
On the \texttt{ZINC250} dataset, we additionally report results for \textsc{Grum} \citep{drum} and \textsc{GraphBFN} \citep{graphBFN}, two competitive baselines, with Grum based on diffusion bridges and \textsc{GraphBFN} built upon Bayesian Flow Networks. In the following, we present detailed experimental results on molecular graphs, synthetic graphs, and large-scale real-world graph datasets.

We use standard metrics for molecular and unattributed graphs. All results are averages over three runs, with standard deviations reported in the appendix. Due to space limitation, full experimental details including additional evaluation metrics, standard deviations, dataset descriptions, are provided in Appendix \ref{ap:eval}. 

\subsubsection{Molecule Generation}

We evaluate our approach on two molecular graph benchmarks, \texttt{QM9H} and \texttt{ZINC250k}. \texttt{QM9H} is a variant of the standard QM9 dataset that includes explicit hydrogen atoms. This setting is considerably more challenging, as molecules can contain up to 29 atoms. \texttt{ZINC250k} is a widely used benchmark for molecular graph generation.

\setlength{\tabcolsep}{3pt}
\begin{table}[ht]
\scriptsize
\begin{center}
\begin{sc}
\caption{\footnotesize Generation results on the \texttt{QM9H} dataset datasets. \hspace{\textwidth}NSPDK results $\times 10^{3}$. Generation time in seconds.}

        \begin{tabular}{ c c  c c c c}
         \toprule
         Model & NFE &  FCD $\downarrow$ & NSPDK$\downarrow$ & Valid \%$\uparrow$ & Time (s) $\downarrow$  \\
        \midrule
        Train. & -- &
        $0.062$ & 
        $0.121$ &
        $98.90$ & -- \\

        DiGress & 1000 & --  & -- &
        $95.4\phantom{0}$ & -- \\ 

        DiscDif & 128 &
        $4.27$ &
        $41.93\phantom{0}$ &
        $22.29$ & -- 
        \\

        SID & 128 &
        $0.37$ &
        $1.15$ &
        $\textbf{97.97}$ & -- 
        \\
        \midrule
        
        DFM & 128 &  $0.25$ & $\textbf{0.44}$ & $97.17$ & $842.5$ \\
    
        HDFM & 128 &  $\textbf{0.23}$ & $0.53$ & $95.78 $& $333.5$ \\ 
        HDFM & 32&  
        $0.24$ & $0.47$ & $94.22$  & $\textbf{174.1}$ \\
        \bottomrule
        \end{tabular}
        \label{tab:qm9H}
\end{sc}
\end{center}
\end{table}

\setlength{\tabcolsep}{3pt}
\begin{table}[ht]
\scriptsize
\begin{center}
\begin{sc}
\caption{\footnotesize Generation results on the \texttt{Zinc250k} datasets. NSPDK results $\times 10^{3}$. Generation time in seconds.}
        \begin{tabular}{ c c c c c c}
         \toprule
         Model & NFE &  FCD $\downarrow$ & NSPDK$\downarrow$ & Valid \%$\uparrow$ & Time (s) $\downarrow$  \\
        \midrule
        Train. & -- &
        $1.13$ &
        $0.10$ &
        $100.00\phantom{0}$ & -- \\

        DiGress & 1000 &
        $3.48$ &
        $2.1\phantom{0}$ &
        $94.99$ &
        $5517$ 
        \\

        SparseDiff & 2000 & 
        15.82\phantom{0} & 
        55.1\phantom{0}\phantom{0} &
        76.05  & 9826\\

        Grum & 1000 &
        2.26 &
        1.5\phantom{0}&
        98.65 & --
        \\
 
        GraBFN & &
        2.12 &
        1.3\phantom{0} &
        99.22 & --
        \\ 
        
        SID & 128 &
        2.06 &
        2.01 &
        \textbf{99.50} & --
        \\

        DeFog 
        & 500 &
          1.43 & 0.8 & 99.22 & -- \\ 
        \midrule

        DFM & 128 & 1.48 & 0.55 & 99.41& 1999 \\
        
        HDFM & 128 & \textbf{1.23} & \textbf{0.43} & 99.30 & 394 \\
        HDFM & 64 & 1.26 & 0.50 & 99.17 & 239 \\ 
        HDFM & 32 &  1.34 & 0.64 & 98.82 & \textbf{168} \\
        \midrule
        \color{darkgray}
        HDFM-R & \color{darkgray}128 & \color{darkgray}1.25 &\color{darkgray} 0.376 &\color{darkgray} 99.28 & \color{darkgray}237  \\


        
        \bottomrule
        \end{tabular}
        \label{tab:zinc}
\end{sc}
\end{center}
\end{table}

Across both datasets, our models (DFM and HDFM) consistently outperform the baselines on FCD and NSPDK, indicating improved modeling of both chemical and structural distributions. On smaller molecules (\texttt{QM9H}), the hierarchical formulation does not yield improvements in modeling performance, but it substantially accelerates sampling. In contrast, on larger molecules (\texttt{ZINC250k}), the hierarchical approach demonstrates its effectiveness by significantly improving both FCD and NSPDK over the dense variant, while maintaining high efficiency.

\subsection{Synthetic Graph Generation}

Stochastic Block Models (SBMs) are commonly used to generate synthetic graphs with community structure, and the standard SBM benchmark is widely adopted for evaluating graph generation methods. 
However, the standard SBM dataset contains only 200 graphs, which leads to two major issues. First, the small dataset size makes overfitting and dataset memorization likely. Second, the limited test set results in high variance of evaluation metrics across runs. As a consequence, it becomes difficult to assess whether a model genuinely learns the underlying distribution or merely memorizes the dataset, often leading to inconsistent or misleading results. For example, some methods report distances lower than the reference distances between the training and test sets.
To address these limitations, we introduce a larger variant of the dataset generated using the same SBM process, but containing 20,000 graphs, denoted as \texttt{SBM20k}. The dataset is publicly available (see Appendix \ref{ap:eval}).

\setlength{\tabcolsep}{3pt}
\begin{table}[ht]
\scriptsize
\begin{center}
\begin{sc}
\caption{\footnotesize Generation results on the \texttt{SBM20k} datasets. \hspace{\textwidth}MMDs results $\times 10^{3}$. Generation time in seconds.}
        \begin{tabular}{ c c  c c c c c }
          \toprule
          
         Model & NFE & Valid$\uparrow$ &Deg. $\downarrow$ & Clust.$\downarrow$ & Spec. $\uparrow$ & Time (s) $\downarrow$  \\
        \midrule

        Train & -- & 90.7 & 0.08 & 1.97 & 0.17 & -- \\
        EDGE & 128 & -- & 154.16\phantom{00} & 766.79\phantom{00} & 10.58 & \phantom{00}350  \\
        DiGress & 1000 & -- & 4.85 & 5.92 & 28.30 & \phantom{0}3906  \\
        SparseDiff & 4000 & -- & 1.10 & 5.25 & 19.11 & 11089  \\ 
        SID & 128 & 0.00 &  $30.95\phantom{0}$ & $38.79\phantom{0}$ & $13.32$ & $2840$  \\ 
        \midrule
        DFM & 128 & 51.63 & 4.81 & 4.92 & \phantom{0}1.00 & \phantom{0}2834  \\
        HDFM & 128 & \textbf{70.40} & \textbf{0.15} & \textbf{2.06} & \textbf{\phantom{0}0.30} & \phantom{00}490 \\
        HDFM & 32 & 68.80 & 0.24 & \textbf{2.06} & \textbf{\phantom{0}0.30} & \textbf{\phantom{00}160} \\
        \midrule
        \color{darkgray}HDFM-R  & \color{darkgray}128 & \color{darkgray}85.00 &  \color{darkgray}0.18 & \color{darkgray}2.13 & \color{darkgray}\phantom{0}0.33 & \color{darkgray}\phantom{00}188 \\
        
        \bottomrule
        \end{tabular}
        \label{tab:SBM}
\end{sc}
\end{center}
\end{table}

The results support our main claims. Our models significantly outperform all baselines by a large margin; the hierarchical variant consistently improves modeling performance over its dense counterpart; and the approach remains robust in low-NFE sampling regimes. On this larger dataset, the gains in generation time are even more pronounced.

\subsection{Large Social Networks}

Finally, we evaluate our approach on two large-scale graph datasets, namely \texttt{Ego} and \texttt{Reddit12K}, which contain 757 and 11551 graphs with upto 399 and 1500 nodes, respectively. Graphs from both dataset represent social networks.  
The \texttt{Ego} dataset is a standard benchmark for large graph generation, containing 757 graphs with up to 400 nodes. Baseline results for this dataset are taken from \citet{sparsediff}. The \texttt{Reddit12K} dataset is derived from the popular TUDataset collection. 

While the \texttt{Ego} dataset is a standard benchmark for large graph generation, we also introduce \texttt{Reddit12K} as a new and challenging benchmark for large graph generation due to its number of graphs and their sizes. In particular, dense denoising-based generative models typically do not support graphs of this size. Baseline results are obtained by adapting the source code from the respective official repositories. Results are reported in Tables~\ref{tab:ego} and~\ref{tab:reddit}.

On the \texttt{Ego} dataset, the performance gap between our models and the baselines is smaller than on other benchmarks. We attribute this behavior to the relatively small number of graphs available for training. In particular, our hierarchical model requires learning the distribution of coarse-level graphs, which may be more challenging in low-data regimes. Nevertheless, our approach still outperforms all baselines on two out of the three evaluation metrics.

\setlength{\tabcolsep}{3pt}
\begin{table}[ht]
\scriptsize
\centering
\begin{sc}
\caption{\footnotesize Generation results on the \texttt{Ego} datasets. \hspace{\textwidth}MMDs results $\times 10^{3}$. Generation time in seconds.}
    \begin{tabular}{
    l  c c c c 
    }
    \toprule
    
    Model &
    NFE &
    Degree $\downarrow$ &
    Cluster. $\downarrow$ &
    Spect. $\downarrow$ 
    \\
    \midrule

    Train. set & --
    & $0.7 \pm 0.4$ 
    & $14.6 \pm 1.6$
    & $2.2 \pm 0.7$
    \\
    
    HiGen & -- 
    & $47.\phantom{0 \pm 0.00}$
    & $3\phantom{ \pm 0.000}$
    & --
    \\
    
    EDGE & 128 
    & $58.\phantom{0 \pm 0.00}$
    & $180\phantom{ \pm 0.0000.}$
    & --
    \\
    
    DiGress & 1000
    & $8.9 \pm 1.6$
    & $54\phantom{.0} \pm 4\phantom{.0}$
    & $19\phantom{.0} \pm 3.2\phantom{0}$
    \\
    
    SparseDiff & 10000
    & $3.7 \pm 0.4$
    & \textbf{32\phantom{.0}} $\pm$ 1\phantom{.0}
    & $5.6 \pm 0.8$ 
    \\
    \midrule
    
    HDFM & 128 &
    1.4 $\pm$ 0.1 & 57.7 $\pm$ 1.3 & \textbf{3.3} $\pm$ 0.4 
    \\
    
    
    HDFM & 32 &
    \textbf{1.3} $\pm$ 0.1 & 57.7 $\pm$ 3.6 & 3.4 $\pm$ 0.3  \\ 
    
    
    \bottomrule
    \end{tabular}
    \label{tab:ego}
\end{sc}
\end{table}

Few existing methods support the graph sizes present in the \texttt{Reddit12K} dataset, and we were therefore only able to evaluate two baseline approaches. Under this constraint, our approach yields substantially better results than these baselines. Moreover, we again observe the strong sampling efficiency of our generative models, which are able to generate large graphs using only a small number of denoising steps.

\setlength{\tabcolsep}{3pt}
\begin{table}[ht]
\scriptsize
\begin{center}
\begin{sc}

\caption{\footnotesize Generation results on the \texttt{Reddit12k} datasets. \hspace{\textwidth}MMDs results $\times 10^{3}$. Generation time in seconds}

\begin{tabular}{c c c c c c c}

 \toprule
 Model & NFE &  degree$\downarrow$ & clust. $\downarrow$ & spect.$\downarrow$ & Time$\downarrow$  \\

\midrule
Train. & -- & 0.17 & 5.17 & 1.51 \\
EDGE & 128 & 154.16 & 766.79\hphantom{0}  & 37.63& 172 \\
SparseDiff & 10000 &\hphantom{0}93.73 & 68.46 &  124.96& 7234  \\
\midrule
HDFM & 128 & \hphantom{00}$0.32$ & $3.35$ & $3.34$ & $107$ \\
HDFM & 32 & \hphantom{00}$0.40$ & $2.31$ & $3.62$ & $47$ \\ 
\bottomrule
\end{tabular}
\label{tab:reddit}
\end{sc}
\end{center}
\end{table}
\subsection{Conditional Generation}

We now evaluate our model under the two additional settings presented in Section \ref{sec:model}: either generating graphs by randomly sampling expanded spanning supergraph from the dataset or by conditioning the graph structure by fixing the spanning supergraph

\subsubsection{Generation Conditioned on Real Spanning Supergraph}

As discussed in Section \ref{sec:model}, graph generation can alternatively be conditioned on a real expanded spanning supergraph sampled from the dataset, rather than proceeding hierarchically from coarse to fine. In this setting, only the finest level is generated, with the spanning supergraph provided as input.

We do not directly compare this variant to the fully generative hierarchical model, as conditioning on ground-truth structure makes the comparison unfair. Nonetheless, this setting is practically relevant, as it requires training and sampling at only a single level and avoids error accumulation across hierarchical stages. From an experimental standpoint, it also serves as an ablation that isolates the modeling error introduced by the hierarchical abstraction. We report results for this setting under the name HDFM-R on the \texttt{ZINC250k} and \texttt{SBM20k} datasets.

On \texttt{ZINC250k}, conditioning on the true spanning supergraph does not yield significant improvements in modeling performance, but leads to faster generation by bypassing the lower-level sampling stage. In contrast, on the more challenging \texttt{SBM20k} dataset, conditioning on real expanded spanning supergraphs results in significant improvement in validity. We attribute this gain to the increased difficulty of modeling coarse-level graphs in SBM, which can introduce discrepancies between real and generated spanning supergraphs in the fully generative setting.

\subsection{Structural Conditioning}

Fixing the spanning supergraph enables a novel form of structural conditioning. We evaluate this setting by computing the spectral distance between a reference graph from the dataset and graphs generated while conditioning on the spanning supergraph associated with that reference graph (see Appendix~\ref{ap:cond_gen} for details). For comparison, Table~\ref{tab:conditional} also reports the spectral distance between the reference graph and a graph randomly sampled from the test set.
The results show that conditioning on the spanning supergraph effectively guides graph generation toward the desired global structure.

\begin{table}[ht]
    \begin{center}
    \scriptsize
    \begin{sc}
    \caption{\footnotesize Average spectral distances to reference graphs }
    \begin{tabular}{c  c c c}
    \toprule
       Dataset & Test set & Conditional Gen. & Ratio Cond. / Test \\
    \midrule
        SBM20k &  4.57 $\pm$ 1.55 & \textbf{0.49 $\pm$ 0.06}  & 0.114 $\pm$ 0.024 \\
        Reddit &  3.57 $\pm$ 1.10 & \textbf{0.57 $\pm$ 0.31}  & 0.161 $\pm$ 0.069 \\
    \bottomrule
    \end{tabular}
    \label{tab:conditional}
    \end{sc}
    \end{center}
\end{table}

\section{Conclusion}\label{sec:conclusion}

We propose a novel hierarchical approach to graph generation based on discrete flow matching. Our method substantially improves graph modeling performance across a wide range of datasets, while maintaining strong computational efficiency that enables fast sampling and scalability to large graphs.

These gains arise from the newly proposed graph coarsening strategy, which induces sparse spanning supergraphs that serve as effective conditioning structures for both accurate and efficient graph generation. More broadly, our results highlight the critical role of architectural design in scalable and high-fidelity graph generative modeling.

\section*{Impact Statement}

This paper presents research aimed at advancing the field of machine learning, with a particular focus on graph and molecular generation. The potential societal impacts of this work are largely positive, and we do not identify any direct specific negative consequences that require special discussion.

\bibliography{ref}

@inproceedings{graphgen,
  author    = {Nikhil Goyal and
               Harsh Vardhan Jain and
               Sayan Ranu},
  editor    = {Yennun Huang and
               Irwin King and
               Tie{-}Yan Liu and
               Maarten van Steen},
  title     = {GraphGen: {A} Scalable Approach to Domain-agnostic Labeled Graph Generation},
  booktitle = {{WWW} '20: The Web Conference 2020, Taipei, Taiwan, April 20-24, 2020},
  pages     = {1253--1263},
  publisher = {{ACM} / {IW3C2}},
  year      = {2020},
  url       = {https://doi.org/10.1145/3366423.3380201},
  doi       = {10.1145/3366423.3380201},
  bibsource = {dblp computer science bibliography, https://dblp.org}
}

@inproceedings{graphrnn,
  title = {{GraphRNN}: {Generating} {Realistic} {Graphs} with {Deep} {Auto}-regressive {Models}},
  shorttitle = {{GraphRNN}},
  url = {https://proceedings.mlr.press/v80/you18a.html},
  language = {en},
  urldate = {2022-03-03},
  booktitle = {Proceedings of the 35th {International} {Conference} on {Machine} {Learning}},
  publisher = {PMLR},
  author = {You, Jiaxuan and Ying, Rex and Ren, Xiang and Hamilton, William and Leskovec, Jure},
  month = jul,
  year = {2018},
  note = {ISSN: 2640-3498},
  pages = {5708--5717},
}

@inproceedings{kusner_grammar_2017,
  title = {Grammar {Variational} {Autoencoder}},
  url = {https://proceedings.mlr.press/v70/kusner17a.html},
  language = {en},
  urldate = {2022-04-03},
  booktitle = {Proceedings of the 34th {International} {Conference} on {Machine} {Learning}},
  publisher = {PMLR},
  author = {Kusner, Matt J. and Paige, Brooks and Hernández-Lobato, José Miguel},
  month = jul,
  year = {2017},
  note = {ISSN: 2640-3498},
  pages = {1945--1954}
}

@misc{graphnvp,
      title={GraphNVP: An Invertible Flow Model for Generating Molecular Graphs}, 
      author={Kaushalya Madhawa and Katushiko Ishiguro and Kosuke Nakago and Motoki Abe},
      year={2019},
      eprint={1905.11600},
      archivePrefix={arXiv},
      primaryClass={stat.ML}
}

@inproceedings{graphaf,
title={GraphAF: a Flow-based Autoregressive Model for Molecular Graph Generation},
author={Chence Shi and Minkai Xu and Zhaocheng Zhu and Weinan Zhang and Ming Zhang and Jian Tang},
booktitle={International Conference on Learning Representations},
year={2020},
url={https://openreview.net/forum?id=S1esMkHYPr}
}

@article{gomez-bombarelli_automatic_2018,
  title = {Automatic {Chemical} {Design} {Using} a {Data}-{Driven} {Continuous} {Representation} of {Molecules}},
  volume = {4},
  issn = {2374-7943},
  url = {https://doi.org/10.1021/acscentsci.7b00572},
  number = {2},
  urldate = {2022-04-03},
  journal = {ACS Central Science},
  author = {Gómez-Bombarelli, Rafael and Wei, Jennifer N. and Duvenaud, David and Hernández-Lobato, José Miguel and Sánchez-Lengeling, Benjamín and Sheberla, Dennis and Aguilera-Iparraguirre, Jorge and Hirzel, Timothy D. and Adams, Ryan P. and Aspuru-Guzik, Alán},
  month = feb,
  year = {2018},
  note = {Publisher: American Chemical Society},
  pages = {268--276}
}

@inproceedings{gran,
	title = {Efficient {Graph} {Generation} with {Graph} {Recurrent} {Attention} {Networks}},
	volume = {32},
	urldate = {2022-04-06},
	booktitle = {Advances in {Neural} {Information} {Processing} {Systems}},
	publisher = {Curran Associates, Inc.},
	author = {Liao, Renjie and Li, Yujia and Song, Yang and Wang, Shenlong and Hamilton, Will and Duvenaud, David K and Urtasun, Raquel and Zemel, Richard},
	year = {2019}
}

@inproceedings{edp-gnn,
author = {Yang, Carl and Zhuang, Peiye and Shi, Wenhan and Luu, Alan and Li, Pan},
booktitle = {Advances in Neural Information Processing Systems},
title = {{Conditional Structure Generation through Graph Variational Generative Adversarial Nets}},
volume = {32},
year = {2019}
}

@inproceedings{hiervae,
author = {Jin, Wengong and Barzilay, Regina and Jaakkola, Tommi},
booktitle = {37th International Conference on Machine Learning, ICML 2020},
isbn = {9781713821120},
pages = {4789--4798},
title = {{Hierarchical Generation of Molecular Graphs using Structural Motifs}},
url = {https://github.com/wengong-jin/hgraph2graph},
volume = {PartF16814},
year = {2020}
}

@InProceedings{jtvae,
  title = 	 {Junction Tree Variational Autoencoder for Molecular Graph Generation},
  author =       {Jin, Wengong and Barzilay, Regina and Jaakkola, Tommi},
  booktitle = 	 {Proceedings of the 35th International Conference on Machine Learning},
  pages = 	 {2323--2332},
  year = 	 {2018},
  editor = 	 {Dy, Jennifer and Krause, Andreas},
  volume = 	 {80},
  series = 	 {Proceedings of Machine Learning Research},
  month = 	 {10--15 Jul},
  publisher =    {PMLR},
  url = 	 {https://proceedings.mlr.press/v80/jin18a.html}
}

@article{moflow,
archivePrefix = {arXiv},
arxivId = {2006.10137},
author = {Zang, Chengxi and Wang, Fei},
doi = {10.1145/3394486.3403104},
eprint = {2006.10137},
isbn = {9781450379984},
journal = {Proceedings of the ACM SIGKDD International Conference on Knowledge Discovery and Data Mining},
month = {aug},
pages = {617--626},
publisher = {Association for Computing Machinery},
title = {{MoFlow: An Invertible Flow Model for Generating Molecular Graphs}},
url = {https://dl.acm.org/doi/10.1145/3394486.3403104},
volume = {10},
year = {2020}
}

@article{graphdf,
archivePrefix = {arXiv},
arxivId = {2102.01189},
author = {Luo, Youzhi and Yan, Keqiang and Ji, Shuiwang},
eprint = {2102.01189},
journal = {Proceedings of the 38th International Conference on Machine Learning},
pages = {7192--7203},
title = {{GraphDF: A Discrete Flow Model for Molecular Graph Generation}},
url = {http://arxiv.org/abs/2102.01189},
volume = {139},
year = {2021}
}

@inproceedings{gnf,
archivePrefix = {arXiv},
arxivId = {1905.13177},
author = {Liu, Jenny and Kumar, Aviral and Ba, Jimmy and Kiros, Jamie and Swersky, Kevin},
booktitle = {Advances in Neural Information Processing Systems},
eprint = {1905.13177},
issn = {10495258},
title = {{Graph normalizing flows}},
volume = {32},
year = {2019}
}

@article{gdss,
archivePrefix = {arXiv},
arxivId = {2202.02514},
author = {Jo, Jaehyeong and Lee, Seul and Hwang, Sung Ju},
eprint = {2202.02514},
journal = {Proceedings of the 39th International Conference on Machine Learning},
pages = {10362--10383},
title = {{Score-based Generative Modeling of Graphs via the System of Stochastic Differential Equations}},
url = {https://github.com/harryjo97/GDSS. http://arxiv.org/abs/2202.02514},
volume = {162},
year = {2022}
}

@article{Chen2018,
archivePrefix = {arXiv},
arxivId = {1809.00773},
author = {Chen, Bo and Sun, Le and Han, Xianpei},
doi = {10.18653/V1/P18-1071},
eprint = {1809.00773},
isbn = {9781948087322},
journal = {ACL 2018 - 56th Annual Meeting of the Association for Computational Linguistics, Proceedings of the Conference (Long Papers)},
pages = {766--777},
publisher = {Association for Computational Linguistics (ACL)},
title = {{Sequence-to-Action: End-to-End Semantic Graph Generation for Semantic Parsing}},
url = {https://aclanthology.org/P18-1071},
volume = {1},
year = {2018}
}

@article{Klawonn2018,
archivePrefix = {arXiv},
arxivId = {1802.02598},
author = {Klawonn, Matthew and Heim, Eric},
doi = {10.1609/AAAI.V32I1.12321},
eprint = {1802.02598},
isbn = {9781577358008},
issn = {2374-3468},
journal = {Proceedings of the AAAI Conference on Artificial Intelligence},
month = {apr},
number = {1},
pages = {6992--6999},
publisher = {AAAI press},
title = {{Generating Triples With Adversarial Networks for Scene Graph Construction}},
url = {https://ojs.aaai.org/index.php/AAAI/article/view/12321},
volume = {32},
year = {2018}
}

@article{Lu2020,
archivePrefix = {arXiv},
arxivId = {1910.13551},
author = {Lu, Chengqiang and Liu, Qi and Sun, Qiming and Hsieh, Chang Yu and Zhang, Shengyu and Shi, Liang and Lee, Chee Kong},
eprint = {1910.13551},
issn = {19327455},
journal = {Journal of Physical Chemistry C},
month = {apr},
number = {13},
pages = {7048--7060},
publisher = {American Chemical Society},
title = {{Deep Learning for Optoelectronic Properties of Organic Semiconductors}},
url = {https://pubs.acs.org/doi/abs/10.1021/acs.jpcc.0c00329},
volume = {124},
year = {2020}
}

@inproceedings{
brockschmidt,
title={Generative Code Modeling with Graphs},
author={Marc Brockschmidt and Miltiadis Allamanis and Alexander L. Gaunt and Oleksandr Polozov},
booktitle={International Conference on Learning Representations},
year={2019},
url={https://openreview.net/forum?id=Bke4KsA5FX},
}

@inproceedings{Ingraham,
 author = {Ingraham, John and Garg, Vikas and Barzilay, Regina and Jaakkola, Tommi},
 booktitle = {Advances in Neural Information Processing Systems},
 editor = {H. Wallach and H. Larochelle and A. Beygelzimer and F. d\textquotesingle Alch\'{e}-Buc and E. Fox and R. Garnett},
 publisher = {Curran Associates, Inc.},
 title = {Generative Models for Graph-Based Protein Design},
 volume = {32},
 year = {2019}
}

@INPROCEEDINGS {Li,
author = {Y. Li and W. Ouyang and B. Zhou and K. Wang and X. Wang},
booktitle = {2017 IEEE International Conference on Computer Vision (ICCV)},
title = {Scene Graph Generation from Objects, Phrases and Region Captions},
year = {2017},
issn = {2380-7504},
pages = {1270-1279},
doi = {10.1109/ICCV.2017.142},
url = {https://doi.ieeecomputersociety.org/10.1109/ICCV.2017.142},
publisher = {IEEE Computer Society},
address = {Los Alamitos, CA, USA},
month = {oct}
}

@article{diffpool,
author = {Ying, Rex and You, Jiaxuan and Morris, Christopher and Ren, Xiang and Hamilton, William L and Leskovec, Jure},
journal = {Advances in Neural Information Processing Systems},
title = {{Hierarchical Graph Representation Learning with Differentiable Pooling}},
volume = {31},
year = {2018}
}

@article{attention_is_all,
archivePrefix = {arXiv},
arxivId = {1706.03762v5},
author = {Vaswani, Ashish and Shazeer, Noam and Parmar, Niki and Uszkoreit, Jakob and Jones, Llion and Gomez, Aidan N and Kaiser, {\L}ukasz and Polosukhin, Illia},
eprint = {1706.03762v5},
issn = {10495258},
journal = {Advances in Neural Information Processing Systems 30},
pages = {5998--6008},
title = {{Transformer: Attention is all you need}},
year = {2017}
}

@article{zinc,
author = {Irwin, John J. and Sterling, Teague and Mysinger, Michael M. and Bolstad, Erin S. and Coleman, Ryan G.},
title = {ZINC: A Free Tool to Discover Chemistry for Biology},
journal = {Journal of Chemical Information and Modeling},
volume = {52},
number = {7},
pages = {1757-1768},
year = {2012},
doi = {10.1021/ci3001277},
    note ={PMID: 22587354},
url = {https://doi.org/10.1021/ci3001277},
eprint = {https://doi.org/10.1021/ci3001277}
}

@article{fcd,
author = {Preuer, Kristina and Renz, Philipp and Unterthiner, Thomas and Hochreiter, Sepp and Klambauer, Günter},
title = {Fréchet ChemNet Distance: A Metric for Generative Models for Molecules in Drug Discovery},
journal = {Journal of Chemical Information and Modeling},
volume = {58},
number = {9},
pages = {1736-1741},
year = {2018},
doi = {10.1021/acs.jcim.8b00234},
    note ={PMID: 30118593},
}

@inproceedings{nspdk,
  author    = {Fabrizio Costa and
               Kurt De Grave},
  editor    = {Johannes F{\"{u}}rnkranz and
               Thorsten Joachims},
  title     = {Fast Neighborhood Subgraph Pairwise Distance Kernel},
  booktitle = {Proceedings of the 27th International Conference on Machine Learning
               (ICML-10), June 21-24, 2010, Haifa, Israel},
  pages     = {255--262},
  publisher = {Omnipress},
  year      = {2010},
  url       = {https://icml.cc/Conferences/2010/papers/347.pdf},
  bibsource = {dblp computer science bibliography, https://dblp.org}
}

@inproceedings{digress,
title={DiGress: Discrete Denoising diffusion for graph generation},
author={Clement Vignac and Igor Krawczuk and Antoine Siraudin and Bohan Wang and Volkan Cevher and Pascal Frossard},
booktitle={The Eleventh International Conference on Learning Representations },
year={2023},
url ={https://openreview.net/forum?id=UaAD-Nu86WX}
}

@misc{gggan,
title={{\{}GG{\}}-{\{}GAN{\}}: A Geometric Graph Generative Adversarial Network},
author={Igor Krawczuk and Pedro Abranches and Andreas Loukas and Volkan Cevher},
year={2021},
url={https://openreview.net/forum?id=qiAxL3Xqx1o}
}

@InProceedings{grapharm,
  title = 	 {Autoregressive Diffusion Model for Graph Generation},
  author =       {Kong, Lingkai and Cui, Jiaming and Sun, Haotian and Zhuang, Yuchen and Prakash, B. Aditya and Zhang, Chao},
  booktitle = 	 {Proceedings of the 40th International Conference on Machine Learning},
  pages = 	 {17391--17408},
  year = 	 {2023},
  editor = 	 {Krause, Andreas and Brunskill, Emma and Cho, Kyunghyun and Engelhardt, Barbara and Sabato, Sivan and Scarlett, Jonathan},
  volume = 	 {202},
  series = 	 {Proceedings of Machine Learning Research},
  month = 	 {23--29 Jul},
  publisher =    {PMLR},
  url = 	 {https://proceedings.mlr.press/v202/kong23b.html}
}

@inproceedings{
graphle,
title={Efficient and Scalable Graph Generation through Iterative Local Expansion},
author={Andreas Bergmeister and Karolis Martinkus and Nathana{\"e}l Perraudin and Roger Wattenhofer},
booktitle={The Twelfth International Conference on Learning Representations},
year={2024},
url = {https://openreview.net/forum?id=2XkTz7gdpc}
}

@InProceedings{spectre,
  title = 	 {{SPECTRE}: Spectral Conditioning Helps to Overcome the Expressivity Limits of One-shot Graph Generators},
  author =       {Martinkus, Karolis and Loukas, Andreas and Perraudin, Nathana{\"e}l and Wattenhofer, Roger},
  booktitle = 	 {Proceedings of the 39th International Conference on Machine Learning},
  pages = 	 {15159--15179},
  year = 	 {2022},
  editor = 	 {Chaudhuri, Kamalika and Jegelka, Stefanie and Song, Le and Szepesvari, Csaba and Niu, Gang and Sabato, Sivan},
  volume = 	 {162},
  series = 	 {Proceedings of Machine Learning Research},
  month = 	 {17--23 Jul},
  publisher =    {PMLR},
  url = 	 {https://proceedings.mlr.press/v162/martinkus22a.html}
}

@inproceedings{reddit,
author = {Yanardag, Pinar and Vishwanathan, S.V.N.},
title = {Deep Graph Kernels},
year = {2015},
isbn = {9781450336642},
publisher = {Association for Computing Machinery},
address = {New York, NY, USA},
url = {https://doi.org/10.1145/2783258.2783417},
doi = {10.1145/2783258.2783417},
booktitle = {Proceedings of the 21th ACM SIGKDD International Conference on Knowledge Discovery and Data Mining},
pages = {1365–1374},
numpages = {10},
location = {Sydney, NSW, Australia},
series = {KDD '15}
}

@inproceedings{
discdiff_haefeli,
title={Diffusion Models for Graphs Benefit From Discrete State Spaces},
author={Kilian Konstantin Haefeli and Karolis Martinkus and Nathana{\"e}l Perraudin and Roger Wattenhofer},
booktitle={The First Learning on Graphs Conference},
year={2022},
url = {https://openreview.net/forum?id=CtsKBwhTMKg}
}

@misc{
sparsediff,
title={Sparse Training of Discrete Diffusion Models for Graph Generation},
author={Yiming Qin and Clement Vignac and Pascal Frossard},
year={2024},
url={https://openreview.net/forum?id=oTRekADULK}
}

@InProceedings{EDGE,
  title = 	 {Efficient and Degree-Guided Graph Generation via Discrete Diffusion Modeling},
  author =       {Chen, Xiaohui and He, Jiaxing and Han, Xu and Liu, Liping},
  booktitle = 	 {Proceedings of the 40th International Conference on Machine Learning},
  pages = 	 {4585--4610},
  year = 	 {2023},
  editor = 	 {Krause, Andreas and Brunskill, Emma and Cho, Kyunghyun and Engelhardt, Barbara and Sabato, Sivan and Scarlett, Jonathan},
  volume = 	 {202},
  series = 	 {Proceedings of Machine Learning Research},
  month = 	 {23--29 Jul},
  publisher =    {PMLR},
  url = 	 {https://proceedings.mlr.press/v202/chen23k.html}
}

@article{
dgae,
title={Discrete Graph Auto-Encoder},
author={Yoann Boget and Magda Gregorova and Alexandros Kalousis},
journal={Transactions on Machine Learning Research},
issn={2835-8856},
year={2024},
url={https://openreview.net/forum?id=bZ80b0wb9d},
note={}
}

@inproceedings{higen,
title={HiGen: Hierarchical Graph Generative Networks},
author={Mahdi Karami},
booktitle={The Twelfth International Conference on Learning Representations},
year={2024},
url={https://openreview.net/forum?id=KNvubydSB5}
}

@InProceedings{drum,
  title = 	 {Graph Generation with Diffusion Mixture},
  author =       {Jo, Jaehyeong and Kim, Dongki and Hwang, Sung Ju},
  booktitle = 	 {Proceedings of the 41st International Conference on Machine Learning},
  pages = 	 {22371--22405},
  year = 	 {2024},
  editor = 	 {Salakhutdinov, Ruslan and Kolter, Zico and Heller, Katherine and Weller, Adrian and Oliver, Nuria and Scarlett, Jonathan and Berkenkamp, Felix},
  volume = 	 {235},
  series = 	 {Proceedings of Machine Learning Research},
  month = 	 {21--27 Jul},
  publisher =    {PMLR},
  url = 	 {https://proceedings.mlr.press/v235/jo24b.html}
}

@inproceedings{glad,
title={{GLAD}: Improving Latent Graph Generative Modeling with Simple Quantization},
author={Van Khoa Nguyen and Yoann Boget and Frantzeska Lavda and Alexandros Kalousis},
booktitle={ICML 2024 Workshop on Structured Probabilistic Inference {\&} Generative Modeling},
year={2024},
url = {https://openreview.net/forum?id=aY1gdSolIv}
}

@inproceedings{graph_pooling_review2,
author = {Liu, Chuang and Zhan, Yibing and Wu, Jia and Li, Chang and Du, Bo and Hu, Wenbin and Liu, Tongliang and Tao, Dacheng},
title = {Graph pooling for graph neural networks: progress, challenges, and opportunities},
year = {2023},
isbn = {978-1-956792-03-4},
url = {https://doi.org/10.24963/ijcai.2023/752},
doi = {10.24963/ijcai.2023/752},
booktitle = {Proceedings of the Thirty-Second International Joint Conference on Artificial Intelligence},
articleno = {752},
numpages = {11},
location = {Macao, P.R.China},
series = {IJCAI '23}
}

@ARTICLE{graph_pooling_review1,
  author={Grattarola, Daniele and Zambon, Daniele and Bianchi, Filippo Maria and Alippi, Cesare},
  journal={IEEE Transactions on Neural Networks and Learning Systems}, 
  title={Understanding Pooling in Graph Neural Networks}, 
  year={2024},
  volume={35},
  number={2},
  pages={2708-2718},
  doi={10.1109/TNNLS.2022.3190922}}

@ARTICLE{graph_pooling_review3,
  title = 	 {Graph pooling in graph neural networks: methods and their applications in omics studies},
  author =       {Wang, Yan and Hou, Wenju and Sheng, Nan and  Zhao, Ziqi and Jialin, Liu and Lan, Huang and Wang, Juexin },
  journal={Artificial Intelligence Review}, 
  year={2024},
  volume={57},
  number={294},
  doi={10.1007/s10462-024-10918-9}}

@article{contraction_loukas,
  author  = {Andreas Loukas},
  title   = {Graph Reduction with Spectral and Cut Guarantees},
  journal = {Journal of Machine Learning Research},
  year    = {2019},
  volume  = {20},
  number  = {116},
  pages   = {1--42},
  url     = {http://jmlr.org/papers/v20/18-680.html}
}

@misc{edge_contractionpoolinggraph,
      title={Edge Contraction Pooling for Graph Neural Networks}, 
      author={Frederik Diehl},
      year={2019},
      eprint={1905.10990},
      archivePrefix={arXiv},
      primaryClass={cs.LG},
      url={https://arxiv.org/abs/1905.10990}, 
}

@misc{mincut,
      title={Spectral Clustering with Graph Neural Networks for Graph Pooling}, 
      author={Filippo Maria Bianchi and Daniele Grattarola and Cesare Alippi},
      year={2020},
      eprint={1907.00481},
      archivePrefix={arXiv},
      primaryClass={cs.LG},
      url={https://arxiv.org/abs/1907.00481}, 
}

@inproceedings{DDPM,
 author = {Ho, Jonathan and Jain, Ajay and Abbeel, Pieter},
 booktitle = {Advances in Neural Information Processing Systems},
 editor = {H. Larochelle and M. Ranzato and R. Hadsell and M.F. Balcan and H. Lin},
 pages = {6840--6851},
 publisher = {Curran Associates, Inc.},
 title = {Denoising Diffusion Probabilistic Models},
 url = {https://proceedings.neurips.cc/paper_files/paper/2020/file/4c5bcfec8584af0d967f1ab10179ca4b-Paper.pdf},
 volume = {33},
 year = {2020}
}

@inproceedings{Discrete_FlowMatching_campbell,
  title={Generative flows on discrete state-spaces: enabling multimodal flows with applications to protein co-design},
  author={Campbell, Andrew and Yim, Jason and Barzilay, Regina and Rainforth, Tom and Jaakkola, Tommi},
  booktitle={Proceedings of the 41st International Conference on Machine Learning},
  pages={5453--5512},
  year={2024}
}

@inproceedings{discrete_flow_matching_gat,
title={Discrete Flow Matching},
author={Itai Gat and Tal Remez and Neta Shaul and Felix Kreuk and Ricky T. Q. Chen and Gabriel Synnaeve and Yossi Adi and Yaron Lipman},
booktitle={The Thirty-eighth Annual Conference on Neural Information Processing Systems},
year={2024},
url = {https://openreview.net/forum?id=GTDKo3Sv9p}
}

@inproceedings{DeFog,
title={DeFoG: Discrete Flow Matching for Graph Generation},
author={Yiming Qin and Manuel Madeira and Dorina Thanou and Pascal Frossard},
booktitle={Forty-second International Conference on Machine Learning},
year={2025},
url={https://openreview.net/forum?id=KPRIwWhqAZ}
}

@inproceedings{continuous_time_graph_diffusion,
title={Discrete-state Continuous-time Diffusion for Graph Generation},
author={Zhe Xu and Ruizhong Qiu and Yuzhong Chen and Huiyuan Chen and Xiran Fan and Menghai Pan and Zhichen Zeng and Mahashweta Das and Hanghang Tong},
booktitle={The Thirty-eighth Annual Conference on Neural Information Processing Systems},
year={2024},
url={https://openreview.net/forum?id=YkSKZEhIYt}
}

@article{qm9,
title={MoleculeNet: a benchmark for molecular machine learning},
author={Zhenqin Wu and Bharath Ramsundar and Evan N Feinberg and, Joseph Gomes and Caleb Geniesse and Aneesh S Pappu and Karl Leswing and Vijay Pande},
journal={Chemical science},
year={2017},
url = {https://pmc.ncbi.nlm.nih.gov/articles/PMC5868307/}
}

@InProceedings{sid,
  title = 	 {Simple and Critical Iterative Denoising: A Recasting of Discrete Diffusion in Graph Generation},
  author =       {Boget, Yoann},
  booktitle = 	 {Proceedings of the 42th International Conference on Machine Learning},
  year = 	 {2025},
  series = 	 {Proceedings of Machine Learning Research},
  month = 	 {July},
  publisher =    {PMLR},
}

@inproceedings{BetaGraph,
  title={Advancing Graph Generation through Beta Diffusion},
  author={Liu, Xinyang and He, Yilin and Chen, Bo and Zhou, Mingyuan},
  booktitle={13th International Conference on Learning Representations (ICLR 2025)},
  year={2025}
}

@inproceedings{CatFlow,
title={Variational Flow Matching for Graph Generation},
author={Floor Eijkelboom and Grigory Bartosh and Christian A. Naesseth and Max Welling and Jan-Willem van de Meent},
booktitle={The Thirty-eighth Annual Conference on Neural Information Processing Systems},
year={2024},
url={https://openreview.net/forum?id=UahrHR5HQh}
}

@InProceedings{graphBFN,
  title = 	 {Smooth Interpolation for Improved Discrete Graph Generative Models},
  author =       {Song, Yuxuan and Shi, Juntong and Gong, Jingjing and Xu, Minkai and Ermon, Stefano and Zhou, Hao and Ma, Wei-Ying},
  booktitle = 	 {Proceedings of the 42nd International Conference on Machine Learning},
  pages = 	 {56363--56388},
  year = 	 {2025},
  editor = 	 {Singh, Aarti and Fazel, Maryam and Hsu, Daniel and Lacoste-Julien, Simon and Berkenkamp, Felix and Maharaj, Tegan and Wagstaff, Kiri and Zhu, Jerry},
  volume = 	 {267},
  series = 	 {Proceedings of Machine Learning Research},
  month = 	 {13--19 Jul},
  publisher =    {PMLR},
  url = 	 {https://proceedings.mlr.press/v267/song25f.html},
}

@article{DMoN,
  author  = {Anton Tsitsulin and John Palowitch and Bryan Perozzi and Emmanuel MÃ¼ller},
  title   = {Graph Clustering with Graph Neural Networks},
  journal = {Journal of Machine Learning Research},
  year    = {2023},
  volume  = {24},
  number  = {127},
  pages   = {1--21},
  url     = {http://jmlr.org/papers/v24/20-998.html}
}

@misc{simGFM,
title={Sim{GFM}: Simplifying Discrete Flow Matching for Graph Generation},
author={Anonymous},
year={2026},
url={https://openreview.net/forum?id=eCftYujQkK}
}

@misc{unside,
title={Unrestrained Simplex Denoising for Discrete Data. A Non-Markovian Approach Applied to Graph Generation},
author={Anonymous},
year={2026},
url={https://openreview.net/forum?id=AKXeom7KH5}
}
\bibliographystyle{icml2026}

\newpage
\appendix
\onecolumn
\section{Proofs}\label{ap:proofs}

\subsection{Proposition \ref{prop:spanning}}\label{ap:proof_spanning}

\textbf{Proposition:}

The expanded graph $\gH^\ell$ is a spanning supergraph of $\gG^\ell_\text{unattr.}$, that is:
\begin{equation}
    \gH^\ell = U(C(\gG^\ell_\text{unattr.}))  \implies \gH^\ell = \gG^\ell + \gE_\gS     
\end{equation}

\textbf{Proof:}

We represent the graphs \(\gG^\ell_{\mathrm{unattr.}}\) and \(\gH^\ell\) by their respective adjacency matrices
\begin{equation}
\gG^\ell_{\mathrm{unattr.}} = \mA^\ell_{\gG},
\qquad
\gH^\ell \mA^\ell_{\gH}.
\end{equation}

Let \(\mC_\ell \in {0,1}^{n^\ell \times n^{\ell+1}}\) denote the cluster-membership matrix at level \(\ell\), where \((\mC_\ell)_{i,u} = 1\) if node \(i\) at level \(\ell\) belongs to cluster \(u\) at level \(\ell+1\). The adjacency matrix at level \(\ell+1\) is defined as
\begin{equation}
\mA_{\gG}^{\ell+1}
= \mathbb{I}\left(\mC_\ell^\top \mA_{\gG}^\ell \mC_\ell > 0\right),
\end{equation}
where \(\mathbb{I}(\cdot)\) denotes the element-wise indicator function.

The expanded graph adjacency matrix \(\mA^\ell_{\gH}\) is then given by
\begin{equation}
\mA^\ell_{\gH}
:= \mC_\ell \bigl(\mA_{\gG}^{\ell+1} + \mI_{n^{\ell+1}}\bigr)\mC_\ell^\top
* \mI_{n^\ell},
  \end{equation}
  where \(\mI_k\) denotes the \(k \times k\) identity matrix. This construction corresponds exactly to applying the operators \(C(\cdot)\) and \(U(\cdot)\) to \(\gG^\ell_{\mathrm{unattr.}}\).

To show that \(\gH^\ell\) is a spanning supergraph of \(\gG^\ell_{\mathrm{unattr.}}\), it suffices to prove that every edge present in \(\gG^\ell_{\mathrm{unattr.}}\) is also present in \(\gH^\ell\). 
Since \(\mA^\ell_{\gH}\) is a permutation-equivariant function of $\mA^\ell_{\gG}$ (Proposition \ref{prop:invar_coarse}), it is sufficient to show that. 

\begin{equation}
(\mA^\ell_{\gG})_{i,j} = 1
\implies
(\mA^\ell_{\gH})_{i,j} = 1,
\qquad \forall, i,j \in [n^\ell].
\end{equation}

By permutation equivariance of the construction, it is sufficient to consider an arbitrary pair of nodes \((i,j) \in [n^\ell]\). Assume that
\begin{equation}
(\mA^\ell_{\gG})_{i,j} = 1.
\end{equation}
Let \(u,v \in [n^{\ell+1}]\) be the unique clusters such that \(i \subseteq u\) and \(j \subseteq v\), i.e.,
\begin{equation}
(\mC_\ell)_{i,u} = 1,
\qquad
(\mC_\ell)_{j,v} = 1.
\end{equation}

Since \((\mA^\ell_{\gG})_{i,j} = 1\), the definition of \(\mA_{\gG}^{\ell+1}\) implies
\begin{equation}
(\mA_{\gG}^{\ell+1})_{u,v} = 1.
\end{equation}
Consequently, \((\mA_{\gG}^{\ell+1} + \mI_{n^{\ell+1}})_{u,v} = 1\), and therefore
\begin{equation}
(\mA^\ell_{\gH})_{i,j}
= \bigl(\mC_\ell (\mA_{\gG}^{\ell+1} + \mI)\mC_\ell^\top\bigr)_{i,j} - (\mI_{n^\ell})_{i,j}
= 1.
\end{equation}

Hence, every edge of \(\gG^\ell_{\mathrm{unattr.}}\) is preserved in \(\gH^\ell\), which proves that \(\gH^\ell\) is a spanning supergraph of \(\gG^\ell_{\mathrm{unattr.}}\). \(\square\)

\subsection{Proposition \ref{prop:invar_coarse}} 

Recall the definitions of the coarsening and expansion operators:
\begin{align}\label{eq:coarsening_clean}
\gG^{\ell+1}
:= \bigl(\mX_\gG^{\ell+1}, \mA_\gG^{\ell+1}\bigr),
\quad
\mX_\gG^{\ell+1} = \mC_\ell^{\top}\vone_{n^\ell},
\quad
\mA_\gG^{\ell+1} = \sI\left(\mC_\ell^{\top}\mA_\gG^{\ell}\mC_\ell > 0\right),
\end{align}
and the expansion operator
\begin{equation}\label{eq:expansion_clean}
\mA_\gH^\ell
:= \mC_\ell\bigl(\mA_\gG^{\ell+1} + \mI_{n^{\ell+1}}\bigr)\mC_\ell^{\top}
- \mI_{n^\ell}.
\end{equation}

We now prove the following proposition:

\begin{proposition}
    As functions of $\gG^\ell$, the coarsened representation $\gG^{\ell+1}=(\mA_\gG^{\ell+1}, \mX_\gG^{\ell+1})$ is permutation-invariant, whereas the expanded representation $\gH^{\ell}=\mA_\gH^\ell$ is permutation-equivariant. \hfill \qedsymbol
\end{proposition}

Let $\Pi$ denote the set of all node permutations, and let $\mP_\pi$ be the permutation matrix associated with $\pi \in \Pi$.

The assignment matrix $\mC_\ell$ is produced by a GNN followed by element-wise operations. Since GNNs are permutation-equivariant by construction and element-wise operations preserve equivariance, we have
\begin{equation}\label{eq:assignment_clean}
\mP_\pi^{\top}\mC_\ell
=
\text{Softmax}_{\text{row}}\left(
\text{GNN}\varphi\bigl(
\mP_\pi^{\top}\mX_\gG^\ell,
\mP_\pi^{\top}\mA_\gG^\ell\mP_\pi
\bigr)
\right),
\qquad \forall \pi \in \Pi.
\end{equation}

\subsubsection{Permutation invariance of the coarsening operator}

We first show that the coarsened features and adjacency are invariant to permutations of $\gG^\ell$.

Node features:
\begin{align}
\mX_\gG^{\ell+1}
= \mC_\ell^{\top}\vone
\implies
(\mP_\pi^{\top}\mC_\ell)^{\top}\vone
= \mC_\ell^{\top}\mP_\pi\vone
= \mC_\ell^{\top}\vone,
\end{align}
since $\mP_\pi\vone = \vone$. Hence $\mX_\gG^{\ell+1}$ is permutation-invariant.

Adjacency matrix:
\begin{align}
\mA_\gG^{\ell+1} = \sI\left(\mC_\ell^{\top}\mA_\gG^{\ell}\mC_\ell > 0\right) \implies  &\sI\left(
(\mP_\pi^{\top}\mC_\ell)^{\top}
(\mP_\pi^{\top}\mA_\gG^\ell\mP_\pi)
(\mP_\pi^{\top}\mC_\ell)
\right) \\
=\;
&\sI\left(
\mC_\ell^{\top}
\mP_\pi\mP_\pi^{\top}
\mA_\gG^\ell
\mP_\pi\mP_\pi^{\top}
\mC_\ell
\right) \\
=\;& \sI\left(\mC_\ell^{\top}\mA_\gG^\ell\mC_\ell\right) \\
=\;& \mA_\gG^{\ell+1},
\end{align}

where we used $\mP_\pi\mP_\pi^{\top}=\mI$.
Thus $\mA_\gG^{\ell+1}$ is permutation-invariant.\

\subsubsection{Permutation equivariance of the expanded graph}

We now show that the expanded graph is permutation-equivariant.

Thanks to the permutation invariance of $\mA_\gG^{\ell+1}$ , we have:
\begin{align}
\mA_\gH^\ell
= \mC_\ell\bigl(\mA_\gG^{\ell+1} + \mI_{n^{\ell+1}}\bigr)\mC_\ell^{\top}
- \mI_{n^\ell}. \implies &
(\mP_\pi^{\top}\mC_\ell)
(\mA_\gG^{\ell+1} + \mI)
(\mP_\pi^{\top}\mC_\ell)^{\top}
- \mI \\
=\;& \mP_\pi^{\top}\mC_\ell(\mA_\gG^{\ell+1} + \mI)\mC_\ell^{\top}\mP_\pi
- \mP_\pi^{\top}\mI\mP_\pi \\
=\;& \mP_\pi^{\top} \bigl(\mC_\ell(\mA_\gG^{\ell+1} + \mI)\mC_\ell^{\top}
- \mI \bigr)\mP_\pi \\
= \;& \mP_\pi^{\top}\mA_\gH^\ell\mP_\pi
\end{align}

Since $\mA_\gG^{\ell+1}$ is permutation-invariant, the above expression is exactly the expansion applied to the permuted assignment matrix.

Hence $\mA_\gH^{\ell}$ is permutation-equivariant.

\hfill $\qedsymbol$

\section{Technical Report}

\subsection{\emph{D}-Min Experiments}\label{ap:gamma_min}

\subsubsection{Evaluation}

In our experiments, we compare our coarsening model with three baselines: DiffPool \citep{diffpool}, DMoN \citep{DMoN} and MinCut \citep{mincut}. We evaluate graph sparsity in the spanning supergraph at the finest level, using the described coarsening approach and parameters. All three models rely on an assignment matrix $\mC$ parameterized by a Graph Neural Network (GNN). We use the same GNN architecture, and the same hyperparameters for all models. The only difference between models lies in their training objectives: DiffPool minimizes a link prediction objective, MinCut minimizes a continuous relaxation of the normalized MinCut problem, our $D$-Min minimizes the density of the spanning supergraph. All baseline models require additional regularization terms to avoid degenerate solutions, such as assigning all nodes to a single cluster.
Our model do not need such regularization. 

\subsubsection{Densities}\label{ap:gamma_index}

We present in Table \ref{tab:gamma_full} all the densities $D$ obtain by the preprocessing, and the preprocessing training time by level. These are the values used for the experiment with our generative model.  We also indicate the reduction factor (reduc.) and the resulting maximum number of nodes ($n_{max}$) for each level.

\begin{table}[H]
    \centering
    \caption{Densities, training time, and maximum number of nodes for all datasets and all levels.}
    \begin{tabular}{|c|c|c|c|c|c|c|}
    \hline
       Dataset  &  level & $n_{\max}$ & Reduc. & $D_{\text{data}}$ & $D_\gH$ & Time (s) \\
    \hline 
        Zinc250k   & 0 & 38 & 4 & 0.094 & 0.247 & 991 \\
        SBM20k & 0 & 194  & 3 & 0.083 & 0.213 & 661 \\
        SBM20k & 1 & 67  & 3 &  0.181 & 0.410 & 478 \\
        SBM20k & 2 & 23  & - &  - & - & - \\
        Reddit12k & 0 & 1499 & 3 & 0.005 & 0.017 & 5891 \\
        Reddit12k & 1 & 500 & 3 & 0.014 & 0.059 & 1516 \\
        Reddit12k & 2 & 167 & 3 & 0.048 & 0.197 & 1281 \\
        Reddit12k & 3 & 56  - &  - & - & - & -\\
    \hline 
    \end{tabular}
    \label{tab:gamma_full}
\end{table}

\subsection{Model Architecture}\label{ap:HDFM}

\subsubsection{GNNs architecture}

Our model is built on Simple Iterative Denoising \citep{sid}. We use the same architecture and reproduce the architecture description.

The denoisers are Graph Neural Networks, inspired by the general, powerful, scalable (GPS) graph Transformer.

A single layer is described as: 

\begin{align}
    \tilde{\mX}^{(l)}, \tilde{\mE}^{(l)} &= \text{MPNN}(\mX^{(l)}, \mE^{(l)}), \\
    \mX^{(l+1)} &= \text{MultiheadAttention}({\tilde{\mX}^{(l)}} + \mX^{(l)}) + {\tilde{\mX}^{(l)}} \\
    \mE^{(l+1)} &= \tilde{\mE}^{(l)} + \mE^{(l)}
\end{align}

where, $\mX^{(l)}$ and $\mE^{(l)}$ are the node and edge hidden representations after the $l$\textsuperscript{th} layer. The \emph{Multihead Attention} layer is the classical multi-head attention layer from \citet{attention_is_all}, and MPNN is a Message-Passing Neural Network layer described hereafter. 

The MPNN operates on each node and edge representations as follow:

\begin{align}
    \vh^{l}_{i, j} &= \text{ReLU}(\mW^{l}_{src}\vx^{l}_i + \mW^{l}_{trg}\vx^{l}_j + \mW^{l}_{edge}\ve^{l}_{i, j}) \\
    \ve^{l+1}_{i, j} &= \text{LayerNorm}(f_{\text{edge}}(\vh^{l}_{i, j}) \\
 	\vx^{l+1}_{i} &= \text{LayerNorm}\left(\vx^{l}_{i} + \sum_{j \in \mathcal{N}(i)}f_{\text{node}}(\vh^{l}_{i, j})\right),  
\end{align}

with $\mW^{l}_{src}$, $\mW^{l}_{trg}$, and $\mW^{l}_{edge}$ denoting trainable weight matrices, and $f{\text{node}}$ and $f_{\text{edge}}$ being small neural networks.

The node hidden states $\vx_i$ and the outputs of $f_{\text{node}}$ are of dimension $d_h$, a tunable hyperparameter. The edge hidden states $\ve_{i,j}$, intermediate messages $\vh^{l}_{i,j}$, and the outputs of $f_{\text{edge}}$ have dimension $d_h/2$.

\subsubsection{Extra Graph Features}

Enhancing GNN inputs by computing additional synthetic features, including spectral embeddings, has become a widespread practice \citep{spectre, digress, sparsediff, graphle, dgae}. 
However, denoising models usually need to recompute these features before each forward pass, both during training generation, which is computationally expensive.
Instead, we compute the extra features over the spanning supergraph. Therefore, we need to compute them only once during preprocessing. 
We use three extra features: eigen features, graph size, and partition size. All these features are concatenated to the input node attributes. 

\paragraph{Eigen features}
First, we use the eigenvectors associated with the $k$ lowest eigenvalues of the graph normalized Laplacian of the spanning supergraph, which are (up to normalization) identical to the eigenvectors of the corresponding coarse graph, except in level $L$, where we recompute the eigenvector of the noisy graph at each training step. Otherwise, and unlike other denoising models, the computation of the eigenvectors is a single preprocessing step.

\paragraph{Graph size}
We use the graph size represented as a ratio between the size of the current graph and the largest graph in the dataset $n/n_{max}$. 
We concatenate the (same) value to all nodes in a graph. 

\paragraph{Partition size}
Similarly, we use the partition size, which is represented as a ratio between its size and the largest partition in the dataset. 
All nodes in the same partition (in a partition with the same number of nodes) have the same value.

\subsection{Sampling Algorithm}\label{ap:sampling_algo}

\begin{algorithm}[H]
  \caption{Hierarchical Sampling}
  \begin{minipage}{0.6\linewidth}
  \label{algo:sample}
  \begin{algorithmic}
    \STATE {\bfseries Input:} $\{p_0^\ell\}_0^L$, $p_{n^L}$ and, $\{g_\theta^\ell\}_0^L$, which instantiate Equation \ref{eq:update_rule}. 
    \STATE $n^L \sim p_{n^L}$
    \STATE $m^L = n^L \times (n^L-1)$
    \STATE $\mE^L = \texttt{to\_dense} ( \vone_{n^L \times n^L} - I_{n^L})$
    \FOR{$\ell=L$ {\bfseries to} $0$} 
    \STATE $\va^\ell_0, \mX_0^\ell \sim p^\ell_0, \quad \text{of sizes } n^{\ell}, m^{\ell} $ \hfill {\footnotesize\textit{// Discrete Flow Matching Steps}}
      \FOR{$t=0$ {\bfseries to} $1$}
        \STATE $\va_{t+\Delta_t}^\ell, \mX_{t+\Delta_t}^\ell \sim g_\theta^\ell(\mE^\ell, \va_{t+\Delta_t}^\ell, \mX_{t+\Delta_t}^\ell) $ 
      \ENDFOR
      \STATE$\mA^\ell_\gG = \texttt{to\_dense} (\mE^\ell , \va_{1}^\ell)$
      \STATE $\mX_{\gG}^\ell = \mX_{1}^\ell$
      \IF{$L > 0$} 
        \STATE$\mC_{\ell-1} = \texttt{repeat\_interleave}(\mI_{n^{\ell}}, \mX^{\ell})$ \hfill {\footnotesize\textit{// Graph Expansion}}
        \STATE$\mA_\gH^{\ell-1} = \mC_{\ell-1}(\mA_\gG^{\ell} + \mI_{n^{\ell}}) _\gG\mC_{\ell-1}^T - \mI_{n^{\ell-1}}$
        \STATE $\mE^{\ell-1} = \text{dense\_to\_sparse}(\mA_\gH^{\ell-1})$
        \STATE $n^{\ell-1}, m^{\ell-1} =  \mC_{\ell-1}.\text{size}(0), \mE_{\ell-1}.\text{size}(0)$
    \ENDIF
    \ENDFOR
    \STATE $\gG = (\mX_{\gG}^0, \mA^0_\gG)$
\end{algorithmic}
\end{minipage}
\end{algorithm}

\section{Related Work}\label{ap:related_work}

\subsection{Generative models}\label{ap:rw_gen}

One of the main challenges in graph representation and generation is that a graph can be represented in multiple ways. There can be up to $n!$ different representations of the same graph, resulting from the $n!$ possible node permutations. Two main approaches address the multiplicity of equivalent representations: models that operate sequentially and equivariant models. 

\paragraph{Sequential models} generate graphs by auto-regressively adding nodes, edges, or subgraphs. 
To limit the number of different sequences that represent a single graph, most of these models use a Breadth-First Search (BFS) approach \citep{graphrnn, graphaf, graphdf, gran, grapharm}. While canonical representations exist for specific domains, e.g. canonical SMILES for molecular graphs \citep{gomez-bombarelli_automatic_2018, kusner_grammar_2017}, methods based on general graph canonization \citep{graphgen} fail for large graphs (see experiments in \citet{graphle}).
Subgraph aggregation \citep{jtvae, hiervae}, sometimes described as hierarchical, falls into this category. It requires listing the set of all possible substructures and connections between them, which is feasible only for some specific applications such as molecular graphs. 

\paragraph{Equivariant models} address the node permutation issue by ensuring a unique computational graph for all possible instantiations of the same object. 
These models have been developed within various generative frameworks such as GANs \citep{gggan, spectre} or Normalizing Flows \citep{graphnvp, moflow, gnf}. Recently, equivariant denoising models used in score-based diffusion \citep{edp-gnn, gdss}, discrete diffusion \citep{ddpm, digress}, diffusion bridges \citep{drum}, and Flow Matching \citep{CatFlow} have significantly improved graph generation for small graphs. 
While less known, equivariant quantized auto-encoders have also demonstrated competitive performance \citep{dgae,glad}.
However, equivariant models are not without limitations. They operate by producing predictions for all node pairs and rely on dense graph representations, which prevents them from scaling to large graphs. 

\paragraph{Sparse Equivariant Model}
SparseDiff \citep{sparsediff} and EDGE \cite{EDGE} are recent diffusion models that address scalability challenges in equivariant models. Both models share similar objectives and generative framework as ours. We present a comparative analysis to highlight their differences from ours.  

Both models construct a sparse structure by selecting a subset of "active" nodes, with 'active edges' defined by their induce complete graph. SparseDiff randomly selects "active" nodes, while EDGE determines them based on predicted changes in degree. As a result, in both models, for the same graph, the sparse structure varies at each iteration. In contrast, our model maintains a fixed structure of 'active edges' established by the spanning supergraph. 

The sparse strategies employed by SparseDiff and EDGE lead to a reduced number of function evaluations per node and edge, which is critical for denoising models and represents a fraction of the total diffusion steps. SparseDiff compensates for this by increasing function evaluations, denoising the graph in blocks of nodes and edges, ultimately considering all node pairs. This approach results in a number of function evaluations (NFE) that is inversely proportional to the edge fraction in each block. EDGE does not compensate for excluded edges, which, we assume, is the reason for its comparatively low generative performance (see \ref{sec:eval}).  

At generation, we still leverage our fixed graph structure. In our model, the additional computational cost arises from generating graphs at lower levels. However, since these graphs are significantly smaller and computational cost scales quadratically with the number of nodes, this cost is comparatively small.

\paragraph{Hierarchical models}
A couple of works follow a similar hierarchical approach to ours, but differ in their coarsening and generative strategies.
HiGen \citep{higen} coarsens the graph using a modularity objective to partition the graph into communities. 
The community-based partitions often produce dense coarse graphs, which strongly limits the effectiveness of the strategy, both in terms of the information extracted from the coarsening and the resulting sparsity used for generation.
Moreover, the model leverages an autoregressive method to generate graphs at each level, which is inefficient for large graphs.

\citet{graphle} use a local coarsening scheme involving edge or neighborhood contraction. To prevent the coarsening from pooling a child node into two parent nodes, by contracting two of its adjacent edges, they sample a different contraction sequence at each training iteration.
Moreover, the method requires a low reduction rate - set to a maximum of 0.3 - necessitating sequences with many coarseness levels. In our experiments, we were not able to generate large graphs in reasonable time with this method (see Section \ref{sec:eval}). 
In contrast to these hierarchical models, our method uses a single coarsening procedure in a preprocessing step, a small number of levels, and, last but not least, maintains equivariance.

\subsection{Discrete Flow Matching for Graph}\label{ap:rw_dfm}

A concurrent submission\footnote{Acceptance status unknown at the time of submission of this article.} \citep{simGFM} proposes simGFM, a discrete flow-matching model for graphs built upon the framework of \citet{discrete_flow_matching_gat}. SimGFM introduces the rvf-denoiser, a sampling-based variant of the original vf-denoisers.

Specifically, instead of directly computing the probability velocity
\begin{equation}
u_t^i(z^i, Z_t)
=
\frac{\dot{\alpha}}{1-\alpha},\Delta_t
\left(
p_{1|t}(z^i \mid Z_t)
-
[z^i = z_t^i]
\right),
\end{equation}
simGFM first samples a candidate
\(
\hat{z}^i \sim p_{1|t}(z^i \mid Z_t)
\)
and replaces the conditional probability by its sampled indicator, yielding
\begin{equation}
u_t^i(z^i, Z_t)
=
\frac{\dot{\alpha}}{1-\alpha},\Delta_t
\left(
[z^i = \hat{z}^i]
-
[z^i = z_t^i]
\right).
\end{equation}

However, the full update step still involves sampling according to
\begin{equation}
z^i_{t+\Delta_t}
\sim
[z^i_{t+\Delta_t} = z_t^i]
+
\Delta_t,
u_t^i(z^i_{t+\Delta_t}, Z_t).
\end{equation}

As a consequence, when the update distribution is expanded, the expectation over the sampled candidate \(\hat{z}^i\) exactly recovers the original probability velocity formulation. Therefore, from a mathematical standpoint, the sampling-based rvf-denoiser used in simGFM is strictly equivalent to the original discrete flow matching formulation of \citet{discrete_flow_matching_gat}, differing only in how the expectation is estimated in practice, simGFM requiring an unnecessary additional sampling step.  

Instead, we directly adapt the framework of \citet{discrete_flow_matching_gat}, resulting in a simpler formulation, avoiding the double sampling process in each denoising step. 

\section{Evaluation}\label{ap:eval}

\subsection{Downloads}\label{ap:downloads}

To download the Stochastic Block Model 20k (SBM20k) in the pytorch geometric Dataset format: 

\hyperlink{https://drive.switch.ch/index.php/s/t5I9N8rDQCfMIVX}{https://drive.switch.ch/index.php/s/t5I9N8rDQCfMIVX}

To download the splits between training and test sets used in our experiments:
\begin{itemize}
    \item SBM20k: \hyperlink{https://drive.switch.ch/index.php/s/zhlXUa4mUKyCP3G}{https://drive.switch.ch/index.php/s/zhlXUa4mUKyCP3G}
    \item Reddit12k \hyperlink{https://drive.switch.ch/index.php/s/xIS3DMY2eUCzN8c}{https://drive.switch.ch/index.php/s/xIS3DMY2eUCzN8c}
\end{itemize}

\subsection{Conditional generation}\label{ap:cond_gen}

In conditional generation, we generate graphs $\hat{\gG}$ conditionally to the spanning supergraph $\gH$. In our experiment, we sample a graph from the validation set to serve as the reference graph $\gG_{ref}$. The conditional task involves generating graphs structurally similar to  $\gG_{ref}$ by conditioning on its spanning supergraph $\gH_{ref}$. Since we directly sample from an actual spanning supergraph, only the finer-level model $p_{\theta}(\gG^0|\gH^0)$ is needed for conditional generation.

We use the spectral distance to measure the distance between the reference graph and the conditionally generated graphs. 

To evaluate similarity, we use the spectral distance between the reference graph and the conditionally generated graphs. The spectral distance between graphs $\gG_1$ and $\gG_2$ is defined as  $\sum_{i=1}^K |\lambda_{\gG_1}(i)- \lambda_{\gG_2}(i)|$, where $\lambda(i)$ is the $i$\textsuperscript{th} eigenvalue of the Laplacian sorted in non-decreasing order. To ensure that we can compute this distance between graphs of different sizes, we compute this distance only on the $K$ smallest eigenvalues. 

\begin{table}[H]
    \centering
    \small
    \caption{Spectral distances}
    \begin{tabular}{c | c c c}
    \hline
       Dataset & Test set & Cond. gen. & Ratio \\
    \hline 
        SBM20k &  4.57 $\pm$ 1.55 & 0.49 $\pm$ 0.06  & 0.114 $\pm$ 0.024 \\
        Reddit &  3.57 $\pm$ 1.10 & 0.57 $\pm$ 0.31  & 0.161 $\pm$ 0.069 \\
    \hline 
    \end{tabular}
    \label{tab:cond}
\end{table}

We computed the spectral distance between the reference graph and 100 conditionally generated graphs, and compared it with the distances between the reference graph and graphs from the test set. This experiment was repeated 10 times with different reference graphs. Table \ref{tab:cond} shows the average distances and ratios. On average, the generated graphs were approximately 10 times closer to the reference graph than those in the test set, demonstrating the effectiveness of our model's conditional setting. 
\subsection{Datasets}\label{ap:dataset}

\paragraph{Zinc250}
The Zinc250k dataset is a subset of the Zinc database \citep{zinc}. 
It includes 250,000 molecules with up to 38 heavy atoms of nine types. 
We used the kekulized representation of this dataset.

\paragraph{Qm9H}
The qm9 dataset is a dataset \citep{qm9} that includes 133'885 molecules with up to 29 atoms of 5 types. 
We use the version with explicit hydrogen which is much more challenging than the standard benchmark for graph generation, which has at most 9 atoms of 4 types. 
We used the kekulized representation of this dataset.

\paragraph{SBM20k}
The original Stochastic Block Model is a synthetic dataset made of community graphs, with 2 to 5 communities, each containing 20 to 40 vertices. 
The intra-community and inter-community edge probability are 0.3 and 0.005, respectively. 
We create a dataset, SBM20k, with exactly the same characteristics but 20000 instances instead of 200.

\paragraph{Ego} 
The dataset contains 757 3-hop ego networks extracted from the Citeseer
network.  Nodes represents documents and edges represent citation relationships.
It contains graph with up to 399 nodes.

\paragraph{Reddit12k} 
Reddit\citep{reddit} contains graphs extracted from the Reddit networks. It is part of the TUDataset. We extracted the graphs containing up to 1500 nodes.
This results in a dataset that collects 11551 graphs with up to 1499 nodes.

\subsection{Evaluation Procedure}\label{ap:eval_metrics}

\subsubsection{Time}

We report wall-clock time in seconds. All timing evaluations were run on a single compute node equipped with 128 CPU cores and an NVIDIA GeForce RTX 3090 GPU (25 GB memory).

\subsubsection{Molecular Datasets}

We report the Fréchet ChemNet Distance (FCD) \citep{fcd}, which measures the similarity between generated and real molecules in chemical space; the Neighborhood Subgraph Pairwise Distance Kernel (NSPDK) \citep{nspdk}, which evaluates structural similarity between graphs; the validity (without correction), defined as the proportion of chemically valid molecules; and the sampling time required to generate 10,000 molecules.

\paragraph{Spits} We used the test sets provided by \citet{gdss}. 

\subsubsection{Large unattributed dataset}

We report standard Maximum Mean Discrepancy (MMD) metrics based on degree distribution, clustering coefficients, and spectral propertie to compare the distributions of graph statistics between
generated and test graphs \citep{graphrnn}. 

\paragraph{SBM} We additionally report validity, which in this case measures whether a generated graph lies within a 90\% confidence interval of the true SBM distribution; consequently, 10\% of graphs sampled from the true distribution would be classified as invalid by this criterion. Baseline results are obtained by adapting the official implementations to the new dataset while following the same specifications used for the standard SBM benchmark. For SID, we employ our own architecture and adapt the sampling procedure accordingly.
The MMDs are computed between the 1000 graphs in the test set and 1000 generated graphs.

\paragraph{Ego} We use 20\% of the dataset as the test set. The remaining graphs are further split, with 20\% used for validation and the rest for training. MMD metrics are computed between graphs from the test set and an equal number of generated graphs.

\paragraph{Reddit} MMD metrics are computed between 100 graphs randomly sampled from the test set and an equal number of generated graphs.

\subsection{Additional and Detailed Results}

\setlength{\tabcolsep}{3pt}
\begin{table}[H]
\scriptsize
\begin{center}
\begin{sc}
\caption{\footnotesize Generation results on the \texttt{QM9H} dataset. NSPDK results rescaled by $10^{3}$.}

        \begin{tabular}{ c c  c c c c}
        \toprule
         \textbf{Model} & NFE &  \textbf{Valid \%}$\uparrow$ & \textbf{FCD} $\downarrow$ & \textbf{NSPDK}$\downarrow$& \textbf{Time} (s) $\downarrow$  \\
        \midrule
        DFM & 128& $97.17 \pm 0.12$ &  $0.251 \pm 0.005$ & $0.438 \pm 0.025$ &  $842.5 \pm 12.8$ \\
        HDFM & 128 & $95.78 \pm 0.03$ &  $0.231 \pm 0.026$ & $0.533 \pm 0.027$ & $333.5 \pm 7.6$ \\ 
        HDFM & 32&  $94.22 \pm 0.31$ &  $0.236 \pm 0.007$ & $0.469 \pm 0.019$ & $174.1 \pm 2.3$ \\
        \bottomrule
        \end{tabular}
        \label{tab:qm9H_ap}
\end{sc}
\end{center}
\end{table}

\setlength{\tabcolsep}{3pt}
\begin{table}[H]
\scriptsize
\begin{center}
\begin{sc}
\caption{\footnotesize Generation results on the \texttt{Zinc250k} datasets. NSPDK results $\times 10^{3}$. Generation time in seconds.}
        \begin{tabular}{ c c c c c c}
         \toprule
         Model & NFE &  FCD $\downarrow$ & NSPDK$\downarrow$ & Valid \%$\uparrow$ & Time (s) $\downarrow$  \\
        \midrule

        DFM & 128 & $1.48 \pm 0.01$ & $0.554 \pm 0.014$ & $99.41 \pm 0.10$& $1998.8 \pm 2.5\hphantom{00}$ \\
        HDFM & 128 & $1.23 \pm 0.02$ & $0.433 \pm 0.025$ & $99.30 \pm 0.07$ & $394.3 \pm 8.6$ \\
        HDFM & 64 & $1.26 \pm 0.01$ & $0.500 \pm 0.009$ & $99.17 \pm 0.09$ &  $238.7 \pm 2.4$ \\ 
        HDFM & 32 &  $1.34 \pm 0.01$ & $0.637 \pm 0.021$ & $98.82 \pm 0.04$ & $167.6 \pm 2.2$ \\
        HDFM & 128 & $1.25 \pm 0.02$ & $0.376 \pm 0.023$ & $99.28 \pm 0.03$ & $\hphantom{0}237.0 \pm 12.1$ \\
        
        \bottomrule
        \end{tabular}
        \label{tab:zinc_ap}
\end{sc}
\end{center}
\end{table}

\setlength{\tabcolsep}{3pt}
\begin{table}[H]
\scriptsize
\begin{center}
\begin{sc}
\caption{\footnotesize Generation results on the \textbf{SBD20k} dataset (MMD results rescaled by $10^{3}$).}

        \begin{tabular}{ c c  c c c c c c}
         \toprule
         \textbf{Model} & NFE & \textbf{Valid}$\uparrow$ &\textbf{Deg.} $\downarrow$ & \textbf{Clust.}$\downarrow$ & \textbf{Spec. }$\uparrow$ & \textbf{Time} (s) $\downarrow$ & \textbf{Uniq+novel} (s) $\uparrow$  \\
        \midrule

        &  EDGE & & 154.16 $\pm$ 3.79\hphantom{00} & 766.79 $\pm$ 19.24\hphantom{0} &  10.58 $\pm$ 0.12\hphantom{0} & 350 $\pm$ 8 & \\
        
        &  DiGress & & 4.85 $\pm$ 6.15 & 5.92 $\pm$ 1.97 & 28.30 $\pm$ 14.4\hphantom{0}  & 3906 $\pm$ 53 & \\ 
        
        &  SparseDiff & & 1.10 $\pm$ 0.34 &  5.25 $\pm$ 0.50 & 19.11 $\pm$ 4.52\hphantom{0} & 11089 $\pm$ 640 \\\
        
        DFM & 128 & $51.63 \pm 1.71$ &  $4.813 \pm 0.140$ & $4.921 \pm 0.144$ & $1.001 \pm 0.032$ & $2834.4 \pm 1.2\hphantom{0}$ & $100.00 \pm 0.00$ \\
        HDFM & 128 & $70.40 \pm 0.96$ &  $0.147 \pm 0.039$ & $2.055 \pm 0.022$ & $0.301 \pm 0.016$ & $\hphantom{0}489.8 \pm 17.3$ & $100.00 \pm 0.00$ \\
        HDFM & 32 & $68.80 \pm 0.54$ &  $0.241 \pm 0.028$ & $2.057 \pm 0.025$ & $0.302 \pm 0.018$ & $160.7 \pm 3.5$ & $100.00 \pm 0.00$ \\

        HDFM-R  & 128 & $85.00 \pm 1.30$ &  $0.184 \pm 0.031$ & $2.130 \pm 0.026$ & $0.327 \pm 0.013$ & $187.7 \pm 7.7$ & $100.00 \pm 0.00$ \\
        
        \bottomrule
        \end{tabular}
        \label{tab:SBM_ap}
\end{sc}
\end{center}
\end{table}

\setlength{\tabcolsep}{3pt}
\begin{table}[H]
\scriptsize
\begin{center}
\begin{sc}
\caption{Generation results on \texttt{Reddit12k} datasets. MMDs results rescaled by $10^{3}$. Generation time in seconds.}
\begin{tabular}{c cc c c c c}
\toprule
& \textbf{Model} &  \textbf{degree}$\downarrow$ & \textbf{clust.} $\downarrow$ & \textbf{spect.}$\downarrow$ & \textbf{Time}$\downarrow$  \\
\midrule
 EDGE & 1000 & 154.16 $\pm$ 3.79 & 766.79 $\pm$ 19.24 & 37.63 $\pm$ 1.63 & 172 $\pm$ 1\hphantom{0} \\

 SparseDiff & 10000 &\hphantom{00}93.73 $\pm$ 10.80 & \hphantom{0}68.46 $\pm$ 39.01 &  124.96 $\pm$ 14.24 & 7234 $\pm$ 859 \\
HDFM & 128 & $0.32 \pm 0.07$ & $3.35 \pm 0.86$ & $3.34 \pm 0.18$ & $107 \pm 6$ \\
HDFM & 32 & $0.40 \pm 0.08$ & $2.31 \pm 0.37$ & $3.62 \pm 0.26$ & $47 \pm 4$ \\ 
\bottomrule
\end{tabular}
\label{tab:reddit_ap}
\end{sc}
\end{center}
\end{table}

\section{Visualizations}

We provide visualizations comparing data and generated graphs. 

\subsection{SBM20k}

\begin{figure}[H]
    \centering
    \begin{tabular}{|ccc|ccc|}
    \hline
      \multicolumn{3}{|c|}{Generated graphs}   &
      \multicolumn{3}{|c|}{Real graphs - SBM20k}
     \\
    \hline
           
        \includegraphics[width=0.15\textwidth]{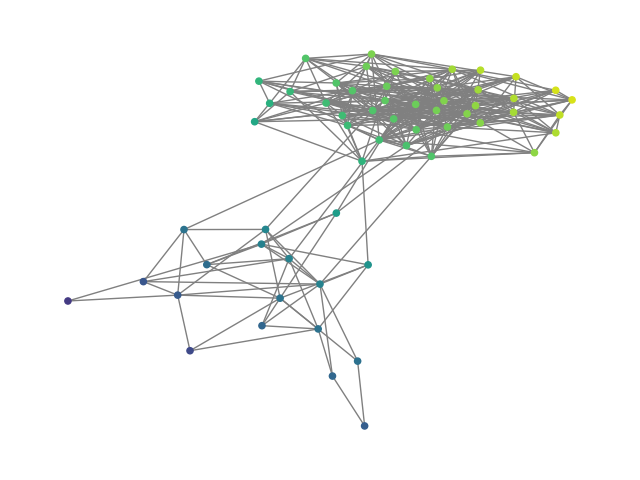} &
        \includegraphics[width=0.15\textwidth]{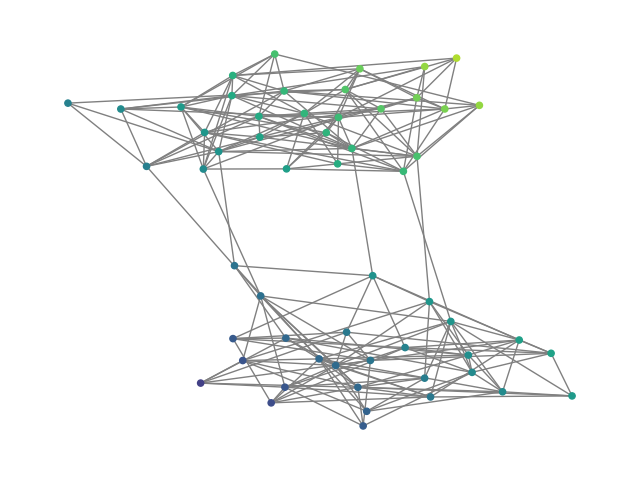} &
        \includegraphics[width=0.15\textwidth]{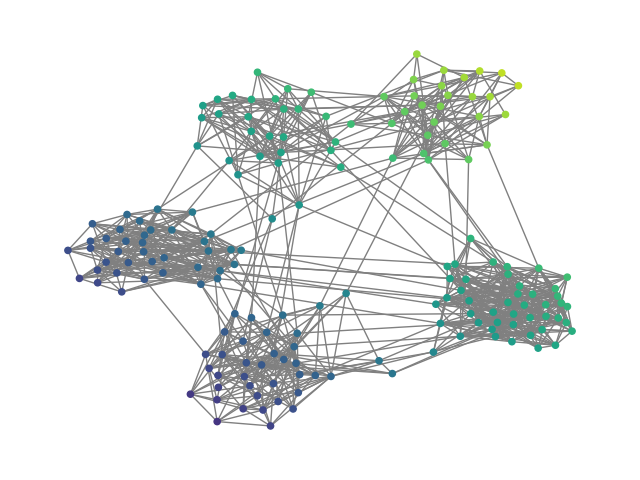} &
        \includegraphics[width=0.15\textwidth]{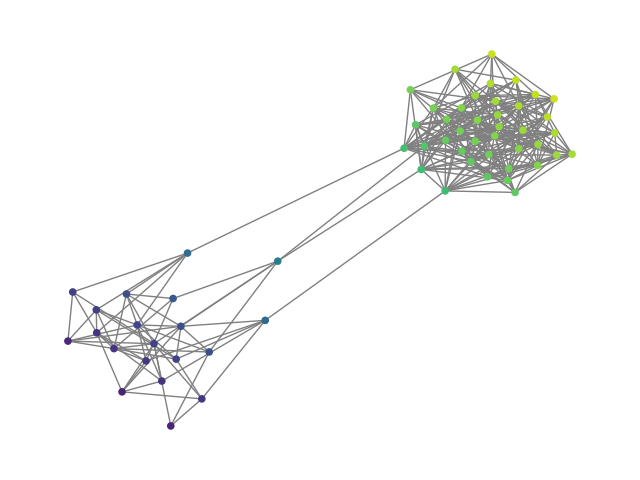} &
        \includegraphics[width=0.15\textwidth]{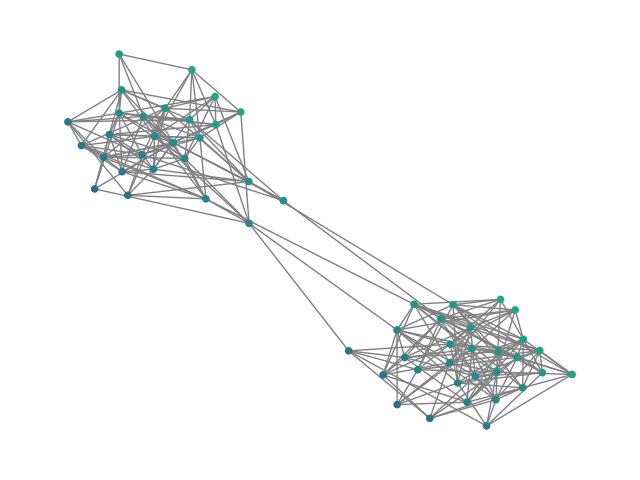}&
        \includegraphics[width=0.15\textwidth]{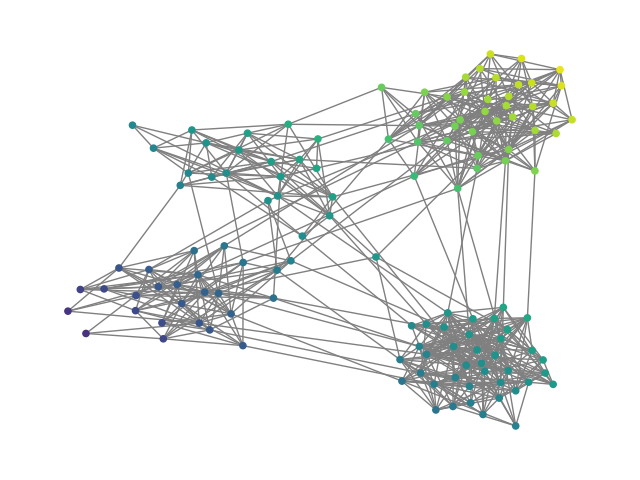} 
        \\
        \includegraphics[width=0.15\textwidth]{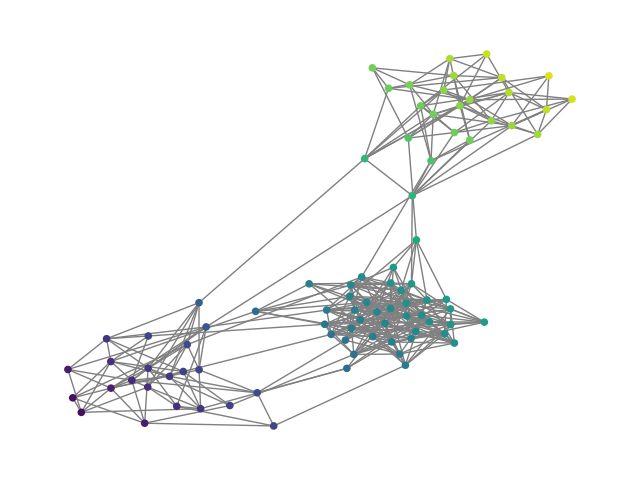} &
        \includegraphics[width=0.15\textwidth]{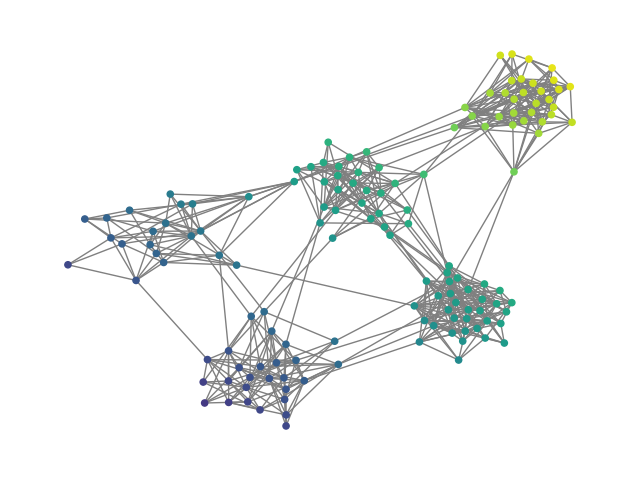} &
        \includegraphics[width=0.15\textwidth]{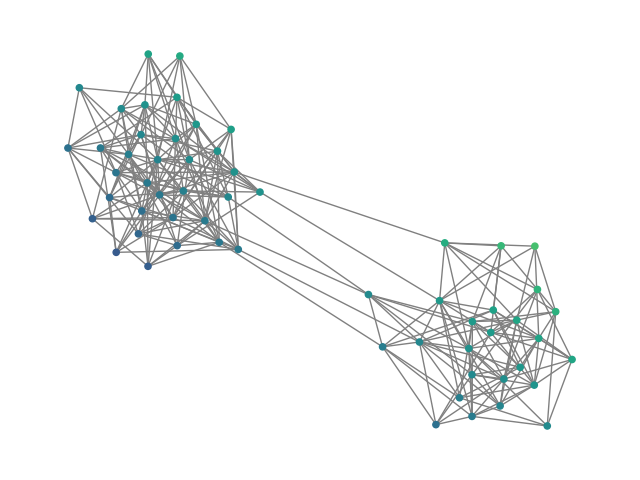} &
        \includegraphics[width=0.15\textwidth]{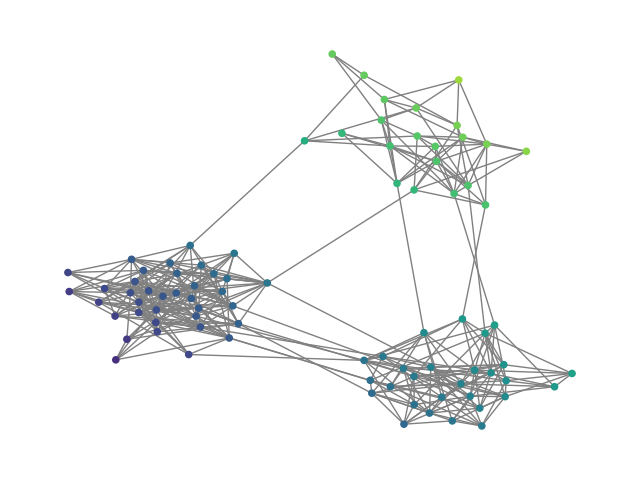} &
        \includegraphics[width=0.15\textwidth]{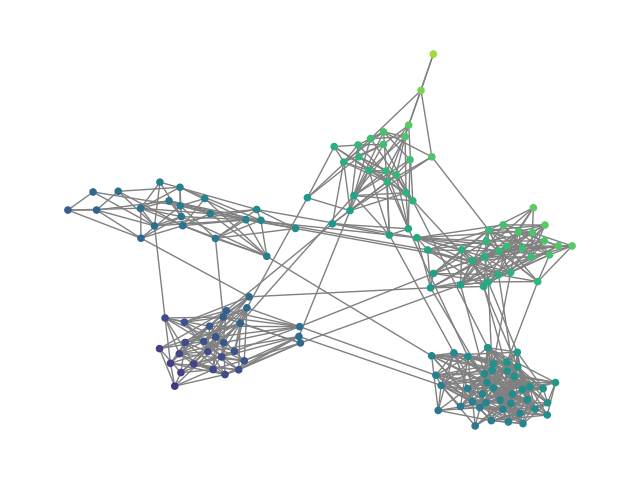}&
        \includegraphics[width=0.15\textwidth]{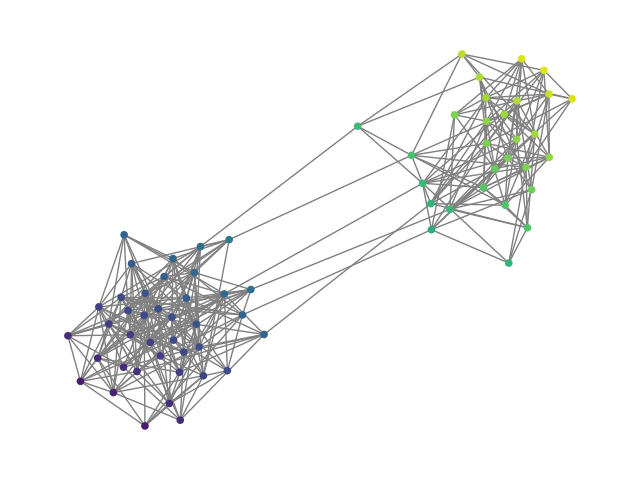} 
        \\
        \includegraphics[width=0.15\textwidth]{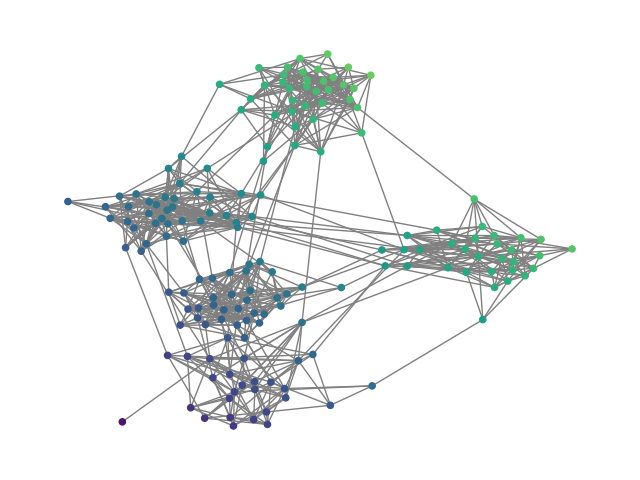} &
        \includegraphics[width=0.15\textwidth]{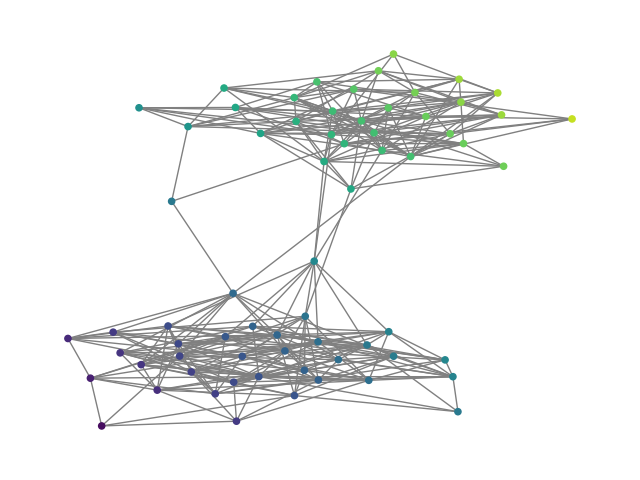} &
        \includegraphics[width=0.15\textwidth]{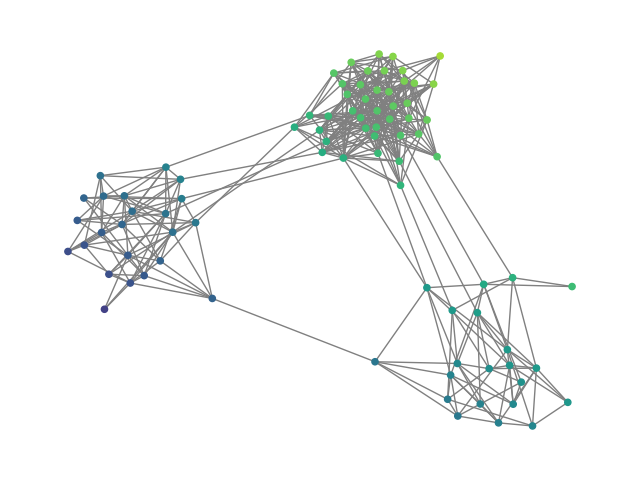} &
        \includegraphics[width=0.15\textwidth]{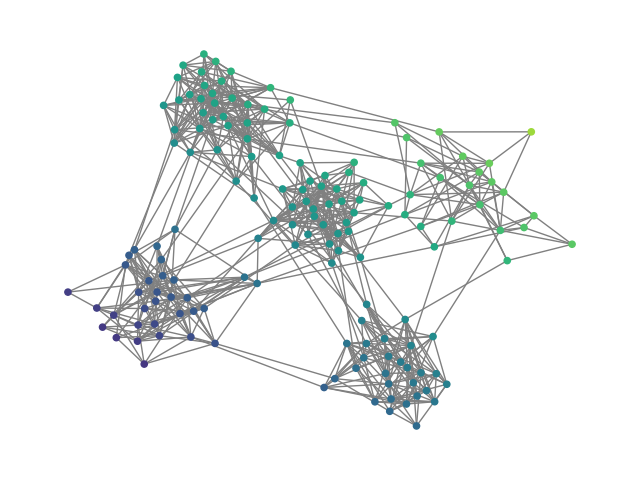} &
        \includegraphics[width=0.15\textwidth]{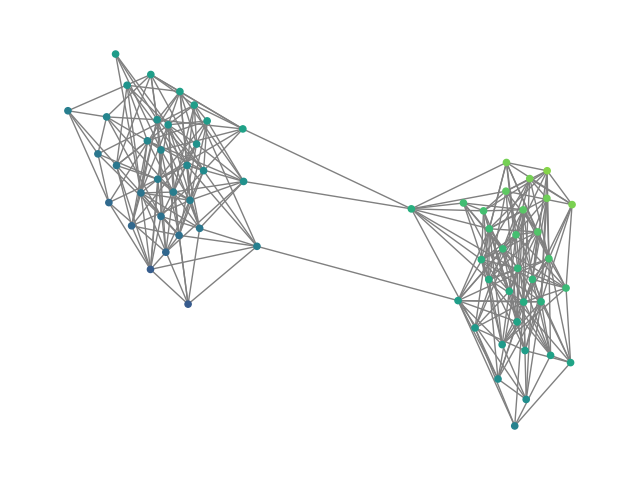}&
        \includegraphics[width=0.15\textwidth]{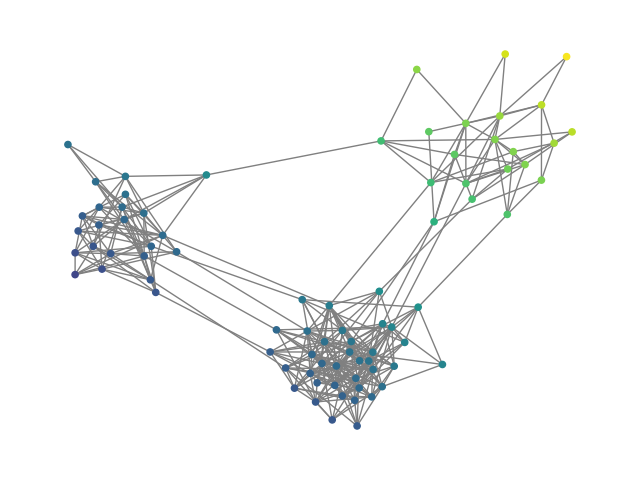} 
        \\
        \includegraphics[width=0.15\textwidth]{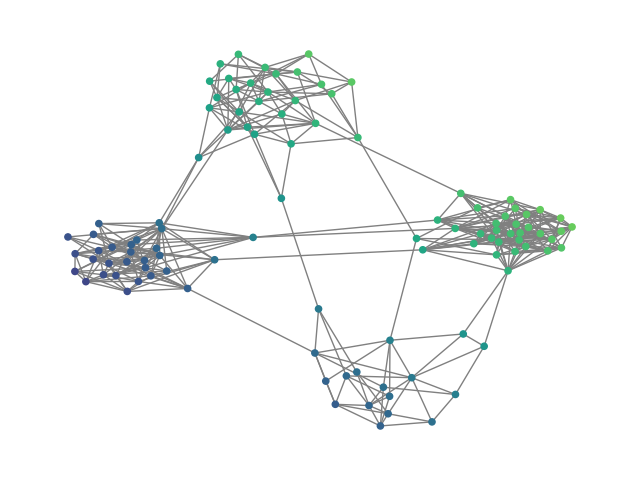} &
        \includegraphics[width=0.15\textwidth]{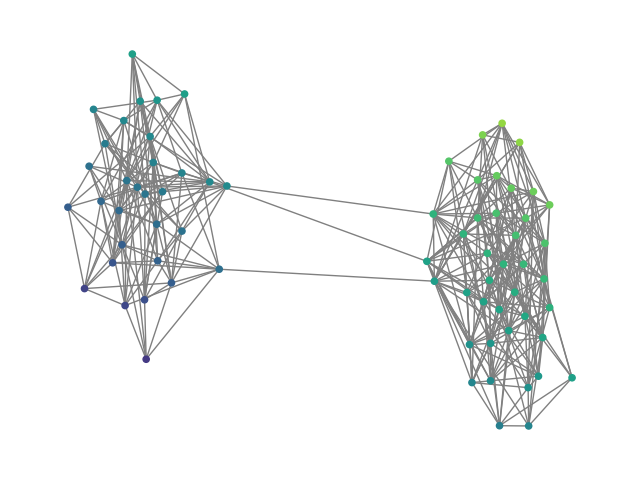} &
        \includegraphics[width=0.15\textwidth]{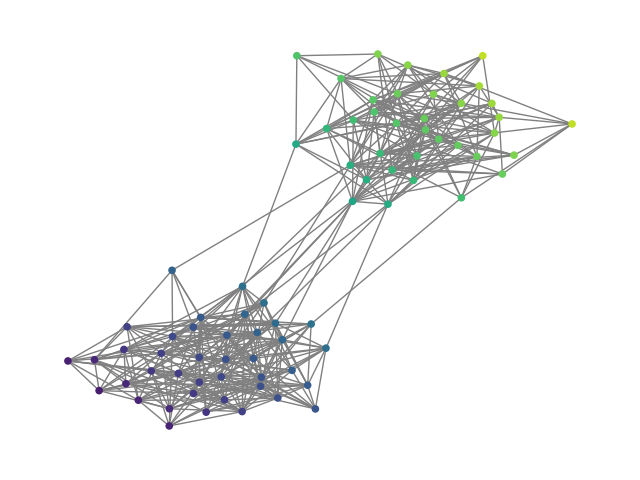} &
        \includegraphics[width=0.15\textwidth]{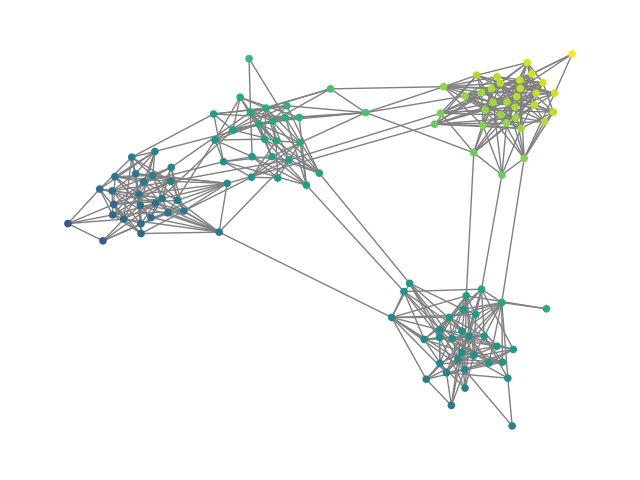} &
        \includegraphics[width=0.15\textwidth]{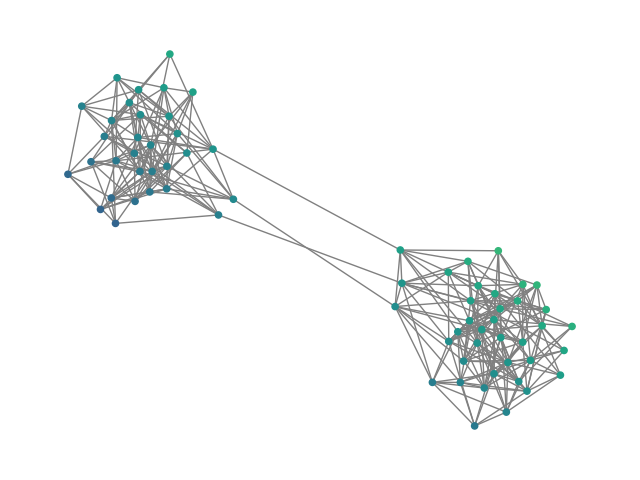}&
        \includegraphics[width=0.15\textwidth]{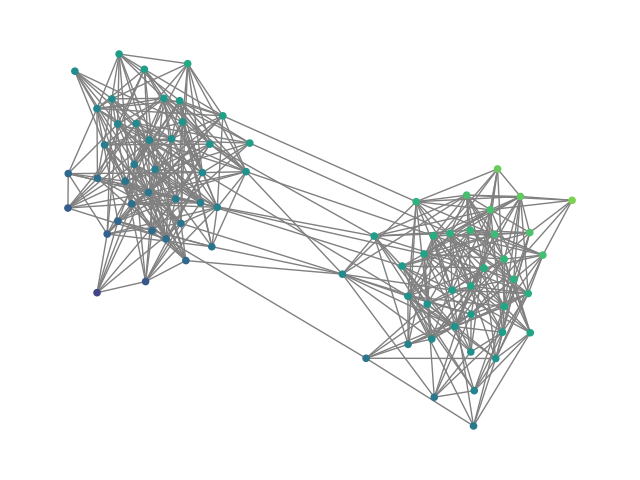} 
        \\
        \includegraphics[width=0.15\textwidth]{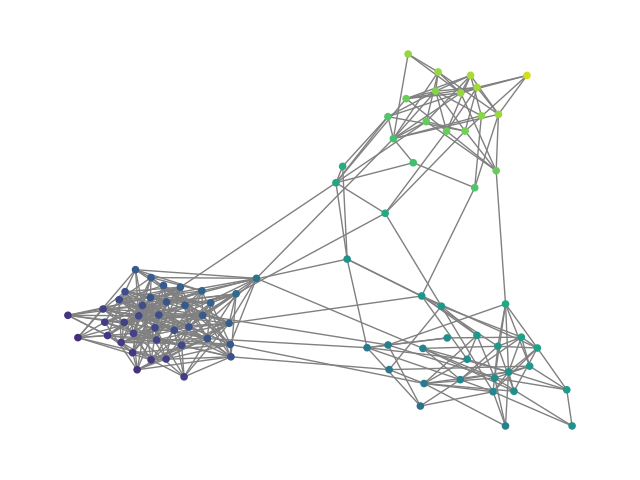} &
        \includegraphics[width=0.15\textwidth]{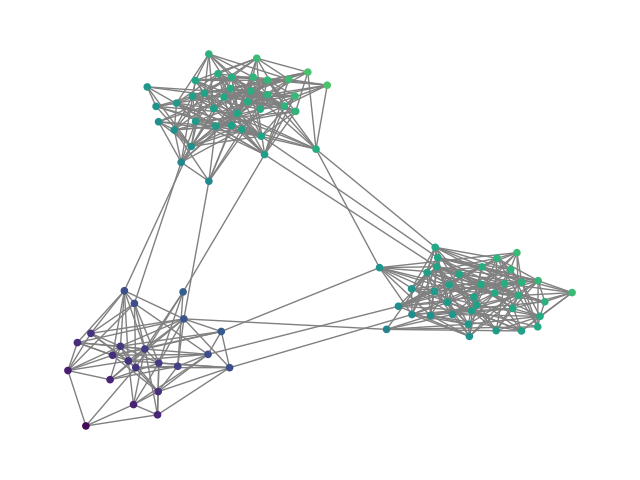} &
        \includegraphics[width=0.15\textwidth]{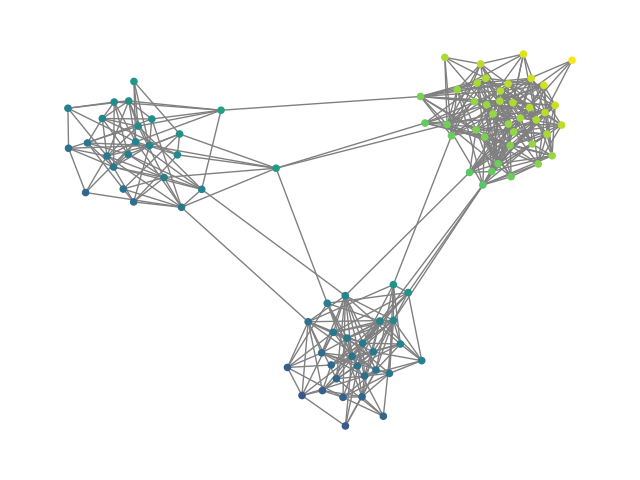} &
        \includegraphics[width=0.15\textwidth]{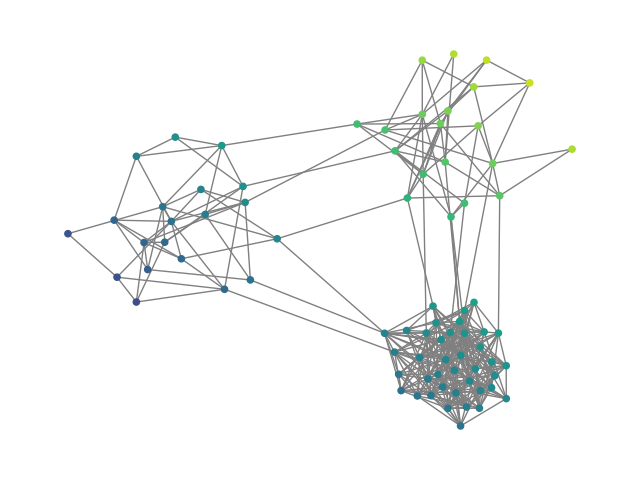} &
        \includegraphics[width=0.15\textwidth]{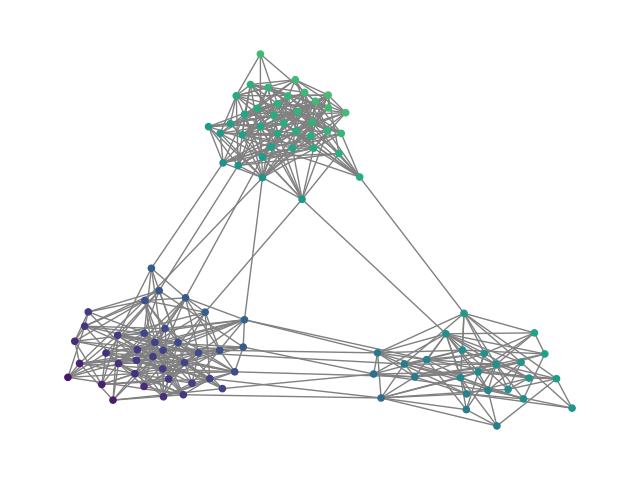}&
        \includegraphics[width=0.15\textwidth]{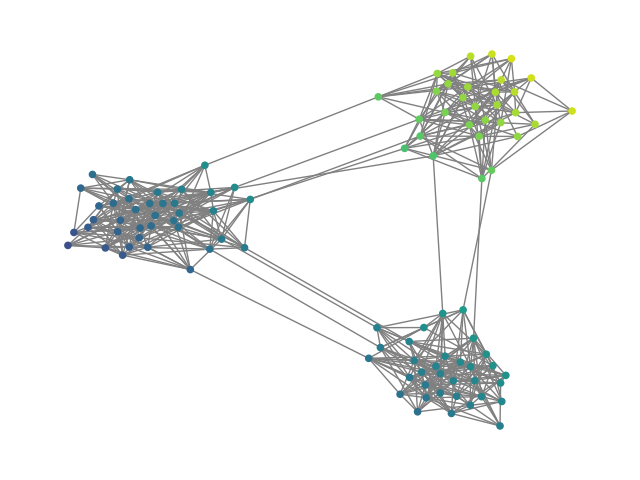} 
        \\
        \includegraphics[width=0.15\textwidth]{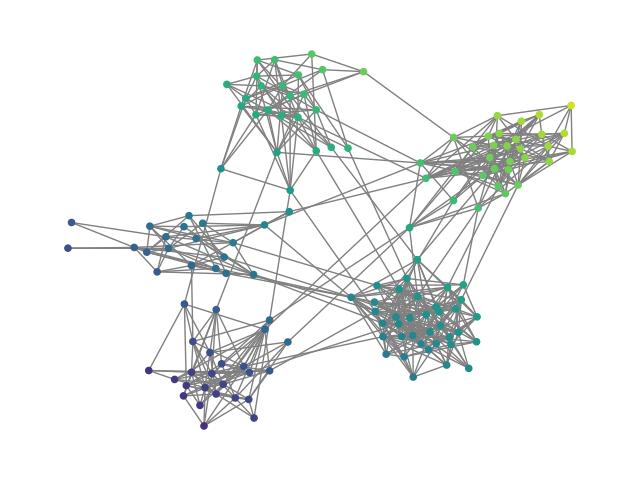} &
        \includegraphics[width=0.15\textwidth]{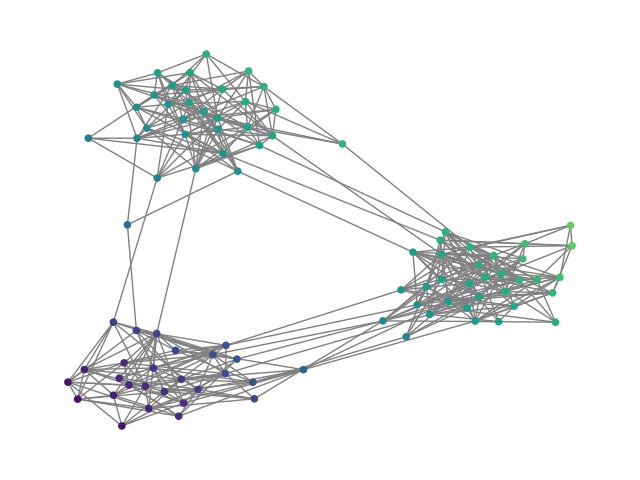} &
        \includegraphics[width=0.15\textwidth]{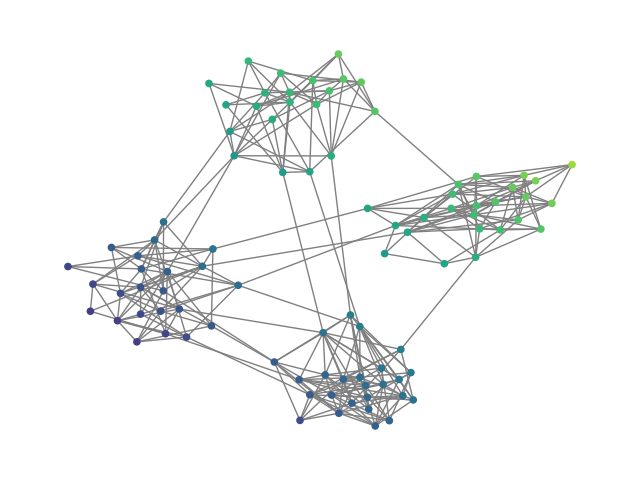} &
        \includegraphics[width=0.15\textwidth]{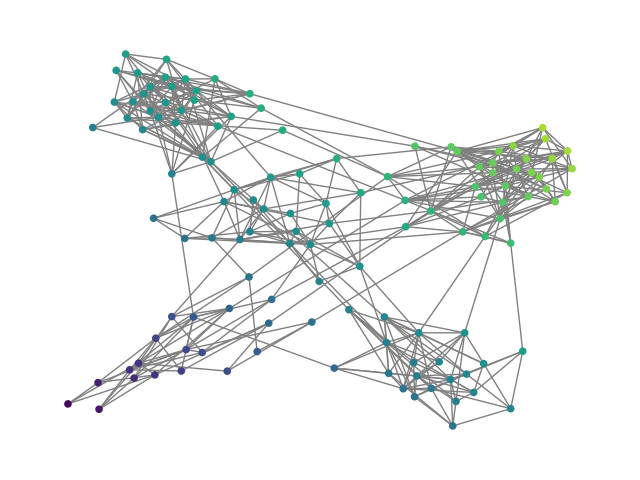} &
        \includegraphics[width=0.15\textwidth]{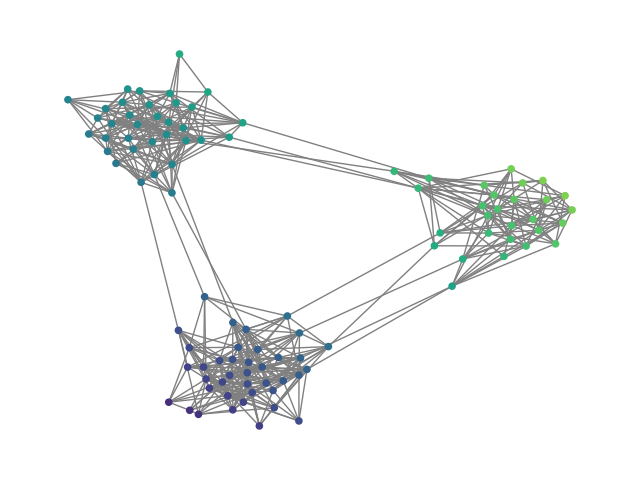}&
        \includegraphics[width=0.15\textwidth]{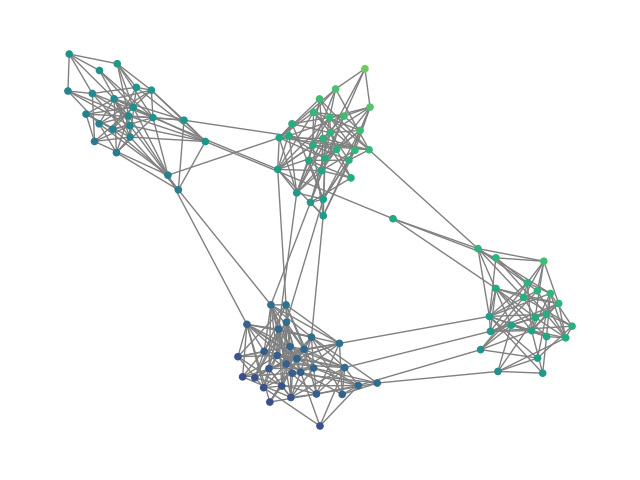} 
        \\
        \hline
        \end{tabular}
    \caption{SBM20k: Comparison of generated graphs with graphs from the dataset.}
    \label{fig:sbm20K}
\end{figure}

\subsection{ego}

\begin{figure}[H]
    \centering
    \begin{tabular}{|ccc|ccc|}
    \hline
      \multicolumn{3}{|c|}{Generated graphs}   &
      \multicolumn{3}{|c|}{Real graphs - ego}
     \\
    \hline
           
        \includegraphics[width=0.15\textwidth]{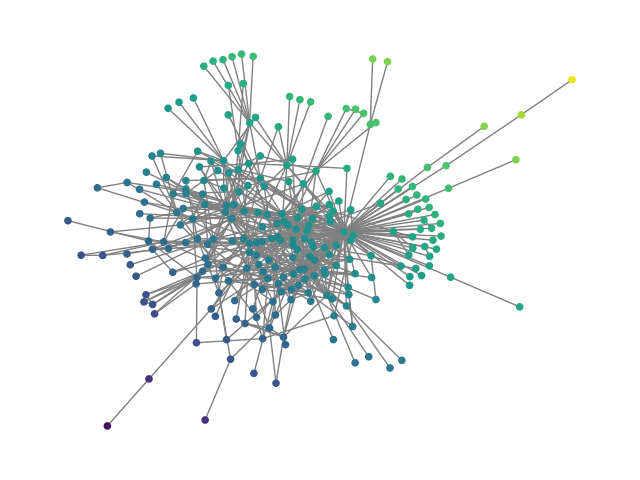} &
        \includegraphics[width=0.15\textwidth]{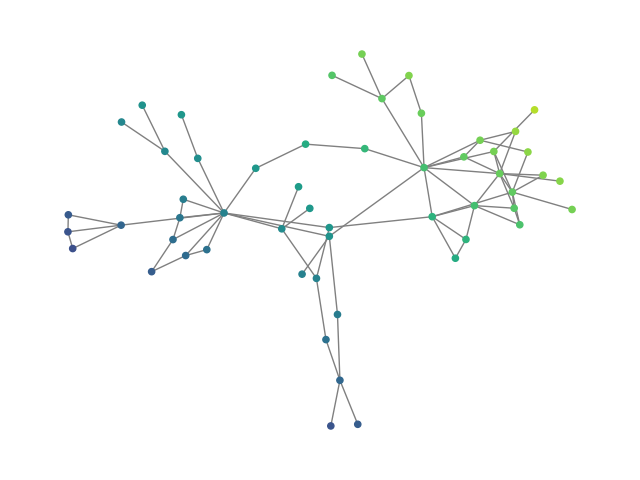} &
        \includegraphics[width=0.15\textwidth]{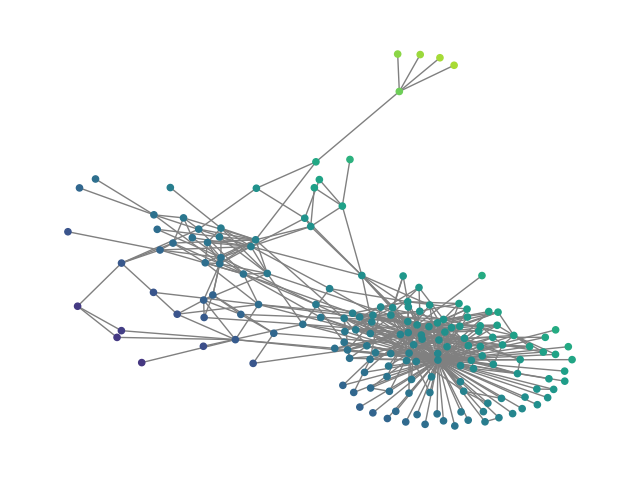} &
        \includegraphics[width=0.15\textwidth]{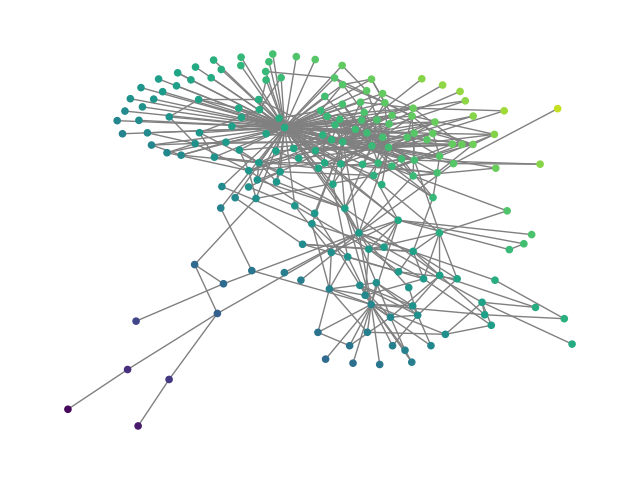} &
        \includegraphics[width=0.15\textwidth]{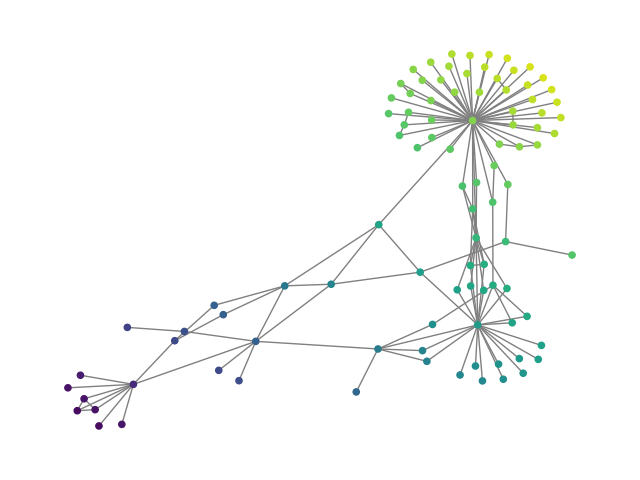}&
        \includegraphics[width=0.15\textwidth]{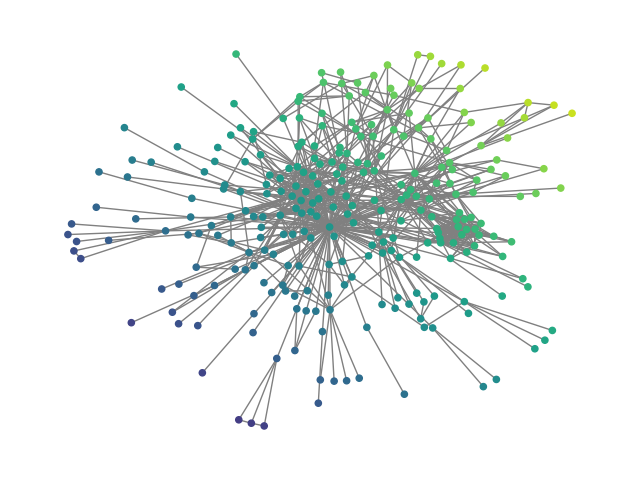} 
        \\
        \includegraphics[width=0.15\textwidth]{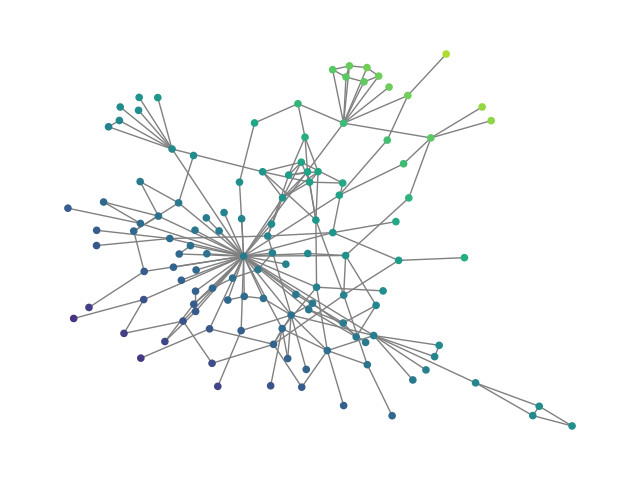} &
        \includegraphics[width=0.15\textwidth]{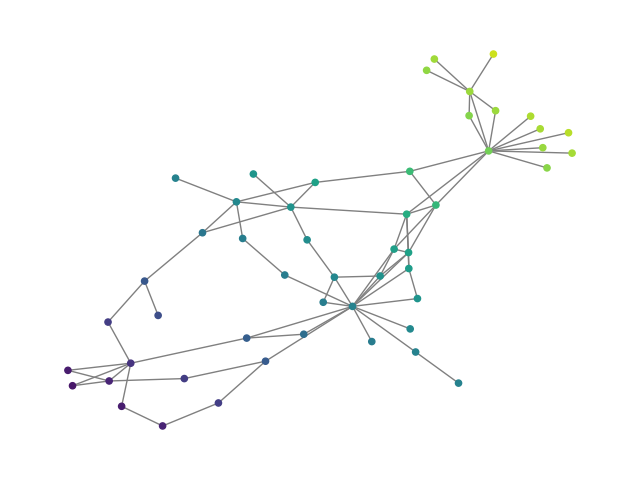} &
        \includegraphics[width=0.15\textwidth]{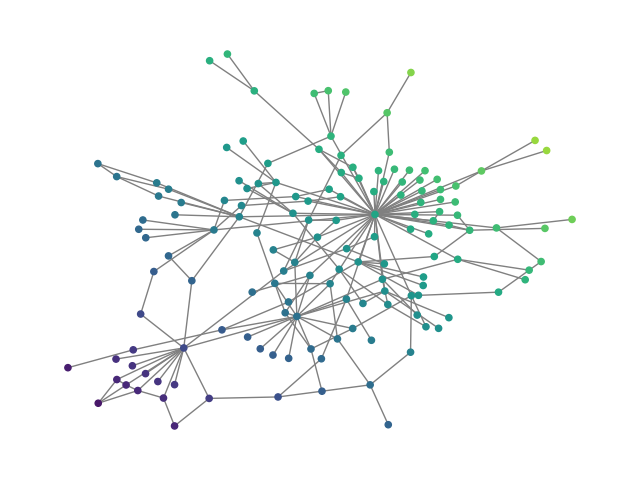} &
        \includegraphics[width=0.15\textwidth]{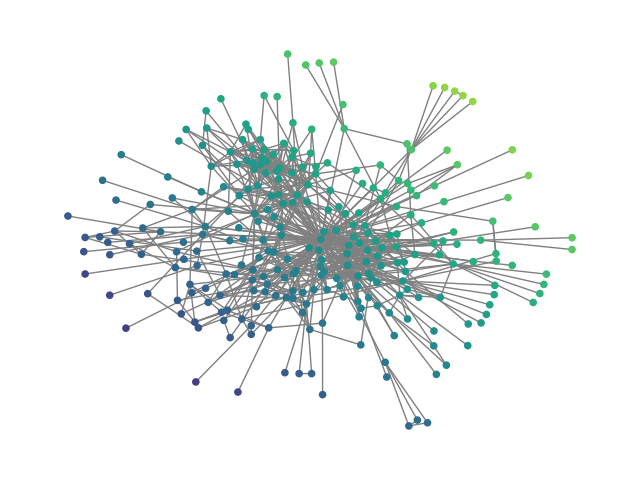} &
        \includegraphics[width=0.15\textwidth]{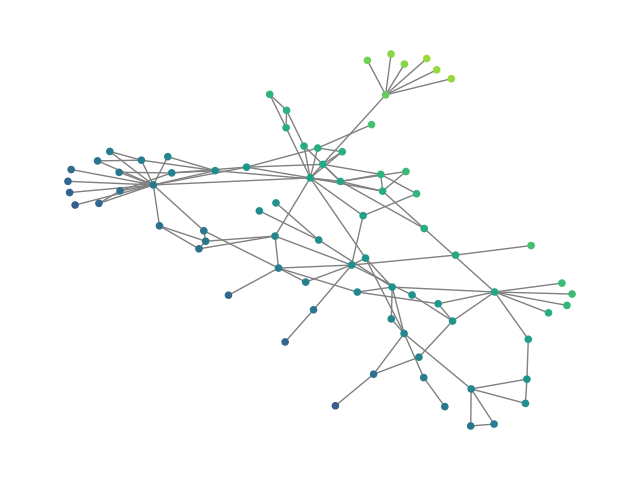}&
        \includegraphics[width=0.15\textwidth]{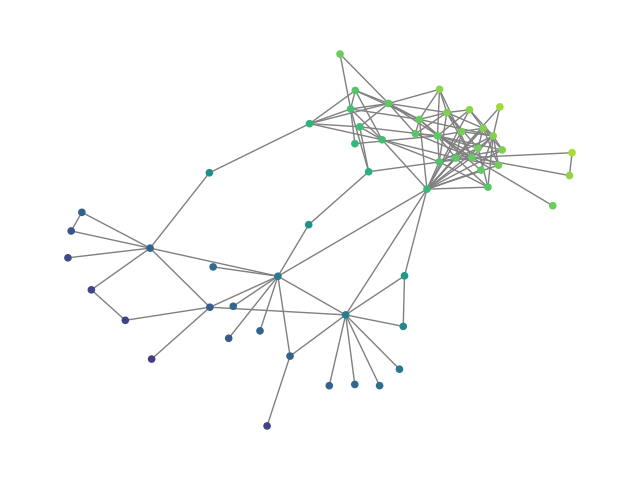} 
        \\
        \includegraphics[width=0.15\textwidth]{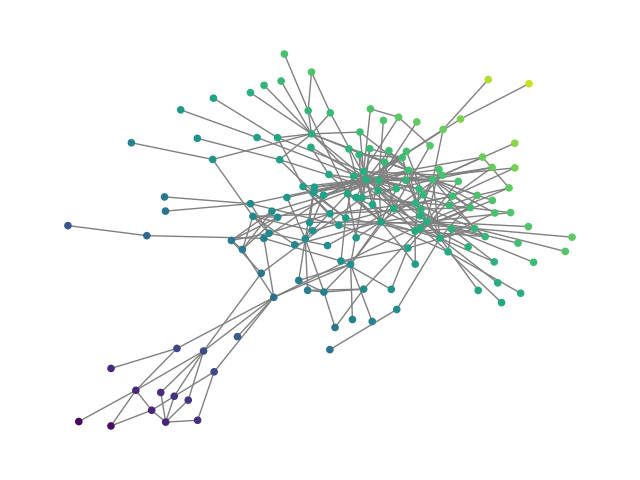} &
        \includegraphics[width=0.15\textwidth]{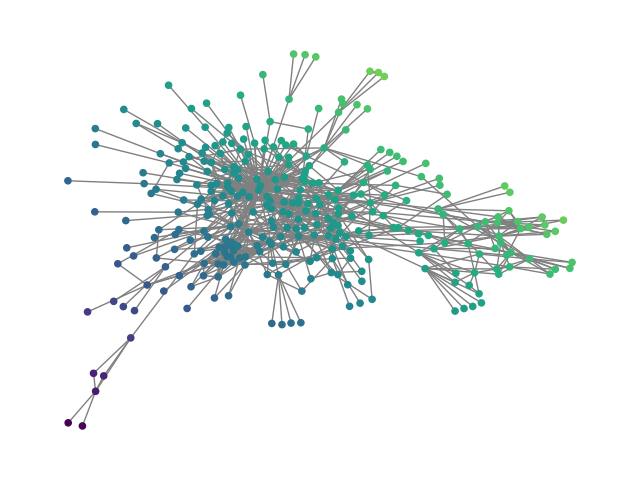} &
        \includegraphics[width=0.15\textwidth]{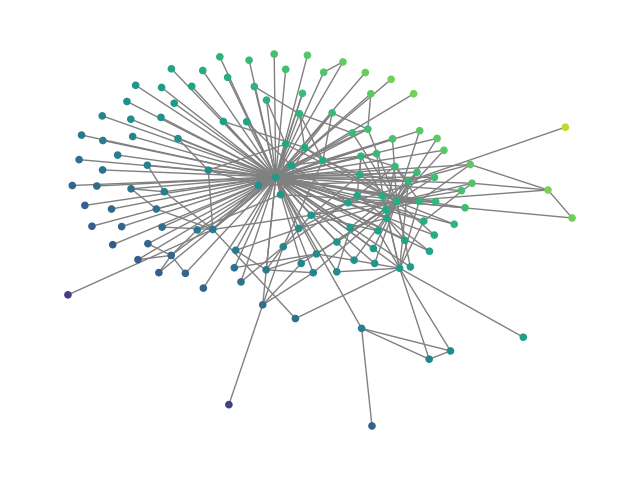} &
        \includegraphics[width=0.15\textwidth]{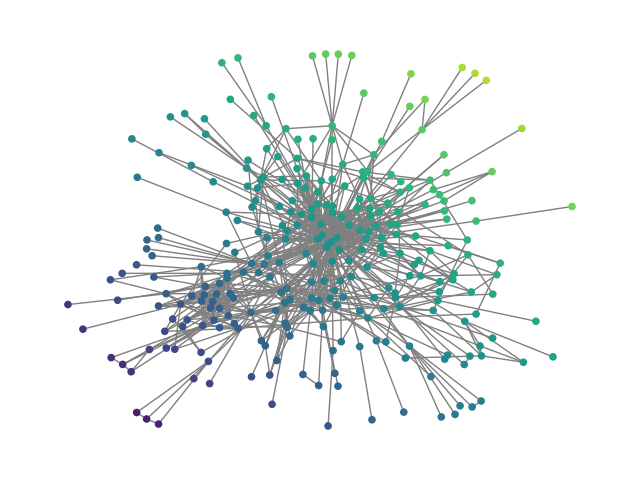} &
        \includegraphics[width=0.15\textwidth]{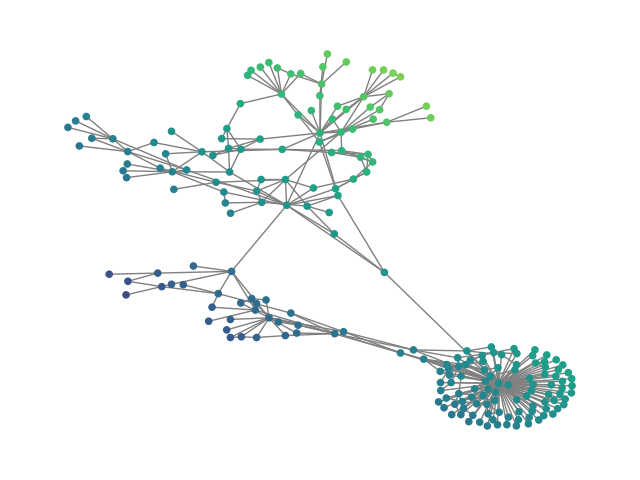}&
        \includegraphics[width=0.15\textwidth]{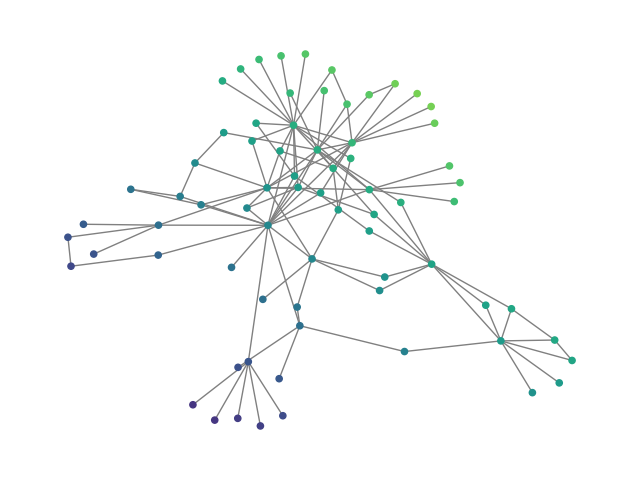} 
        \\
        \includegraphics[width=0.15\textwidth]{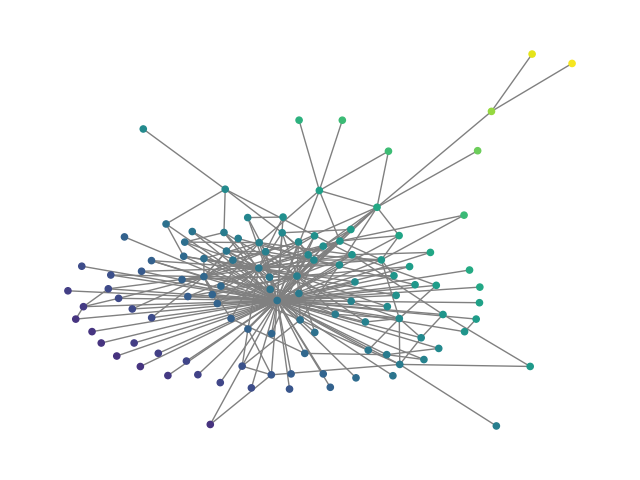} &
        \includegraphics[width=0.15\textwidth]{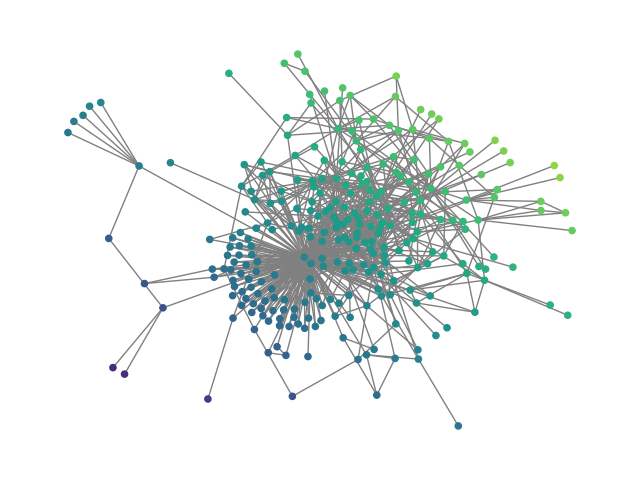} &
        \includegraphics[width=0.15\textwidth]{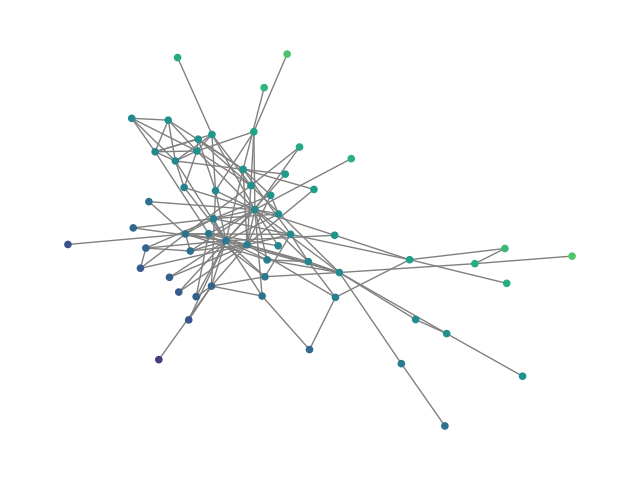} &
        \includegraphics[width=0.15\textwidth]{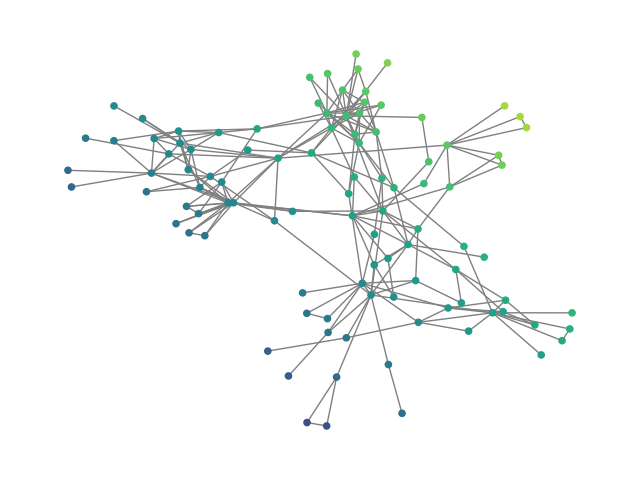} &
        \includegraphics[width=0.15\textwidth]{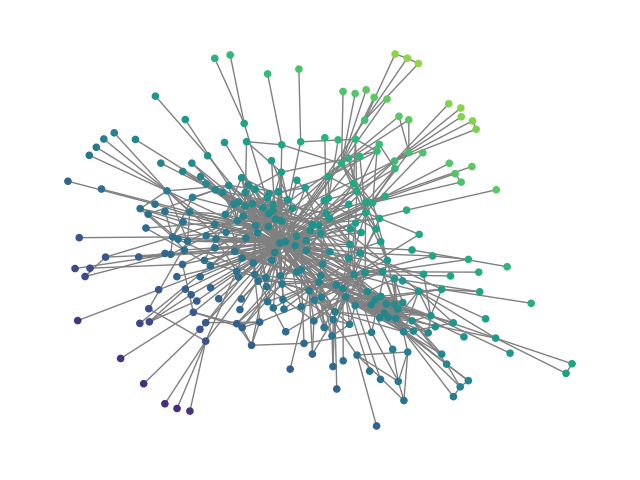}&
        \includegraphics[width=0.15\textwidth]{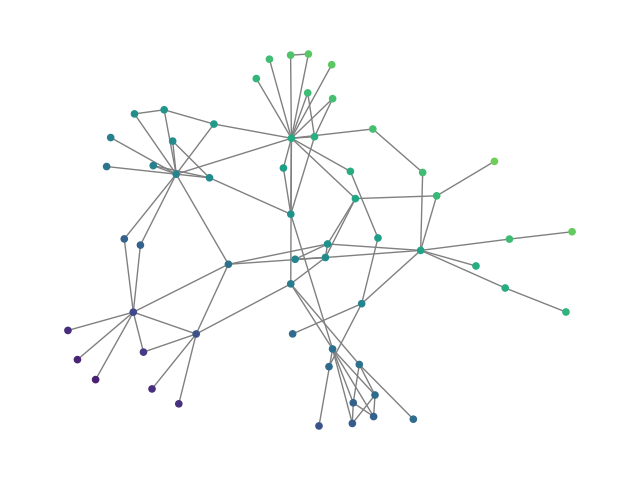} 
        \\
        \includegraphics[width=0.15\textwidth]{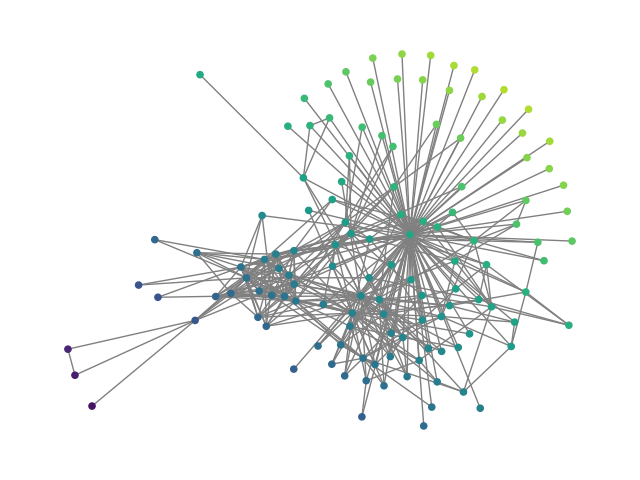} &
        \includegraphics[width=0.15\textwidth]{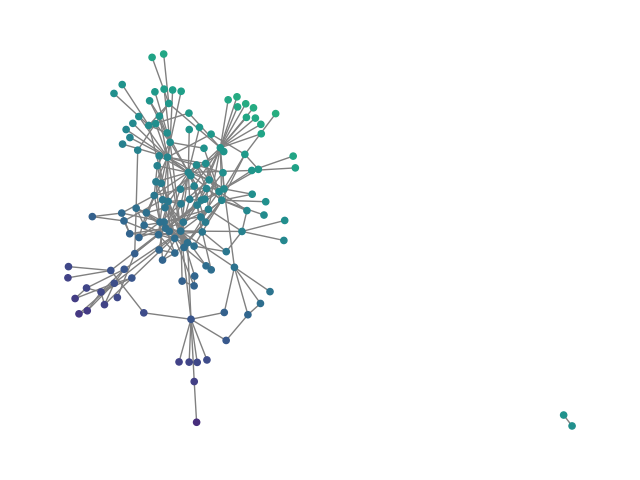} &
        \includegraphics[width=0.15\textwidth]{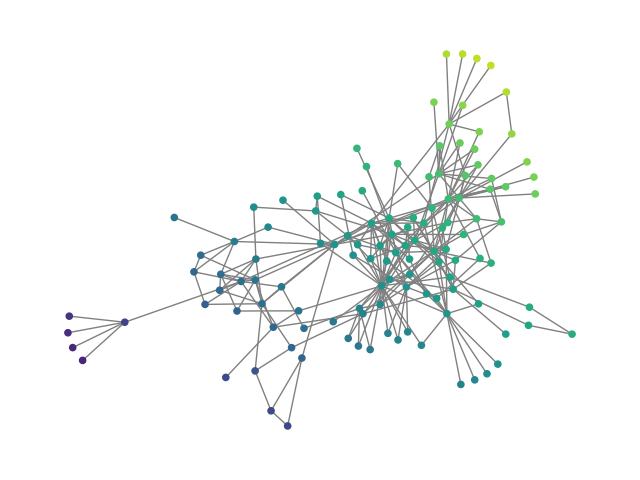} &
        \includegraphics[width=0.15\textwidth]{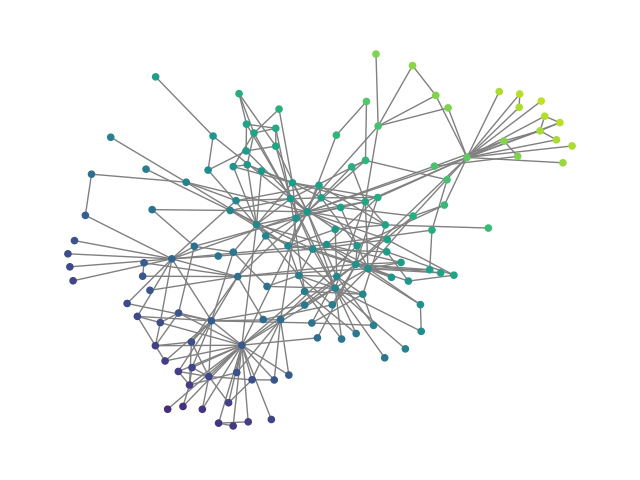} &
        \includegraphics[width=0.15\textwidth]{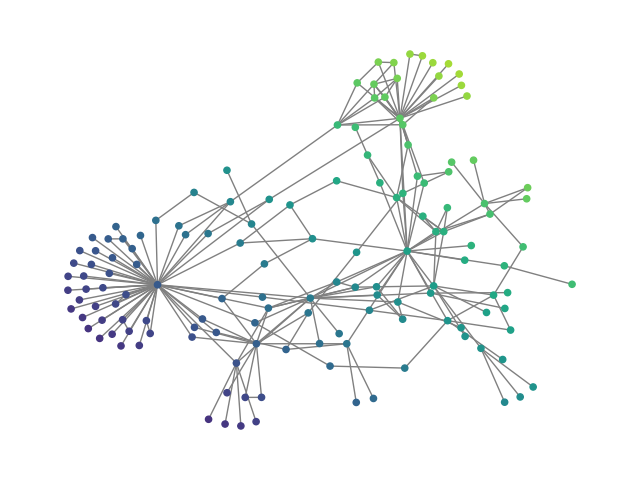}&
        \includegraphics[width=0.15\textwidth]{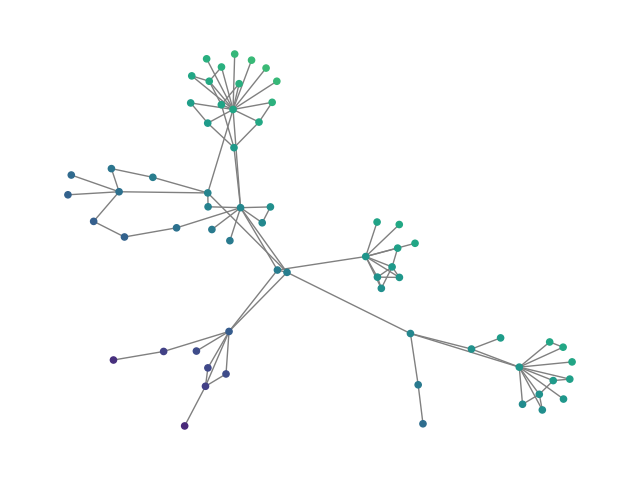} 
        \\
        \includegraphics[width=0.15\textwidth]{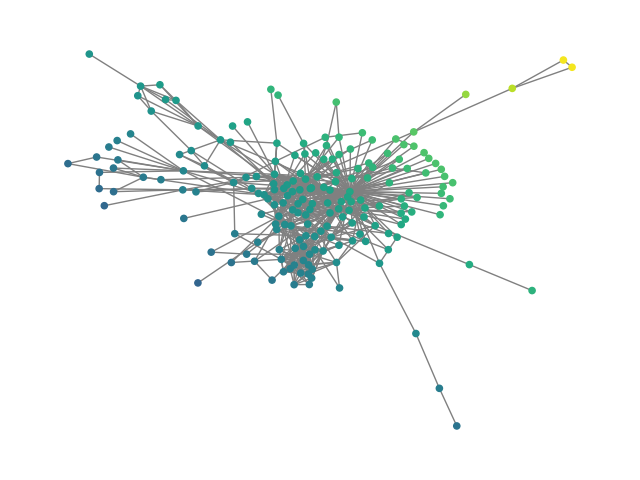} &
        \includegraphics[width=0.15\textwidth]{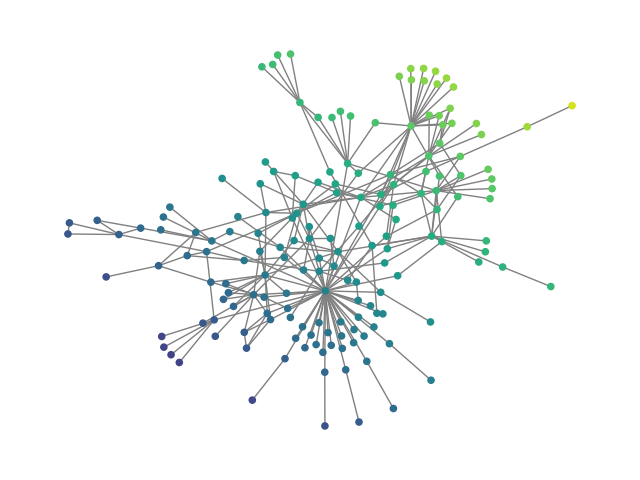} &
        \includegraphics[width=0.15\textwidth]{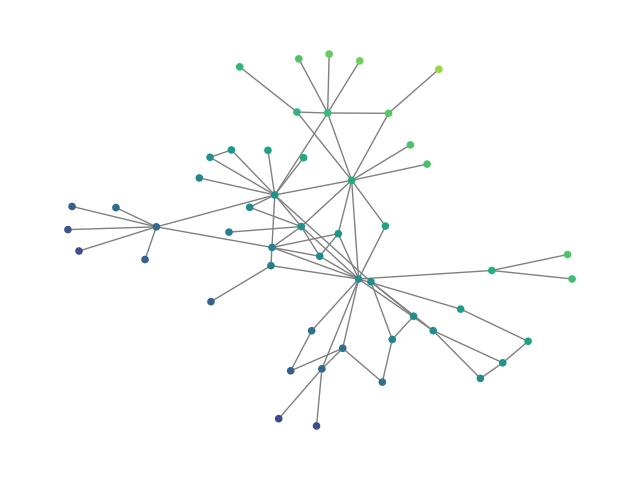} &
        \includegraphics[width=0.15\textwidth]{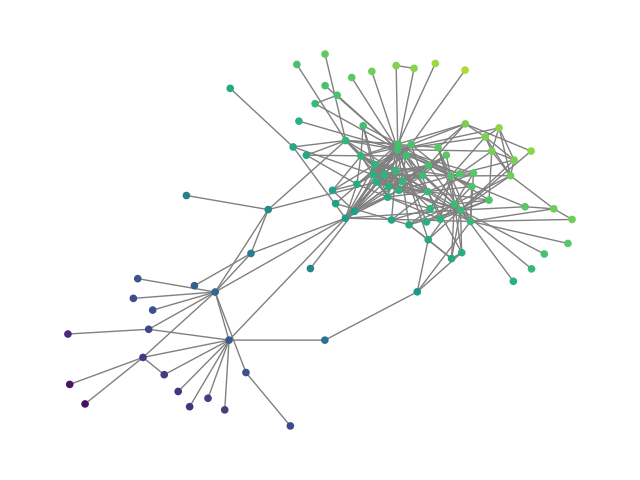} &
        \includegraphics[width=0.15\textwidth]{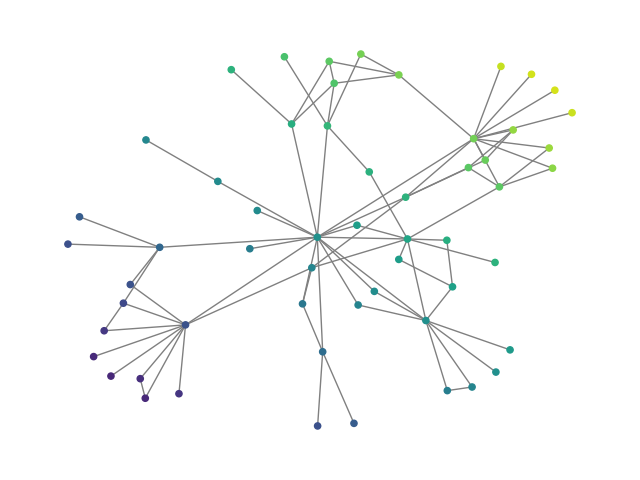}&
        \includegraphics[width=0.15\textwidth]{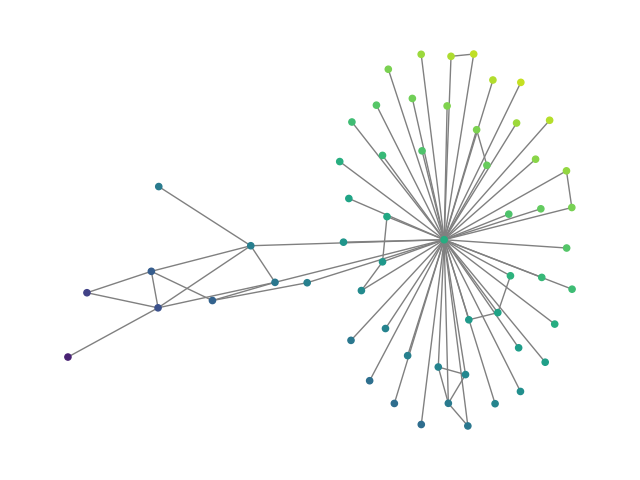} 
        \\
        \hline
        \end{tabular}
    \caption{ego: Comparison of generated graphs with graphs from the dataset.}
    \label{fig:ego}
\end{figure}

\subsection{Reddit}

\begin{figure}[H]
    \centering
    \begin{tabular}{|ccc|ccc|}
    \hline
      \multicolumn{3}{|c|}{Generated graphs}   &
      \multicolumn{3}{|c|}{Real graphs - reddit}
     \\
    \hline
           
        \includegraphics[width=0.15\textwidth]{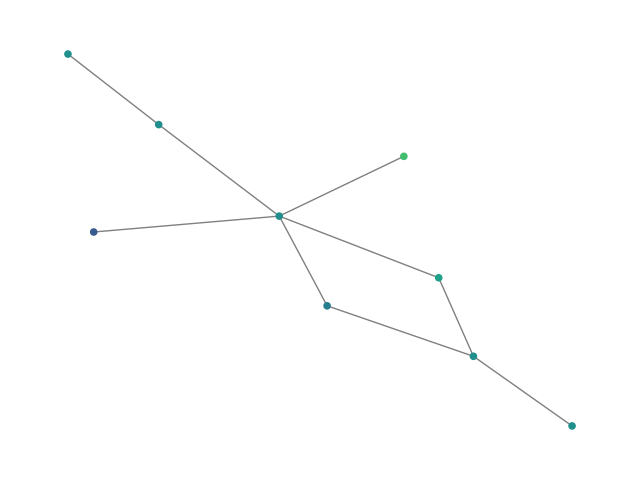} &
        \includegraphics[width=0.15\textwidth]{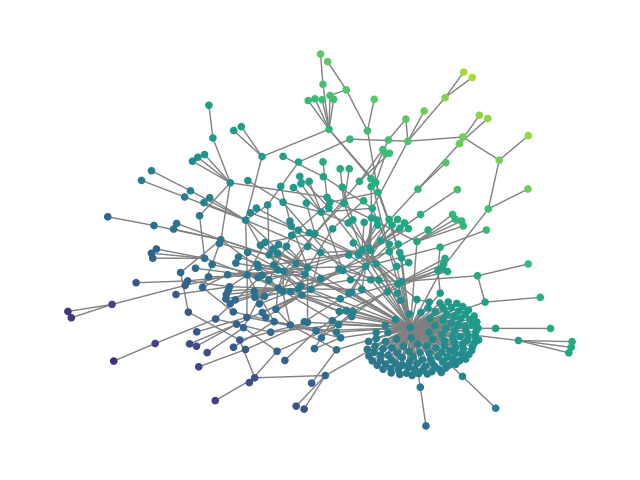} &
        \includegraphics[width=0.15\textwidth]{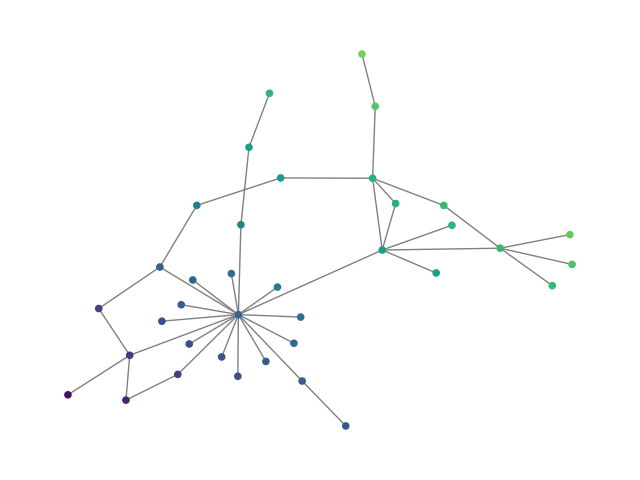} &
        \includegraphics[width=0.15\textwidth]{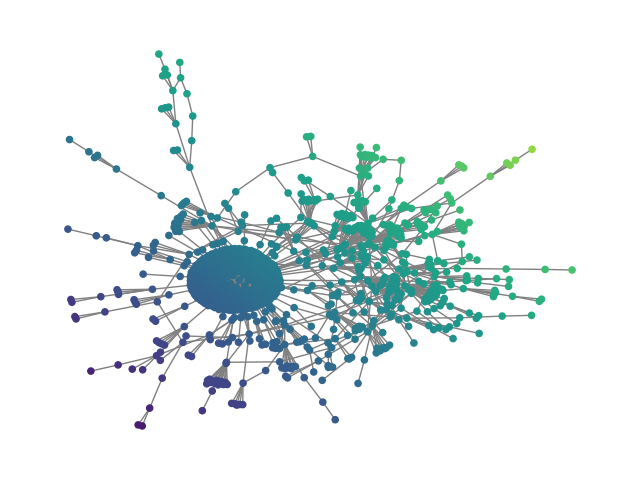} &
        \includegraphics[width=0.15\textwidth]{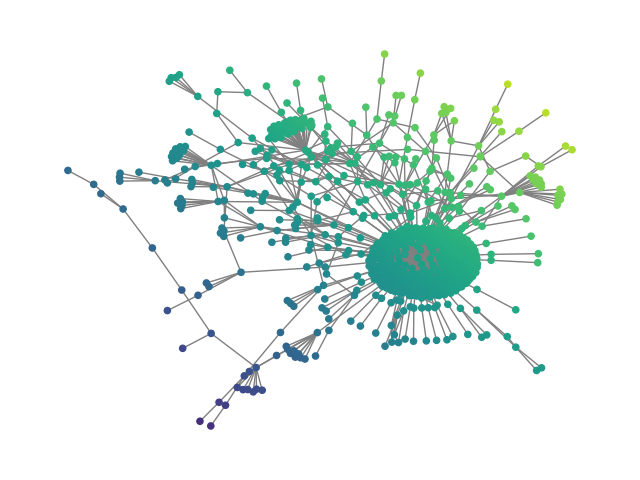}&
        \includegraphics[width=0.15\textwidth]{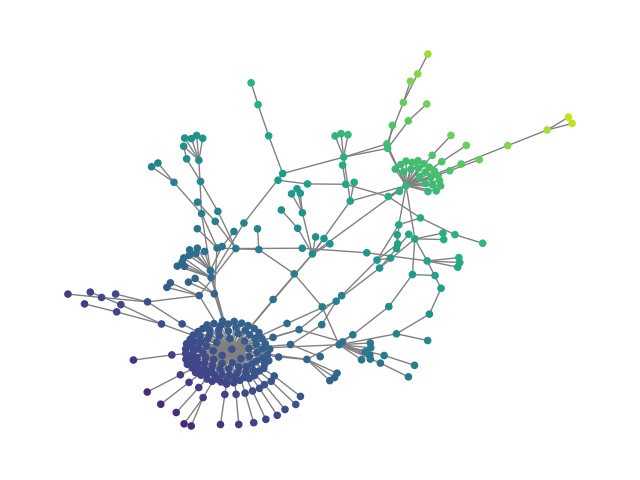} 
        \\
        \includegraphics[width=0.15\textwidth]{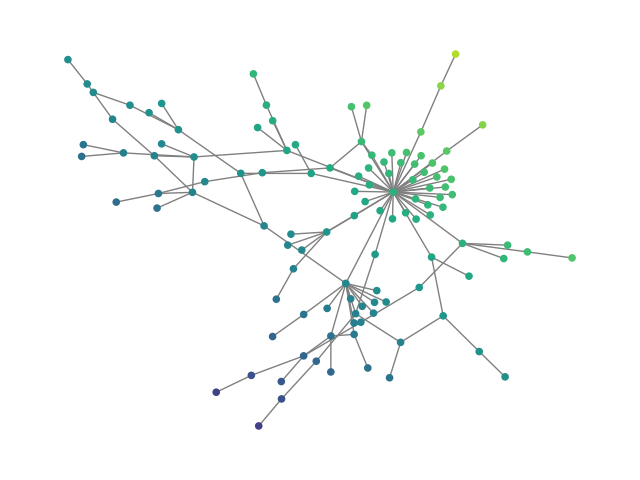} &
        \includegraphics[width=0.15\textwidth]{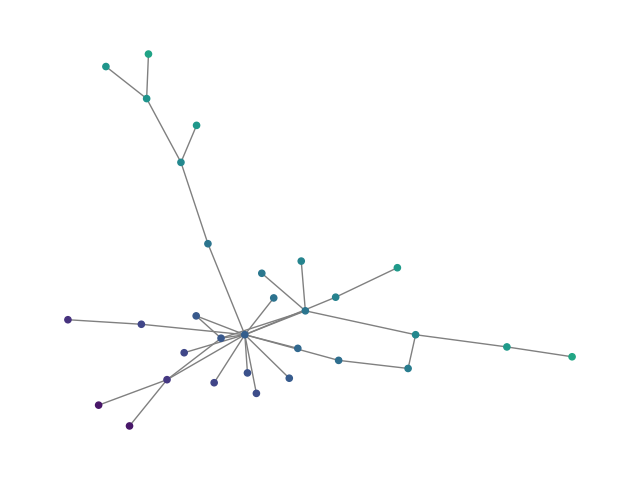} &
        \includegraphics[width=0.15\textwidth]{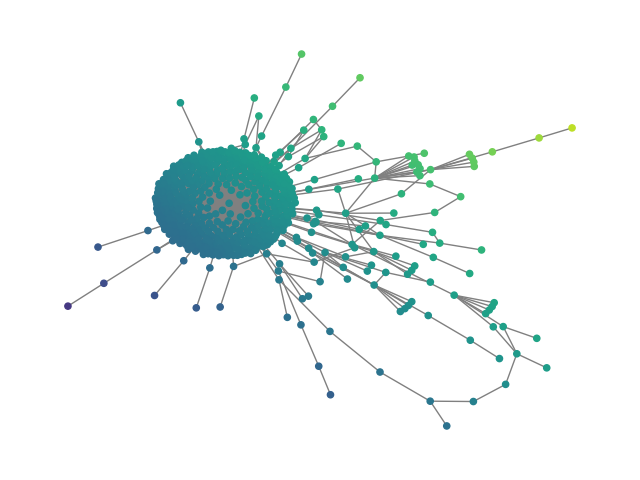} &
        \includegraphics[width=0.15\textwidth]{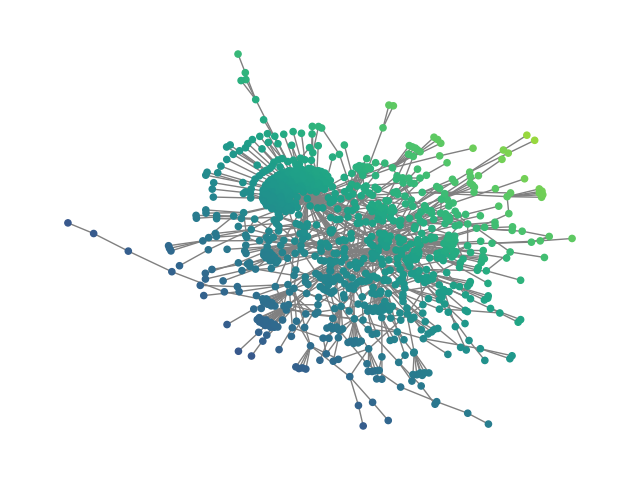} &
        \includegraphics[width=0.15\textwidth]{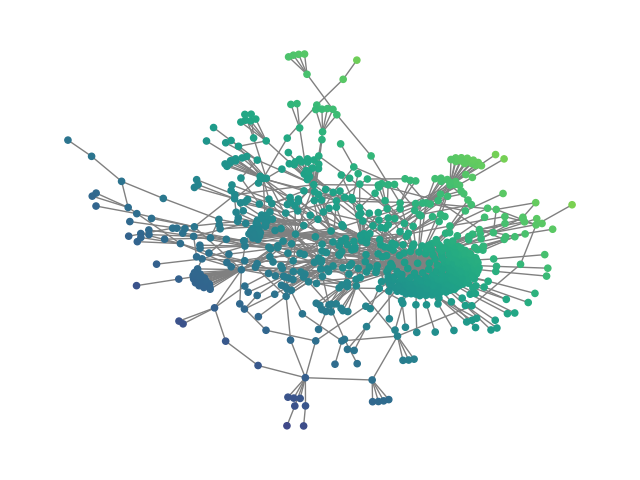}&
        \includegraphics[width=0.15\textwidth]{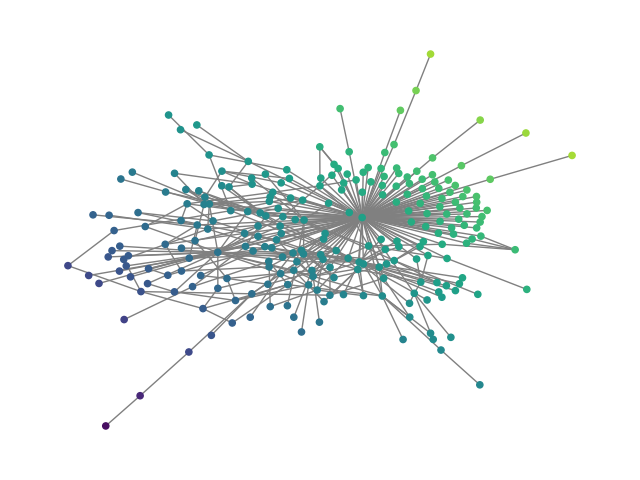} 
        \\
        \includegraphics[width=0.15\textwidth]{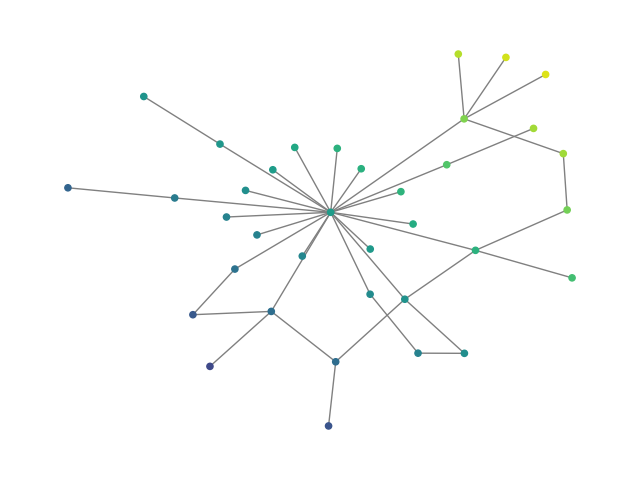} &
        \includegraphics[width=0.15\textwidth]{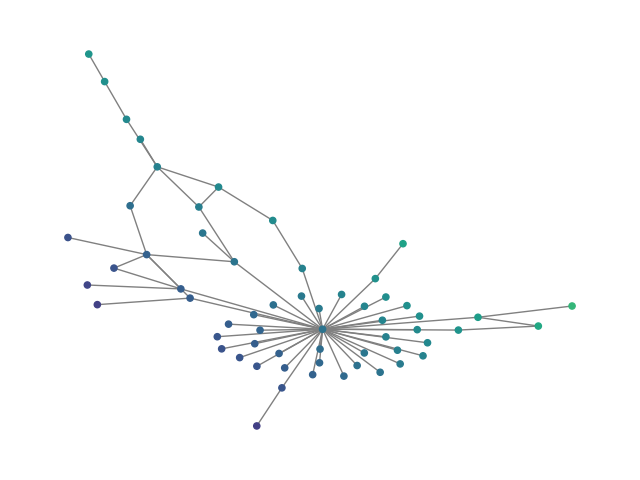} &
        \includegraphics[width=0.15\textwidth]{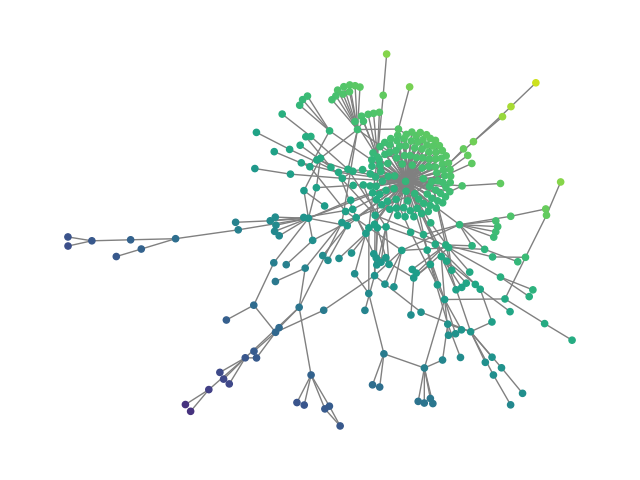} &
        \includegraphics[width=0.15\textwidth]{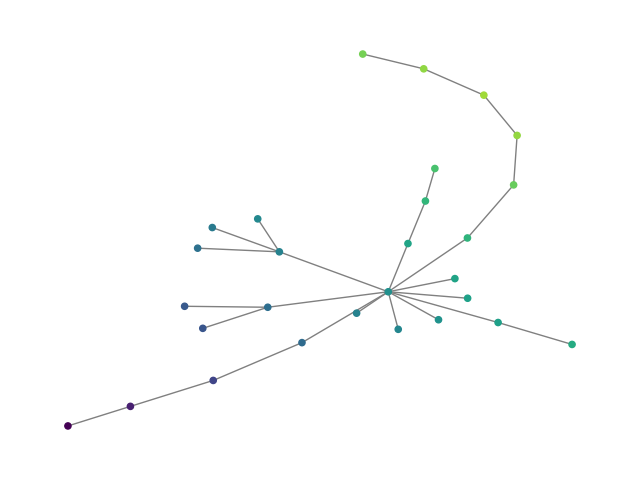} &
        \includegraphics[width=0.15\textwidth]{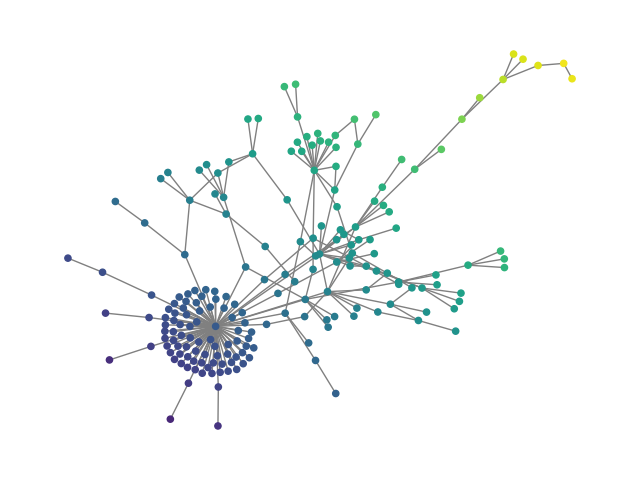}&
        \includegraphics[width=0.15\textwidth]{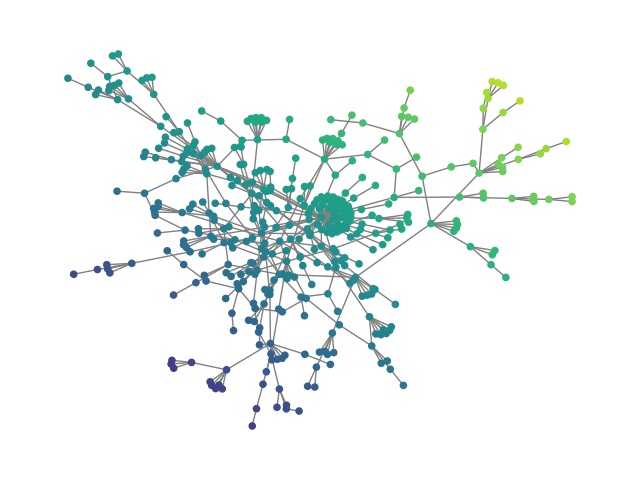} 
        \\
        \includegraphics[width=0.15\textwidth]{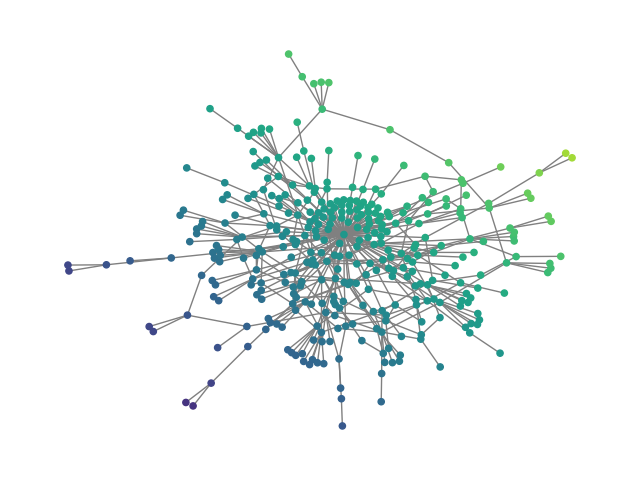} &
        \includegraphics[width=0.15\textwidth]{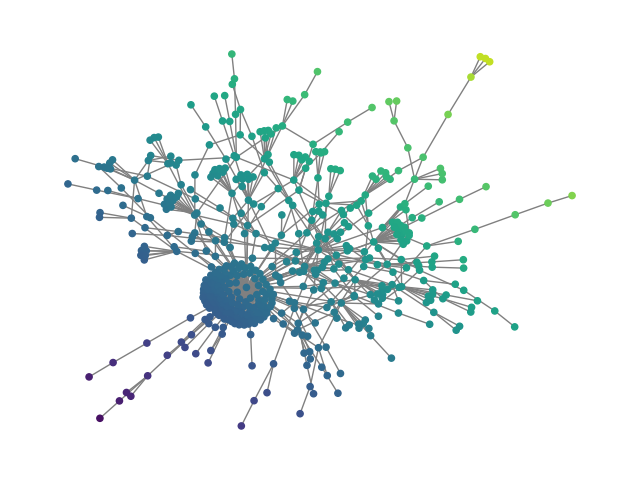} &
        \includegraphics[width=0.15\textwidth]{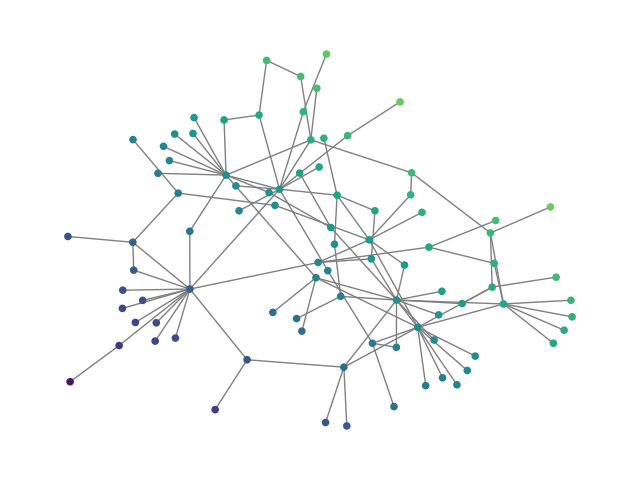} &
        \includegraphics[width=0.15\textwidth]{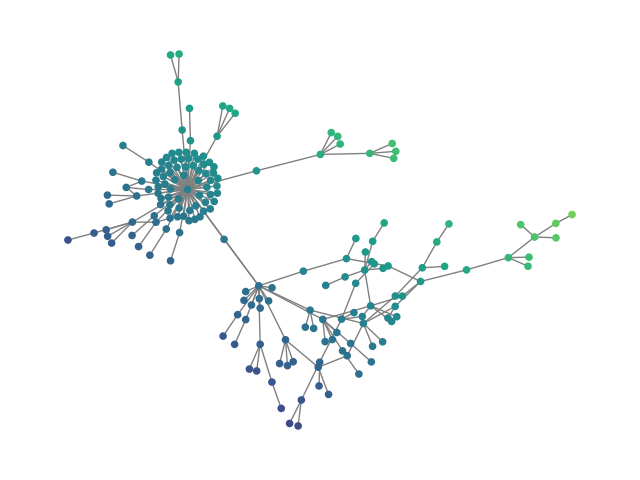} &
        \includegraphics[width=0.15\textwidth]{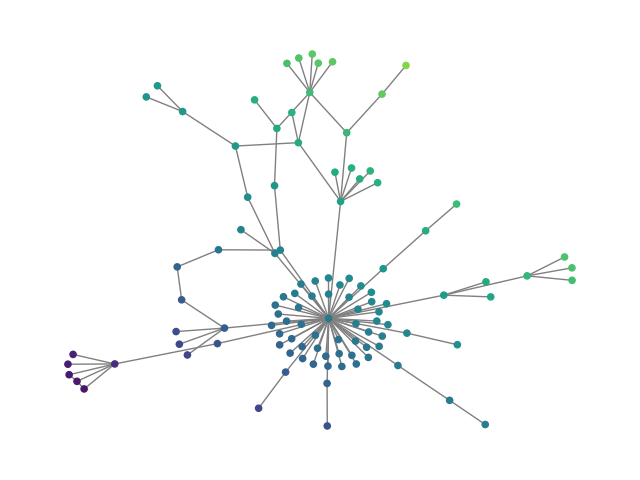}&
        \includegraphics[width=0.15\textwidth]{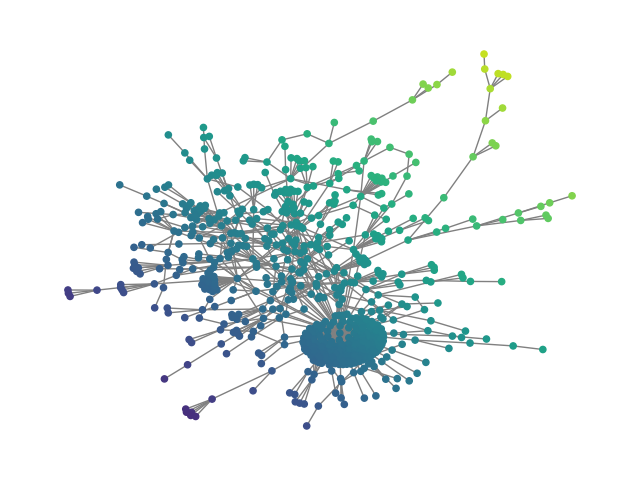} 
        \\
        \includegraphics[width=0.15\textwidth]{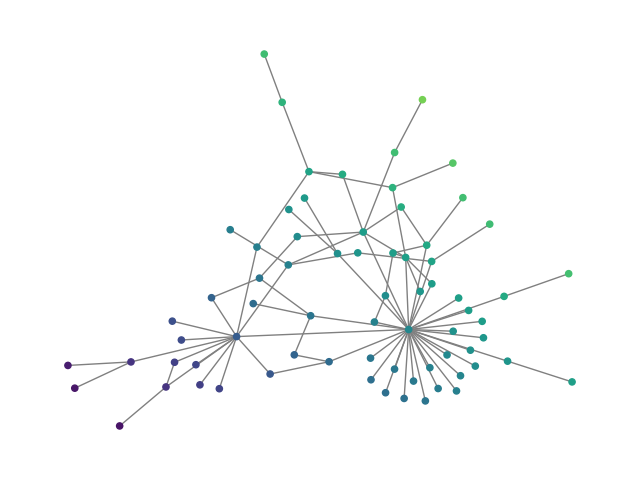} &
        \includegraphics[width=0.15\textwidth]{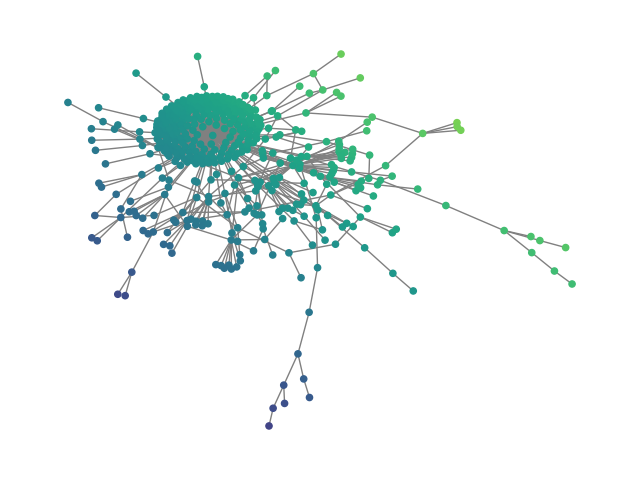} &
        \includegraphics[width=0.15\textwidth]{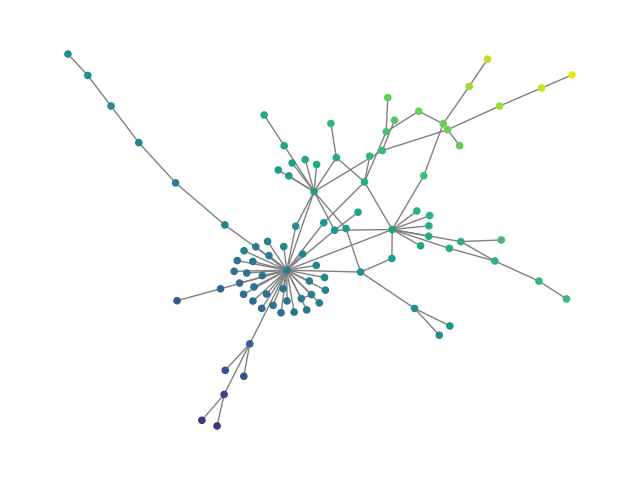} &
        \includegraphics[width=0.15\textwidth]{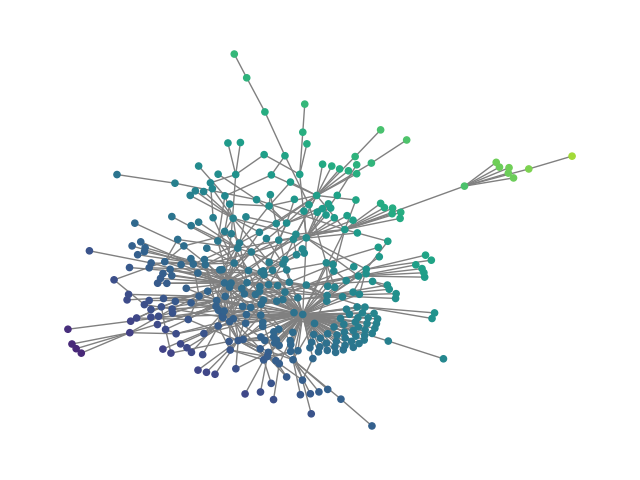} &
        \includegraphics[width=0.15\textwidth]{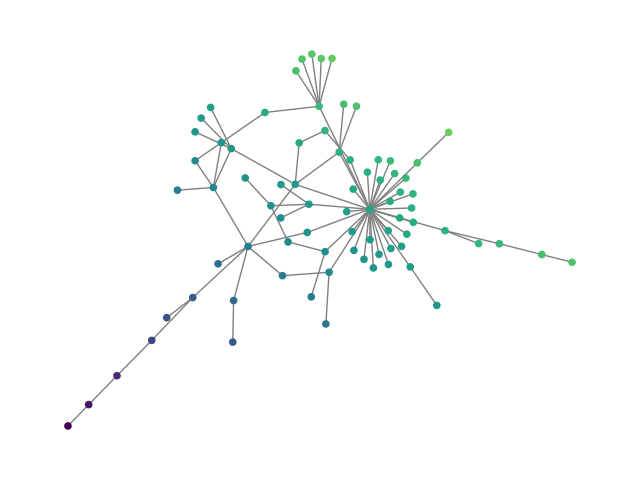}&
        \includegraphics[width=0.15\textwidth]{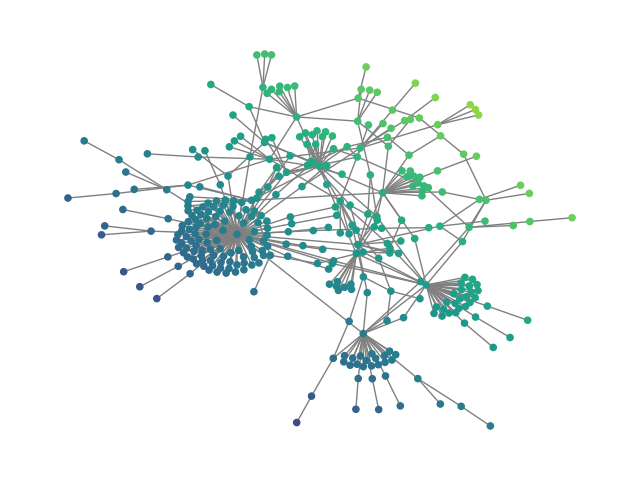} 
        \\
        \includegraphics[width=0.15\textwidth]{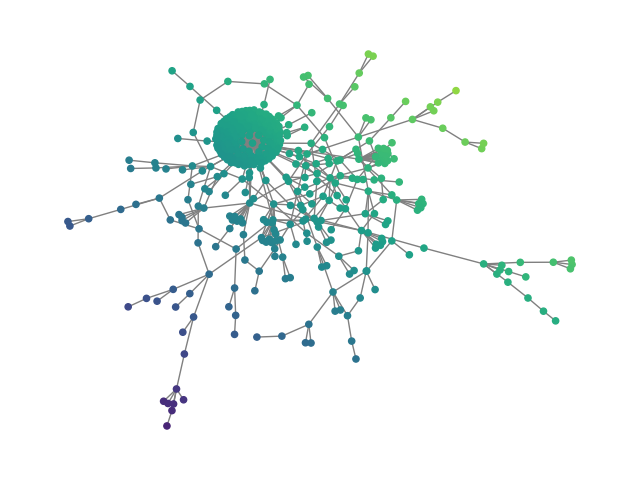} &
        \includegraphics[width=0.15\textwidth]{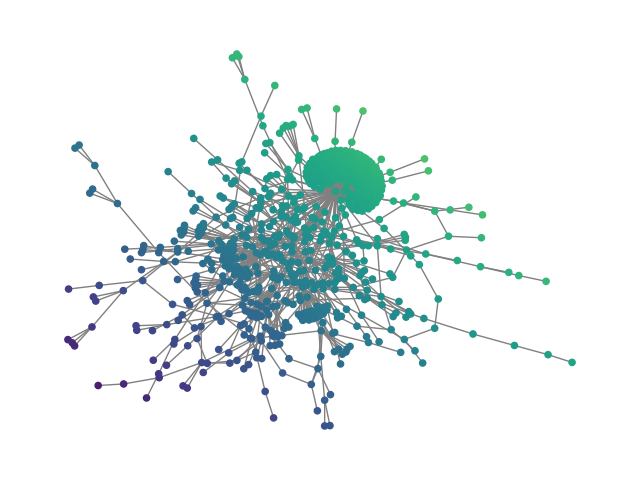} &
        \includegraphics[width=0.15\textwidth]{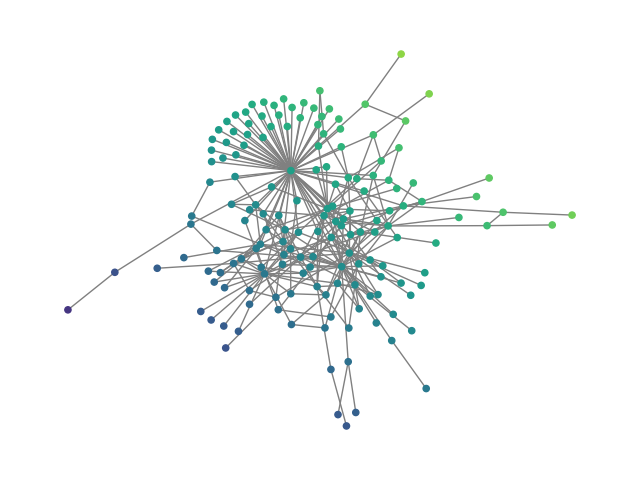} &
        \includegraphics[width=0.15\textwidth]{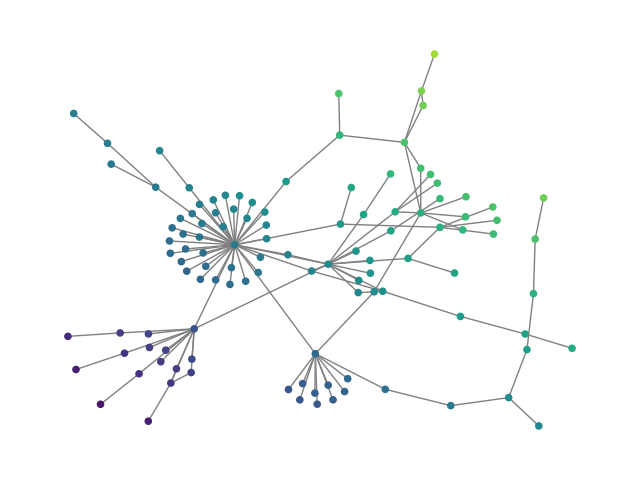} &
        \includegraphics[width=0.15\textwidth]{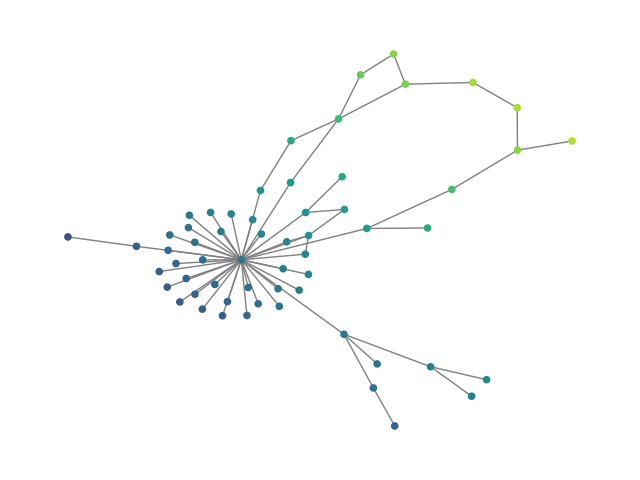}&
        \includegraphics[width=0.15\textwidth]{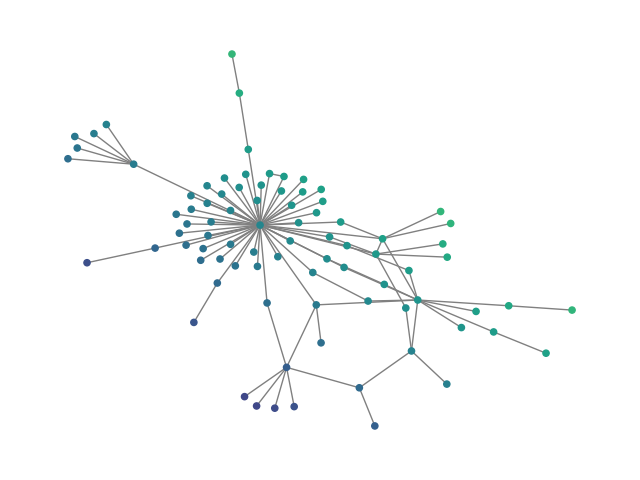} 
        \\
        \hline
        \end{tabular}
    \caption{reddit: Comparison of generated graphs with graphs from the dataset.}
    \label{fig:reddit}
\end{figure}

\end{document}